\documentclass[USenglish]{article}

\usepackage[final,nonatbib]{neurips_2022}

\usepackage[utf8]{inputenc} %
\usepackage[T1]{fontenc}    %
\usepackage{babel}
\usepackage{csquotes}
\usepackage{url}            %
\usepackage{booktabs}       %
\usepackage{amsmath,amsfonts,amsthm,amssymb}
\usepackage{nicefrac}       %
\usepackage{microtype}      %
\usepackage{xcolor}         %
\usepackage{mathtools}
\usepackage{thmtools,thm-restate}
\usepackage{bbm}  %
\usepackage{enumitem}
\usepackage{thm-restate}
\usepackage{comment}

\usepackage{subcaption}

\definecolor{citecolor}{RGB}{0,180,0}
\definecolor{linkcolor}{RGB}{180,0,0}
\definecolor{urlcolor}{RGB}{0,0,180}
\usepackage[colorlinks,citecolor=citecolor,linkcolor=linkcolor,urlcolor=urlcolor]{hyperref}

\usepackage[capitalize,noabbrev]{cleveref}

\usepackage[style=authoryear,sortcites=ynt,maxcitenames=2,maxbibnames=10,maxalphanames=5,useprefix,uniquelist=minyear,dashed=false]{biblatex}
\renewbibmacro{in:}{} %
\DefineBibliographyStrings{english}{urlseen={visited in}}
\DeclareFieldFormat[unpublished,misc]{title}{\mkbibquote{#1}}  %
\let\citet\textcite
\let\citep\parencite

\newcommand{\httpurl}[1]{\href{http://#1}{\nolinkurl{#1}}}
\newcommand{\httpsurl}[1]{\href{https://#1}{\nolinkurl{#1}}}

\newcommand{\norm}[1]{\left\lVert #1 \right\rVert}

\DeclareMathOperator{\diag}{diag}
\DeclareMathOperator{\Tr}{Tr}
\newcommand{\OP}{\mathit{op}}

\newcommand{\N}{\mathbb{N}}
\newcommand{\Z}{\mathbb{Z}}

\newcommand{\cK}{\mathcal{K}}
\newcommand{\cH}{\mathcal{H}}
\newcommand{\cX}{\mathcal{X}}
\newcommand{\cN}{\mathcal{N}}
\newcommand{\cB}{\mathcal{B}}
\newcommand{\cD}{\mathcal{D}}

\newcommand{\cS}{\mathcal{S}}

\newcommand{\bE}{\mathbb{E}}
\newcommand{\R}{\mathbb{R}}
\newcommand{\F}{\mathcal{F}}
\renewcommand{\L}{\mathcal{L}}
\newcommand{\I}{\mathcal{I}}

\let\pr\Pr

\DeclareMathOperator*{\Var}{\mathrm{Var}}
\DeclareMathOperator*{\E}{\bE}      %

\newcommand{\bone}{\mathbbm{1}}

\DeclareMathOperator*{\argmax}{arg\,max}

\newtheorem{theorem}{Theorem}
\newtheorem{corollary}{Corollary}

\newtheorem{lemma}{Lemma}

\theoremstyle{definition}
\newtheorem{defn}{Definition}
\newtheorem{remark}{Remark}

\title{A Non-Asymptotic Moreau Envelope Theory for \\High-Dimensional Generalized Linear Models}

\author{%
Lijia Zhou\thanks{These authors contributed equally.}\\
University of Chicago\\
\texttt{zlj@uchicago.edu}
\And
Frederic Koehler$^*$\\
Stanford University\\
\texttt{fkoehler@stanford.edu}
\And
Pragya Sur\\
Harvard University\\
\texttt{pragya@fas.harvard.edu}
\And
Danica J.\ Sutherland\\
University of British Columbia \& Amii\\
\texttt{dsuth@cs.ubc.ca}
\And
Nathan Srebro\\
Toyota Technological Institute at Chicago\\
\texttt{nati@ttic.edu}
}

\begin{document}

\maketitle
\setcounter{footnote}{0}
\begin{center}
\vspace{-11mm}
{{Collaboration on the Theoretical Foundations of Deep Learning} (\httpsurl{deepfoundations.ai})}
\vspace{3mm}
\end{center}

\begin{abstract}
We prove a new generalization bound that shows for any class of linear predictors in Gaussian space,
the Rademacher complexity of the class and the training error under any continuous loss $\ell$ can control the test error under all $\emph{Moreau envelopes}$ of the loss $\ell$.
We use our finite-sample bound to directly recover the ``optimistic rate'' of \citet{optimistic-rates} for linear regression with the square loss, which is known to be tight for minimal $\ell_2$-norm interpolation, but we also handle more general settings where the label is generated by a potentially misspecified multi-index model. The same argument can analyze noisy interpolation of max-margin classifiers through the squared hinge loss, and establishes consistency results in spiked-covariance settings.
More generally, when the loss is only assumed to be Lipschitz, our bound effectively improves Talagrand’s well-known contraction lemma by a factor of two, and we prove uniform convergence of interpolators \citep{uc-interpolators} for all smooth, non-negative losses. Finally, we show that application of our generalization bound using localized Gaussian width will generally be sharp for empirical risk minimizers, establishing a non-asymptotic Moreau envelope theory for generalization that applies outside of proportional scaling regimes, handles model misspecification, and complements existing asymptotic Moreau envelope theories for M-estimation.

\end{abstract}

\section{Introduction}

Modern machine learning models often contain more parameters than the number of training samples. Despite the capacity to overfit, training these models without explicit regularization has been empirically shown to achieve good generalization performance \citep{NTS:real-inductive-bias, ZBHRV:rethinking, reconcile:interpolation}. On the theoretical side, the study of minimal-norm interpolation has revealed fascinating phenomena that challenge traditional understandings of machine learning.

We now have a better understanding of how properties of the data distribution and algorithmic bias can affect generalization in high-dimensional linear regression. For example, data with a spiked covariance structure can ensure that the test error of ridge regression will be approximately constant once the regularization strength is small enough for the model to fit the signal \citep{optimistic-rates, tsigler2020benign}, contradicting the classical U-shaped curve expected from arguments about the bias-variance tradeoff. Surprisingly, even when the signal is sparse, the risk of the minimal-$\ell_1$ norm interpolator can be shown to converge much slower to the Bayes risk than the minimal-$\ell_2$ norm interpolator in the junk feature setting \citep{chatterji2021foolish, uc-interpolators}. In contrast, the minimal-$\ell_2$ norm interpolator fails to achieve consistency in the isotropic setting, while the minimal-$\ell_1$ norm interpolator is consistent with sparsity but suffers from an exponentially slow rate in the number of parameters $d$ \citep{wang2021tight,muthukumar:interpolation}. However, 
we can still achieve the minimax rate with minimal-$\ell_p$ norm interpolators with $p$ extremely close to 1 \citep{donhauser2022fastrate}. %

In fact, many of the intriguing phenomena from the work above may be understood using the norm of a predictor; localized notions of uniform convergence have emerged as essential tools for doing so. Compared to other techniques, uniform convergence analyses can have the benefit of requiring neither particular proportional scaling regimes nor closed-form expressions for the learned model, since only an approximate estimate of its complexity is needed. Despite uniform convergence's potential for wider applicability, though, work in this area has mostly focused on linear regression settings with strong assumptions: that the conditional expectation of the label is linear with respect to the features, and that the residual has constant variance. In contrast, classical agnostic learning guarantees established by uniform convergence usually need only much weaker assumptions on the data distribution, and apply to a broader range of losses and function classes. For example, \citet{srebro2010optimistic} show that bounds with an ``optimistic rate'' hold generally for any smooth, nonnegative loss, though the hidden logarithmic factor in their result is too loose for explaining noisy interpolation; this was recently addressed by \citet{optimistic-rates} in the special case of well-specified linear regression.

In this work, we take a step further towards agnostic interpolation learning and consider a high-dimensional generalized linear model (GLM) setting where the label is generated by a potentially misspecified model. We show a new generalization bound that allows us to use the Moreau envelopes of any continuous loss function as an intermediate tool. By optimizing over the smoothing parameter to balance the approximation and generalization errors, our general Moreau envelope theory yields sharp non-asymptotic generalization bounds in a wide variety of settings. Applying to linear regression with the square loss, we recover the optimistic rate of \citet{optimistic-rates} and show that it can more generally be extended to handle model misspecification, such as nonlinear trends and heteroskedasticity. The generality of our result comes from the fact that taking the Moreau envelope of the square loss only scales by a constant; this property alone suffices to obtain a generalization guarantee in terms of the original square loss. The squared hinge loss enjoys the same property, and hence a completely analogous argument shows an optimistic rate in that setting. Combined with an analysis of the margin, we show a novel consistency result for max-margin classifiers.

More generally, we apply the Moreau envelope theory to obtain a generic bound for any Lipschitz loss and smooth, nonnegative loss with sharp constants. Looking specifically at the test error of  an Empirical Risk Minimizer (ERM), we show our generalization bound with localized Gaussian width will be asymptotically sharp even when overfitting is not necessarily benign, yielding a version of the asymptotic Moreau envelope framework for analyzing M-estimators \citep{el2013robust, bean2013optimal, donoho2016high, thrampoulidis2018precise, sur2019modern} but for the problem of generalization. Numerical simulations on a variety of feature distributions and label generating processes confirm the wide applicability of our theory.

\section{Related Work} \label{sec:related}

The Moreau envelope has been useful in characterizing asymptotic properties of M-estimators in linear models \citep{bean2013optimal,el2013robust,donoho2016high,el2018impact,thrampoulidis2018precise} and logistic regression \citep{sur2019modern,sur2019likelihood,candes2020phase,salehi2019impact,zhao2022asymptotic}. This theory focuses on estimation and inference under proportional asymptotics, rather than generalization, and does not provide any non-asymptotic results.

For linear regression, \citet{bartlett2020benign} identify nearly-matching necessary and sufficient conditions for the consistency of minimal-$\ell_2$ norm interpolation; their subsequent work \citep{tsigler2020benign} shows generalization bounds for overparametrized ridge regression. Following their work, \citet{negrea:in-defense} and \citet{junk-feats} explore the role of uniform convergence, including showing that uniformly bounding the difference between training error and test error fails to explain interpolation learning. \citet{junk-feats} argue, however, that uniform convergence \textit{of interpolators} is sufficient to establish consistency in a toy example. \citet{uc-interpolators} extend their result to arbitrary data covariance and norm, recovering the benign overfitting conditions of \citet{bartlett2020benign} as well as proving novel consistency results for the minimal-$\ell_1$ norm interpolator. Based on this uniform convergence framework, \citet{wang2021tight} establish tight bounds for the minimal-$\ell_1$ norm interpolator under a sparse signal with isotropic data. Earlier work \citep{JLL:basis-pursuit, chinot2020robustness, li2021minimum} also studied the minimal-$\ell_1$ norm interpolator, without showing consistency. Though the minimal-$\ell_1$ norm interpolator suffers from an exponentially slow rate, \citet{donhauser2022fastrate} show the minimal-$\ell_p$ norm interpolator can achieve faster rates with $p$ close to 1. \citet{optimistic-rates} show a risk-dependent (``localized'') bound that extends the uniform convergence of interpolators guarantee to predictors with arbitrary training loss, and used it to establish generalization for regularized estimators such as Ridge and LASSO. Our Moreau envelope theory builds on the techniques developed in this line of work to apply uniform convergence in the interpolation regime.

In terms of requirements on the data distribution, \citet{bartlett2020benign,tsigler2020benign} only require the feature vector to be sub-Gaussian, but assume a well-specified linear model for the conditional distribution of the label. The uniform convergence-based works also assume a well-specified linear model, but the assumptions are more restrictive in the sense that the marginal distribution of the feature needs to be \emph{exactly} Gaussian because their proof techniques rely on the Gaussian Minimax Theorem (GMT). Our Moreau envelope theory's application to linear regression significantly relaxes the assumption on the label generating process, though it is still constrained by the Gaussian data assumption. \citet{shamir2022implicit} also studies model misspecification in linear regression, but allows non-Gaussian features, and shows that benign overfitting does not necessarily occur in the most general setting, even with a spiked-covariance structure (see his Example 1).

For linear classification, \citet{muthukumar:classification} analyze $\ell_2$ max-margin classifier by connecting to minimal-norm interpolation in regression. Similarly, our analysis in the classification case depends on the fact the squared hinge loss goes through the same transformation as the square loss under smoothing by Moreau envelope.  \citet{donhauser2022fastrate} prove generalization bounds for $\ell_p$ max-margin classifiers in the isotropic setting and do not consider the spiked-covariance case. \citet{deng2019model,montanari2019generalization,liang2020precise} derive exact expressions for the asymptotic prediction risk of the $\ell_2$ and $\ell_p$ (with $p \in [1,2)$) max-margin classifiers. Though their proof techniques also rely on the GMT, our approaches are drastically different. We use GMT in order to show uniform convergence for a class of predictors and establish a non-asymptotic bound, whereas their results are asymptotic and assume a proportional scaling limit. This is a key distinction, because overfitting usually cannot be benign with proportional scaling \citep[e.g.][Proposition 1]{donhauser2022fastrate}. Similar lower bounds have also been shown in the context of linear regression \citep{muthukumar:interpolation, wang2021tight, junk-feats}. 

Some concurrent works have obtained consistency results for max-margin classification in the spiked covariance setting. In particular, the May 2022 version of the work by \citet{shamir2022implicit} also studies convergence to the minimum of the squared hinge loss, and obtains consistency under conditions similar to the benign covariance condition of \citet{bartlett2020benign}. During preparation of this manuscript we  learned of concurrent work by Montanari et al., not yet publicly available, which also studies consistency results for classification. Comparing our \cref{corr:benign-overfitting} to \citet{shamir2022implicit}, their result applies to some non-Gaussian settings, but in the Gaussian setting their result is not as general as ours. (Combining Assumptions 1 and 2 of Theorem 7 there, they require the norm of the data to be bounded, whereas our \cref{corr:benign-overfitting} applies even if $o(n)$ eigenvalues of $\Sigma$ grow arbitrarily quickly with $n$.) More conceptually, our result follows from a norm-based generalization bound that applies to all predictors and outside of the ``benign overfitting'' conditions, generalizing the result of \citet{uc-interpolators} and unlike the analysis of prior work.

\section{Problem Formulation} \label{sec:model}

\paragraph{GLM setting.}
Given a continuous loss function $f: \R \times \R \to \R$ and i.i.d. sample pairs $(x_i, y_i)$ from some data distribution $\cD$, we can learn a linear model $(\hat{w}, \hat{b})$ by minimizing the empirical loss $\hat{L}_f$ with the goal of achieving small population loss $L_f$:
\begin{equation}
	\hat{L}_f(w, b) = \frac{1}{n} \sum_{i=1}^n f(\langle w, x_i\rangle + b, y_i), \quad L_f(w, b) = \E_{(x,y) \sim \cD} f(\langle w, x \rangle + b, y).
\end{equation}

\paragraph{Multi-index model.}
We assume that the data distribution $\cD$ over $(x,y)$ is such that
\begin{enumerate}
	\item $x \sim \cN (0, \Sigma)$ is a centered 
    Gaussian with unknown covariance matrix $\Sigma$.
	\item There are unknown weight vectors $w_1^*, ..., w_k^* \in \R^d$ such that the $\Sigma^{1/2} w_i^*$ are orthonormal, a function $g: \R^{k+1} \to \R$, and a random variable $\xi \sim \cD_\xi$ independent of $x$ (not necessarily Gaussian) such that 
	\begin{equation} \label{eqn:model}
	    \eta_i = \langle w^*_i, x \rangle, \quad y = g(\eta_1, ..., \eta_k, \xi).
	\end{equation}
\end{enumerate}

We can assume that the distribution of $x$ is centered without loss of generality since presence of a mean term simply  corresponds to changing the bias term $b$:
$\langle w, x \rangle + b = \langle w, x-\mu \rangle + (b - \langle w, \mu \rangle).$ We can also assume that $\Sigma^{1/2}w_1^*, ..., \Sigma^{1/2} w_k^*$ are orthonormal without loss of generality since we have not imposed any assumption on the link function $g$. The multi-index model includes well-specified linear regression, by setting $k = 1$ and $g(\eta, \xi) = \eta + \xi$. It also allows nonlinear trends and heteroskedasticity (such as the model in \cref{fig:flatness}) by changing the definition of $g$. Since $g$ need not be continuous, the label $y$ can be binary, as in linear classification.

\section{Moreau Envelope Generalization Theory}

Our theory vitally depends on the Moreau envelope, defined as follows.

\begin{defn} \label{def:moreau}
The \emph{Moreau envelope} of $f: \R \times \R \to \R$ with parameter $\lambda \in \R^+$ is defined as the function $f_{\lambda}: \R \times \R \to \R$ given by
\begin{equation}
	f_{\lambda}(\hat{y}, y) = \inf_{u} \, f(u, y) + \lambda ( u - \hat{y})^2.
\end{equation}
The Moreau envelope can be viewed as a smooth approximation to the original function $f$: in our parameterization, smaller $\lambda$ corresponds to more smoothing. The map that outputs the minimizer $u$, known as the \emph{proximal operator}, plays an important role in convex analysis \citep{parikh2014proximal,bauschke2011convex}. %

\end{defn}

Our general theory, as stated in \cref{thm:main-gen} below, essentially upper bounds the generalization gap between the population Moreau envelope $L_{f_{\lambda}}$ and the original training loss $\hat{L}_{f}$ by the sum of two parts: a parametric component that can be controlled by the dimension $k$ of the ``meaningful'' part of $x$, and a non-parametric component that can be controlled by a dimension-free complexity measure such as the Euclidean norm of the predictor. Typically, the first term is negligible since $k$ is small, and the complexity of fitting all the noise is absorbed into the second term. More precisely, we introduce the following definitions to formalize separating out a low dimensional component:

\begin{defn}
Under the model assumptions \eqref{eqn:model}, define a (possibly oblique) projection matrix $Q$ onto the space orthogonal to $w_1^*, ..., w_k^*$ and a mapping $\phi$ from $\R^d$ to $\R^{k+1}$ by
\begin{equation} \label{eqn:Q-defn}
    Q = I_d - \sum_{i=1}^k  w_i^* ( w_i^*)^T \Sigma, \quad \phi(w) = (\langle w, \Sigma w_1^* \rangle, ..., \langle w, \Sigma w_k^* \rangle, \|\Sigma^{1/2} Qw \|_2 )^T.
\end{equation}
We let $\Sigma^{\perp} = Q^T \Sigma Q$ denote the covariance matrix of $Q^T x$. We also define a low-dimensional \emph{surrogate distribution} $\tilde{\cD}$ over $\R^{k+1} \times \R$ by
\begin{equation} \label{eqn:surrogate-model}
\tilde{x} \sim \cN (0, I_{k+1}), \quad
\tilde\xi \sim \cD_\xi, \quad
\text{ and } \quad
\tilde{y} = g(\tilde x_1, ..., \tilde x_k, \tilde \xi).
\end{equation}
\end{defn}
This surrogate distribution compresses the ``meaningful part'' of $x$ while maintaining the test loss,
as shown by our main result \cref{thm:main-gen} (proved in \cref{apdx:proof-main-gen}).
Note that as a non-asymptotic statement,
the functions $\epsilon_{\lambda,\delta}$ and $C_\delta$ only need hold for a specific choice of $n$ and $\cD$.

\begin{restatable}{theorem}{MainGen} \label{thm:main-gen}
Suppose $\lambda \in \R^+$ satisfies that for any $\delta \in (0, 1)$, there exists a continuous function $\epsilon_{\lambda, \delta}: \R^{k+1} \to \R$ such that with probability at least $1-\delta/4$ over independent draws $(\tilde{x}_i, \tilde{y}_i)$ from the surrogate distribution $\tilde{\cD}$ defined in \eqref{eqn:surrogate-model}, we have uniformly over all $(\tilde{w}, \tilde{b}) \in \R^{k+2}$ that
\begin{equation} \label{eqn:low-dimension-concentration}
    \frac{1}{n} \sum_{i=1}^n f_{\lambda} (\langle \tilde{w}, \tilde{x}_i\rangle +\tilde{b}, \tilde{y}_i) \geq \E_{(\tilde{x}, \tilde{y}) \sim \tilde{D}} \, [f_{\lambda} (\langle \tilde{w}, \tilde{x}\rangle + \tilde{b}, \tilde{y}) ] - \epsilon_{\lambda, \delta}(\tilde{w}, \tilde{b}).
\end{equation}
Further, assume that for any $\delta \in (0, 1)$, there exists a continuous function $C_{\delta}: \R^d \to [0, \infty]$ such that with probability at least $1-\delta/4$ over $x \sim \cN \left(0, \Sigma \right)$, uniformly over all $w \in \R^d$,
	\begin{equation} \label{eqn:complexity-defn}
		\left\langle Qw, x \right\rangle \leq  C_{\delta}(w).
	\end{equation}
Then it holds with probability at least $1-\delta$ that uniformly over all $(w, b) \in \R^{d+1}$, we have
	\begin{equation} \label{eqn:main-bound}
		L_{f_{\lambda}} (w, b) \leq \hat{L}_{f}(w, b) + \epsilon_{\lambda, \delta}( \phi(w), b)+ \frac{\lambda C_{\delta}(w)^2}{n}.
	\end{equation}
If we additionally assume that \eqref{eqn:low-dimension-concentration} holds uniformly for all $\lambda \in \R^+$, then \eqref{eqn:main-bound} does as well.
\end{restatable}

As we will see, we can generally bound the difference between $L_{f_{\lambda}}$ and $L_f$ when the loss is assumed to be Lipschitz. If $f$ is not Lipschitz but smooth (i.e.\ $\nabla f$ is Lipschitz, as for the squared loss), we can always write it as the Moreau envelope of another function $\tilde{f}$. In the special case of square loss or squared hinge loss, the Moreau envelope $f_{\lambda}$ is proportional to $f$, meaning that \eqref{eqn:main-bound} becomes a generalization guarantee in terms of $L_f$. Optimizing over $\lambda$ will establish optimal bounds that recover the result of \citet{uc-interpolators,optimistic-rates}, and lead to other novel results.

\begin{remark}\label{rmk:rademacher}
The complexity functional $C_{\delta}(w)$ should be thought of as a localized, high-probability version of Rademacher complexity. This is because the Gaussian width of a convex set $\cK$, $\E \sup_{w \in \cK} \langle w, x \rangle$, is the same as the Rademacher complexity of the class of linear functions $\{x \mapsto \langle w, x \rangle : w \in \cK \}$  \citep[Proposition 1]{optimistic-rates}. A somewhat similar complexity functional appears in  \citet{panchenko2003symmetrization}. Also, note \eqref{eqn:low-dimension-concentration} requires only \emph{one-sided concentration} --- see Remark~\ref{rmk:one-sided}.
\end{remark}

\subsection{VC Theory for Low-dimensional Concentration}
To apply our generalization result (\cref{thm:main-gen}), we should check the low-dimensional concentration assumption \eqref{eqn:low-dimension-concentration}. The quantitative bounds in the low-dimensional concentration (i.e.\ the precise form of error term $\epsilon_{\lambda,\delta}$) will inevitably depend on the exact setting we consider (see e.g.\ \cite{vapnik2006estimation,koltchinskii2015bounding,lugosi2019mean} for discussion). 

First, we recall the following result from VC theory.

\begin{restatable}[Special case of Assertion 4 of \citet{vapnik2006estimation}, Chapter 7.8; see also Theorem 7.6]{theorem}{VCtheory} \label{thm:vc-hypercontractive}
Let $\cK \subset \R^d$ and $\cB \subset \R$.
Suppose that a distribution $\cD$ over $(x,y) \in \R^d \times \R$ satisfies that for some $\tau > 0$, it holds uniformly over all $(w, b) \in \cK \times \cB$ that
\begin{equation}\label{eqn:hypercontractive-assumption}
\frac{\left(\E f(\langle w, x \rangle + b,y)^4]\right)^{1/4}}{\E f(\langle w,x \rangle + b,y)} \le \tau
.\end{equation}
Also suppose the class of functions 
$\{(x,y) \mapsto \bone\{f(\langle w, x \rangle + b, y) > t\} : w \in \cK, b \in \cB, t \in \mathbb{R}\}$ has VC-dimension at most $h$.
Then for any $n > h$, with probability at least $1 - \delta$ over the choice of $((x_1,y_1),\ldots,(x_n,y_n)) \sim \cD^n$, it holds uniformly over all $w \in \cK, b \in \cB$ that
\[ \frac{1}{n} \sum_{i = 1}^n f(\langle w,x_i \rangle + b,y_i) \ge \left(1 - 8 \tau \sqrt{\frac{h(\log(2n/h) + 1) + \log(12/\delta)}{n}}\right) \E f(\langle w, x \rangle + b, y). \]
\end{restatable}

The assumption \eqref{eqn:hypercontractive-assumption} is standard (indeed, this is the setting primarily focused on in \cite{vapnik2006estimation}) and is sometimes referred to as \emph{hypercontractivity} or \emph{norm equivalence} in the literature; a variant of the result holds with $4$ replaced by $1 + \epsilon$. In many settings of interest, this can be directly checked using the fact that $x$ is Gaussian (for instance, see \cref{thm:poly-hypercontractivity,apdx:linear-clf}). Of course, our general result can be applied without this assumption, by using low-dimensional concentration under an alternative assumption: \citet{vapnik2006estimation,panchenkooptimistic,panchenko2003symmetrization,mendelson2017extending} have further discussion and alternative results;
in particular, Assertion 3 of \citet[Chapter 7.8]{vapnik2006estimation} gives a bound based on a fourth-moment assumption, and \citet[Theorem 3]{panchenko2003symmetrization} gives one based on a version of Rademacher complexity.

Combining \cref{thm:main-gen,thm:vc-hypercontractive} yields the following.

\begin{restatable}{corollary}{VCgen}\label{corr:main-gen-vc}
Under the model assumptions \eqref{eqn:model}, suppose that $C_{\delta}$ satisfies condition \eqref{eqn:complexity-defn}.
Also suppose that for some fixed $\lambda \ge 0$, $\cK \subseteq \R^d$, and $\cB \subseteq \R$,
the surrogate distribution $\tilde \cD$ satisfies assumption \eqref{eqn:hypercontractive-assumption} under $f_{\lambda}$ uniformly over $\phi(\cK) \times \cB$,
and that the class 
$\{(x,y) \mapsto \bone \{ f_{\lambda}(\langle \tilde w, \phi(x) \rangle + \tilde b, y) > t \} : \tilde w \in \phi(\cK), \tilde b \in \cB, t \in \mathbb{R}\}$
has VC-dimension at most $h$.
Then with probability at least $1 - \delta$, uniformly over all $(w, b) \in \cK \times \cB$ 
\[ \left(1 - 8\tau\sqrt{\frac{h(\log(2n/h) + 1) + \log(48/\delta)}{n}}\right) L_{f_{\lambda}}(w,b) \le \hat{L}_f(w,b) +  \frac{\lambda C_{\delta}(w)^2}{n}.   \]
Furthermore, if assumption \eqref{eqn:hypercontractive-assumption} holds uniformly for all $\{f_{\lambda}: \lambda \in \mathbb{R}_{\ge 0} \}$ and the class
$\{(x,y) \mapsto \bone \{ f_{\lambda}(\langle \tilde w, \phi(x) \rangle + \tilde b, y) > t \} : (\tilde w, \tilde b) \in \phi(\cK) \times \cB, t \in \mathbb{R}, \lambda \in \mathbb{R}_{\ge 0}\}$
has VC-dimension at most $h$, then the same conclusion holds uniformly over $\lambda$.
\end{restatable}

The last conclusion (uniformity over $\lambda$)
follows by going through the proof of Theorem~\ref{thm:vc-hypercontractive}, since it is based on reduction to uniform control of indicators.
In every situation we will consider, it is easy to check that the VC dimension $h$ in the theorem statement is $O(k)$, generally by reducing to the fact that halfspaces in $\mathbb{R}^k$ have VC dimension $k + 1$.

\section{Applications} \label{sec:applications}

\subsection{Linear Regression with Square Loss}

In this section, we show how to recover optimistic rates \citep{optimistic-rates} for linear regression without assuming the model is well-specified. We will consider the square loss,
$f(\hat{y}, y) = (\hat{y} - y)^2$.
A key property of the square loss is that the Moreau envelope is proportional to itself:
\begin{equation}\label{eqn:square-loss-moreau}
f_{\lambda}(\hat{y}, y) = \inf_{u} \, (u-y)^2 + \lambda ( u - \hat{y})^2 = \frac{\lambda}{1+\lambda} \, f(\hat{y}, y).
\end{equation}
Thus we can multiply by $(1 + \lambda)/\lambda$ in our generalization bound and solve for the optimal choice of $\lambda$.

\begin{restatable}{corollary}{GenSquareLoss} \label{corr:gen-square-loss}
Suppose $f$ is the square loss and the surrogate distribution $\tilde \cD$ satisfies assumption \eqref{eqn:hypercontractive-assumption} uniformly over $(w, b) \in \R^{k+1}$, then with probability at least $1 - \delta$, uniformly over all $w,b$ we have
\[ \left(1 - 8\tau\sqrt{\frac{k(\log(2n/k) + 1) + \log(48/\delta)}{n}}\right) L_f(w,b) \le \left(\sqrt{\hat{L}_f(w,b)} + C_{\delta}(w)/\sqrt{n}\right)^2. \]
\end{restatable}

As mentioned earlier, assumption \eqref{eqn:hypercontractive-assumption} usually holds under mild conditions on $y$. For example, $\tau$ can be chosen to be a constant when $y$ is a bounded-degree polynomial of a Gaussian due to Gaussian hypercontractivity \citep[Section 11.1]{o2014analysis}. Specializing \cref{corr:gen-square-loss} to interpolators ($\hat{L}_f = 0$) recovers the uniform convergence of interpolators guarantee from \citet{uc-interpolators}. Combined with a more general norm analysis in \cref{sec:benign-overfitting}, we establish $\ell_2$ benign overfitting with misspecification. In the well-specified case, see \citet{optimistic-rates} for detailed examples on ordinary least squares, ridge regression, and LASSO.

\subsection{Classification with Squared Hinge Loss}
In this section, we show a novel optimistic rate bound for max-margin linear classification with the squared hinge loss,
$f(\hat{y}, y) =\max(0, 1 - y \hat{y})^2$.
Its Moreau envelope is given by
\begin{equation*}
	\begin{split}
		f_{\lambda}(\hat{y}, y) 
		&= \inf_{u} \,\max(0, 1 - y u)^2 + \lambda ( u - \hat{y})^2 
		=
		\begin{cases}
			0 &\mbox{if } 1 - y\hat{y} \leq 0 \\
			\frac{\lambda}{1+\lambda} (1-y\hat{y})^2 &\mbox{if } 1 - y\hat{y} > 0 \\
		\end{cases}
		= \frac{\lambda}{1+\lambda} \, f(\hat{y}, y)
	\end{split}
.\end{equation*}

We consider the case of a general binary response $y$ valued in $\{\pm 1\}$ satisfying the model assumptions in equation (\ref{eqn:model}). In this case as well, $f$ is proportional to its Moreau envelope; thus, the same proof as for the squared loss shows that \cref{corr:gen-square-loss}  continues to hold  when square loss is replaced by  squared hinge loss!
In \cref{apdx:classification}, we discuss certain settings (including noisy settings) where minimizing the squared hinge loss also minimizes the zero-one loss, i.e.\ the misclassification rate. 

\begin{figure} 
  \centering
  \includegraphics[width = 13cm]{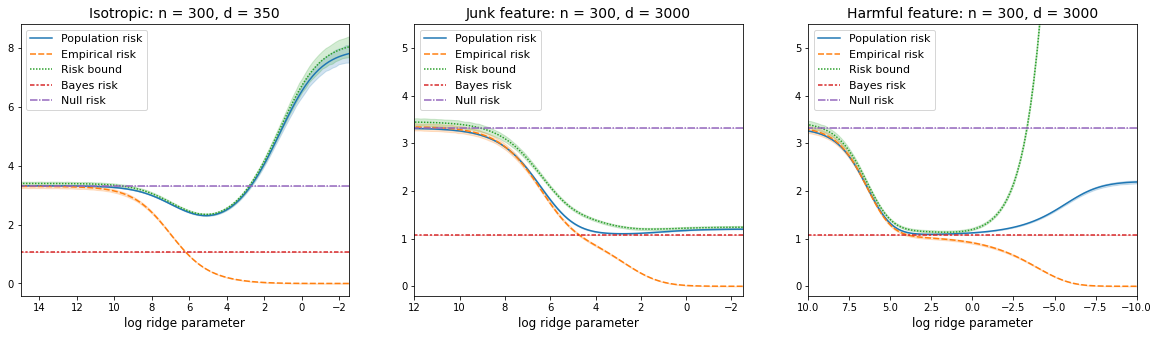}
  \includegraphics[width = 13cm]{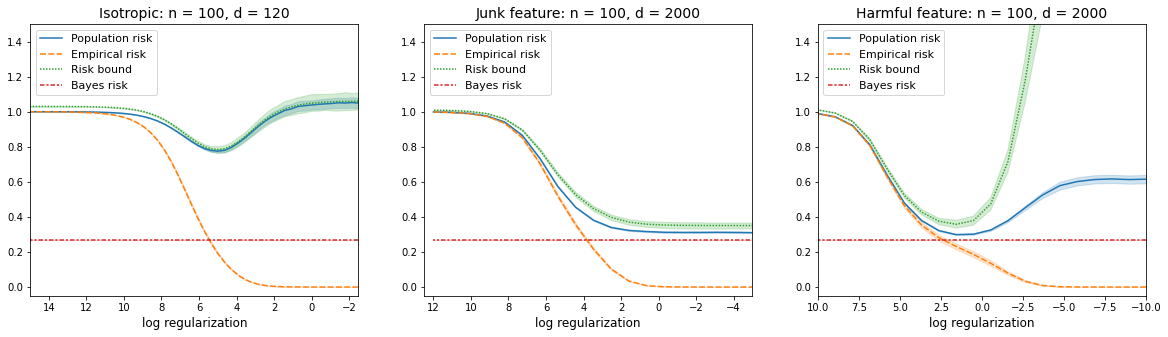}
  \caption{Top: ridge regression with model misspecification; bottom: $\ell_2$ margin classification with a logistic model.
  Here the features are Gaussian; see \cref{sec:experiment} for additional experiments on $\ell_1$ regularization and non-Gaussian features.
  The covariance in the first column (isotropic) is $\Sigma = I_d$,
  in the second column (junk features) is $\Sigma = \text{diag}(1, ..., 1, 0.05^2, ..., 0.05^2)$,
  and in the third (harmful features) is $\Sigma = \text{diag}(1, ..., 1, \frac{1}{(k+1)^2}, ..., \frac{1}{d^2}) $.
  For regression, the number of leading eigenvalues is $k = 3$, and the label is generated according to $y = 1.5 \, x_1 + |x_1| \cos x_2 + x_3 \cdot \cN(0, 0.5)$.
  For classification, the number of leading eigenvalues is $k = 1$ and $\Pr(y = 1 \mid x) = \operatorname{sigmoid}(5 x_1 + 3)$.
  The risk bound is calculated using the expression $\big( \sqrt{\hat{L}} + \sqrt{\|w\|_2^2 \Tr(\Sigma^{\perp})/n} \big)^2$, which corresponds to the choice of $C(w)$ from \cref{lem:gen-ball1}, and is expected to be tight in the junk features setting.
  In the isotropic case,  
  we use an easy improvement of the bound where $w$ is projected to the orthogonal complement of the Bayes predictor $w^*$ \citep{optimistic-rates}.
  In the other cases, we use covariance splitting without projection of $w$.
  Each point on the curve is the average from repeated experiments, and shaded areas correspond to 95\% bootstrap confidence intervals.
  }  
  \label{fig:flatness}
\end{figure}
\subsection{Further applications}
In this section, we discuss further interesting examples where our general theory applies. As before, we can obtain a generalization bound by appealing to the general \cref{corr:main-gen-vc}: we omit stating the formal corollary for each case, and simply show the Moreau envelope and its consequences.

\paragraph{$L_1$ loss (LAD) and Hinge.} 
For the $L_1$ loss
$f(\hat{y}, y) = | \hat{y} - y |$
 its Moreau envelope is given by
\[
f_{\lambda}(\hat{y}, y) = \inf_{u} \, | u-y | + \lambda ( u - \hat{y})^2 = 
\begin{cases}
	\lambda (\hat{y} - y)^2 &\mbox{if } |\hat{y} - y| \leq \frac{1}{2\lambda} \\
	|\hat{y} - y| - \frac{1}{4\lambda} &\mbox{if } |\hat{y} - y| > \frac{1}{2\lambda} \\
\end{cases}
,\]
which is $2\lambda$ times the Huber loss with parameter $\delta=\frac{1}{2 \lambda}$. Therefore, the population Huber loss is controlled by the empirical $L_1$ loss.
Clearly, interpolators have zero training error under  both $L_1$ and $L_2$ training  losses. 
We already see that Corollary~\ref{corr:gen-square-loss} implies
$(1 - o(1)) \E(\langle w, x \rangle + b - y)^2 \le C_{\delta}(w)^2/n$. 
Now, considering Corollary~\ref{corr:main-gen-vc} with $f$ the $L_1$ loss and using the above Moreau envelope calculation, we can check that when $\lambda \to 0$ we reproduce the exact same bound, %
since the Huber loss becomes the squared loss in the limit. 
Further insight into this phenomenon appears later in %
\cref{thm:training-error}: the Huber loss naturally shows up when computing the training error of the LAD estimator. An entirely analogous situation occurs when $f$ is the hinge loss $f(\hat y, y) := \max(0, 1 - \hat y y)$: its Moreau envelope $f_{\lambda}$ will be a rescaling of the Huber hinge loss (c.f.~\cite{zhang2004solving}).

\paragraph{Lipschitz loss: improved contraction.}
If $f$ is $M$-Lipschitz in $\hat y$, Proposition 3.4 of \citet{planiden2019proximal} gives that
$0 \leq f - f_{\lambda} \leq \frac{M^2}{4\lambda}$.
Thus, assuming $k/n = o(1)$, \cref{corr:main-gen-vc} implies
$(1 - o(1))(L_f(w) - \frac{M^2}{4\lambda}) \leq \hat{L}_{f}(w) + \frac{\lambda C_{\delta}(w)^2}{n}$;
optimizing over $\lambda$ yields
\begin{equation}\label{eqn:lipschitz-bd}
	(1 - o(1))L_f(w) \leq \hat{L}_f(w) + M \sqrt{\frac{ C_{\delta}(w)^2}{n}}.
\end{equation}
For $w \in \cK$, a high-probability upper bound on the Rademacher complexity of the function class $x \mapsto \langle w, x \rangle$ upper bounds the second term (Remark~\ref{rmk:rademacher}). 
In comparison, the standard symmetrization and contraction argument \citep{bartlett2002rademacher} loses a factor of two.
Note that if $f$ is the $L_1$ loss, this  applies with $M = 1$, and  it can further be shown that the constant factor cannot be improved further (see \cref{apdx:l1-sharpness}),
but this bound is also not as tight as the version with Huber test loss.

\paragraph{Uniform Convergence of Interpolators for Smooth Losses.} Suppose that $f$ is an $H$-smooth (in the sense that $\lvert \frac{\partial^2 f}{\partial \hat y^2} \rvert \le H$) and convex function of $\hat y$. In addition, assume that the minimum of $f(\hat y, y)$ is zero for any fixed $y$.
Since $f$ is $H$-smooth, there exists a function $\tilde{f}$ such that $f = \tilde{f}_{H/2}$ 
\citep[Corollary 18.18]{bauschke2011convex}. If $w$ satisfies $\hat{L}_f(w) = 0$, then $\hat{L}_{\tilde f}(w) = 0$ as well.\footnote{%
Since $f \le \tilde f$ is nonnegative,
$\tilde f$ is nonnegative as well.
When $f(\hat y, y) = 0$,
there must be $u_\varepsilon$ such that $\tilde f(u_\varepsilon, y) + \lambda(u_\varepsilon - \hat y)^2 < \varepsilon$;
this is only possible if we have $u_\varepsilon \to \hat y$,
implying since $f$ is smooth that $\tilde f(\hat y, y) = 0$.
If $\hat L_f(w) = 0$,
we must have for every $i$ that $f(\hat y_i, y_i) = 0$;
thus $\tilde f(\hat y_i, y_i) = 0$
and $L_{\tilde f}(w) = 0$.
}
Finally, by \cref{corr:gen-square-loss}, if $k / n = o(1)$ we have uniformly over all $w$ such that $\hat{L}_f(w) = 0$ that
\begin{equation}
	(1 - o(1)) L_f(w) \leq  \frac{H}{2}\cdot \frac{ C_{\delta}(w)^2}{n}
,\end{equation}
which generalizes the main result of \citet{uc-interpolators} to arbitrary smooth losses.

\section{\texorpdfstring{$\ell_2$}{l2} Benign Overfitting} \label{sec:benign-overfitting}
We can obtain conditions for consistency, like those of \citet{bartlett2020benign}, by simply combining a norm-based uniform generalization bound with an upper bound on the norm of interpolators.
\citet{uc-interpolators} used the same strategy,
but our more powerful and general tools give better results:
\begin{enumerate}
    \item For squared loss regression, we show that the assumption that the ground truth is generated by a linear model (i.e. well-specified), made in previous work, is not required. The same result holds under the much more general model\footnote{Theorem 3 of concurrent work by \citet{shamir2022implicit} also proves benign overfitting for some misspecified models, but requires the very strong assumption $n^2 \|\Sigma^{\perp}\|_{\OP} \to 0$ that fails to hold in several examples from \citet{bartlett2020benign}.} assumption of \cref{sec:model}.
    \item We show an analogous result in the \emph{classification} setting, replacing the squared loss with the squared hinge loss. 
    In fact, the argument is exactly the same in the two cases: all we need to use is that $f_{\lambda} = \frac{\lambda}{1 + \lambda} f$ and that $f$ is the square of a Lipschitz function. 
\end{enumerate}

First, the following lemma (essentially just the standard Rademacher complexity bound for the $\ell_2$ ball) combines with \cref{corr:gen-square-loss} and its squared hinge loss version to give generalization bounds.
These bounds are demonstrated in \cref{fig:flatness} and \cref{sec:experiment}.

\begin{restatable}{lemma}{GenBall} \label{lem:gen-ball1}
In the setting of Theorem~\ref{thm:main-gen}, letting $\Sigma^{\perp} = Q^T \Sigma Q$, the following $C_{\delta}(w)$ will satisfy \eqref{eqn:complexity-defn}:
$C_{\delta}(w) = \|w\|_2\left[\sqrt{\Tr(\Sigma^{\perp})} + 2\sqrt{\|\Sigma^{\perp}\|_{\OP} \log(8/\delta)}\right]$.
\end{restatable}

Next, we provide a sufficient condition for a zero-training error predictor $w$ to exist in an $\ell_2$ ball. In the case of classification, this allows us to lower-bound the margin of the max-margin halfspace. 

\begin{restatable}[\cite{bartlett2020benign}]{defn}{euclideffrank} \label{def:ranks-l2}
The \emph{effective ranks} of a covariance matrix $\Sigma$ are
\[
r(\Sigma) = \frac{\Tr(\Sigma)}{\norm{\Sigma}_\OP} \quad \text{and} \quad R(\Sigma) = \frac{\Tr(\Sigma)^2}{\Tr(\Sigma^2)}
.\]
\end{restatable}

\begin{restatable}{lemma}{NormBound} \label{lem:norm-bound-interpolator}
Suppose that $f(\hat y, y)$ is either squared loss or squared hinge loss. 
Let $(w^{\sharp}, b^{\sharp}) \in \mathbb{R}^{d+1}$ be an arbitrary vector 
satisfying $Qw^{\sharp} = 0$ and  with probability at least $1 - \delta/4$, 
\begin{equation}\label{eqn:low-dimensional-upper}
\hat{L}_f(w^{\sharp},b^{\sharp}) \le L_f(w^{\sharp},b^{\sharp}) + \rho_1(w^{\sharp},b^{\sharp}) \end{equation}
for some $\rho_1(w^{\sharp},b^{\sharp}) > 0$. Then for any $\rho_2 \in (0,1)$, provided $\Sigma^{\perp} = Q^T \Sigma Q$ satisfies
\begin{equation}\label{eqn:rho2}
R(\Sigma^{\perp}) = \Omega\left(\frac{n\log^2(4/\delta)}{\rho_2} \right),
\end{equation}
we have that with probability at least $1 - \delta$ that $\min_{\|w\| \le B} L_f(w,b^{\sharp}) = 0$
for $B > 0$ defined by $B^2 = \|w^{\sharp}\|_2^2 + (1 + \rho_2)\frac{n}{\Tr(\Sigma^{\perp})}(L_f(w^{\sharp},b^{\sharp}) + \rho_1)$.
\end{restatable}
We note that for any vector $w^{\sharp}$, we have $L_f((I - Q) w^{\sharp},b^{\sharp}) < L_f(w^{\sharp},b^{\sharp})$ by Jensen's inequality over $Q^T x$, so the assumption $Q w^{\sharp} = 0$ in the lemma is always satisfied for the minimizer of $L_f(w,b)$. 
Combining the norm bound \cref{lem:norm-bound-interpolator} and the generalization bound \cref{lem:gen-ball1} yields the following.

\begin{restatable}{theorem}{BenignOverfit} \label{thm:l2-overfitting}
Let $(\hat w, \hat b) = \arg\min_{w \in \mathbb{R}^d, b \in \mathbb{R} \,:\, \hat{L}_f(w,b) = 0} \|\hat w\|_2$ be the minimum-$\ell_2$ norm predictor with zero training error.
In the setting of \cref{lem:norm-bound-interpolator},
we have
\[ L_f(\hat w, \hat b) - \epsilon_{\delta}(\phi(\hat w),\hat b) \le (1 + \rho_3) \inf_{w^{\sharp} \in \mathbb{R}^d, b^{\sharp} \in \cB}\left(L_f(w^{\sharp},b^{\sharp}) + \rho_1(w^{\sharp},b^{\sharp}) +  \frac{\|w^{\sharp}\|_2^2\Tr(\Sigma^{\perp})}{n} \right) ,\]
where $\rho_3 > 0$ is defined by
$1 + \rho_3 = (1 + \rho_2)\left[1 + 2 \sqrt{\frac{\log(2/\delta)}{r(\Sigma^{\perp})}}\right]^2$
and we recall %
$\rho_1(w^{\sharp},b^{\sharp})$ from \eqref{eqn:low-dimensional-upper}.
\end{restatable}
We now show that this formally implies convergence to the optimal test loss (i.e.\ consistency) under the $\ell_2$ benign overfitting conditions \eqref{eqn:l2-benign-assumptions} from \citet{bartlett2020benign,tsigler2020benign}:

\begin{restatable}{corollary}{Consistency} \label{corr:benign-overfitting}
Suppose that $\mathcal{D}_n$ is a sequence of data distributions following our model assumptions (\ref{eqn:model}), with $k_n$ such that $y = g(\eta_1,\ldots,\eta_{k_n},\xi)$, and
projection operator $Q_n$ defined as in \eqref{eqn:Q-defn}.
Suppose $f$ is either the squared loss or the squared hinge loss,
and define $(w^{\sharp}_n,b^{\sharp}_n) = \arg\min_{w,b} L_{f,n}(w,b)$ where $L_{f,n}(w,b)$ is the population loss over distribution $\mathcal{D}_n$ with loss $f$. Suppose that the hypercontractivity assumption \eqref{eqn:hypercontractive-assumption} holds with some fixed $\tau > 0$ for all $\mathcal{D}_n$. 
Define $\Sigma_n := \E_{\mathcal{D}_n}[xx^T]$ and $\Sigma^{\perp}_n = Q_n^T \Sigma_n Q_n$. 
Suppose that as $n \to \infty$, we have
\begin{equation}\label{eqn:l2-benign-assumptions}
\frac{n}{R(\Sigma^{\perp}_n)} \to 0, \quad \frac{\|w^{\sharp}_n\|_2^2\Tr(\Sigma_n^{\perp})}{n} \to 0, \quad \frac{k_n}{n} \to 0 .
\end{equation}
Then we have the following convergence in probability, as $n \to \infty$:
\begin{equation} \label{eqn:consistency-to-sharp}
\frac{L_{f,n}(\hat w_n, \hat b_n)}{L_{f,n}(w^{\sharp}_n, b^{\sharp}_n)} \to 1 
,\end{equation}
where $(\hat w_n, \hat b_n) = \arg\min_{w \in \mathbb{R}^d, b \in \mathbb{R} : \hat{L}_f(w,b) = 0} \|w\|_2$
is the minimum-norm interpolator,
and $\hat{L}_{f,n}$ is the training error based on $n$ i.i.d. samples from the distribution $\mathcal{D}_n$.
\end{restatable}

Note when applying \cref{corr:benign-overfitting}, we have the flexibility to increase $k_n$ and shrink $\Sigma^{\perp}_n$ by choosing additional weights $w_i^*$ and letting the link function $g$ ignore the extra components. 

\begin{remark}[Flatness of the test loss along regularization path]
Our method can easily show a slightly stronger statement: let 
$(\hat w_n, \hat b_n) \in \arg\min_{\|w\| \le B_n, b \in \mathbb{R}} \hat{L}_{f,n}(w^{\sharp}_n, b^{\sharp}_n)$
such that, if there are multiple minima, we pick the one with smallest $\norm{w}$.
As long as $B_n \ge \|w^{\sharp}_n\|$, we still have \eqref{eqn:consistency-to-sharp}, and this is established uniformly over all sequences $B_n$ satisfying the constraint. Therefore, under the benign overfitting conditions we get consistency as long as we do not over-regularize the predictor. See \cref{fig:flatness} for an experimental demonstration of the flatness.
\end{remark}

\section{Training Error and Local Gaussian Width} \label{sec:localGW}
 \cref{thm:main-gen} shows how to upper-bound the test error of a predictor (under the Moreau envelope loss) by its training error and an upper bound on the class complexity. The following theorem
is the dual result, which upper-bounds the training error of the constrained ERM (Empirical Risk Minimizer) by the Moreau envelope and a complexity term. In particular, this general result is used to derive the norm bound for interpolators in Lemma~\ref{lem:norm-bound-interpolator} above.

\begin{restatable}{theorem}{LocalGW} \label{thm:training-error}
Let $\cK,\cB$ be bounded convex sets, and let $f(\hat y,y)$ be convex in $\hat y$. 
Suppose that $\tau$ is such that with probability at least $1 - \delta$, 
for $(\tilde x, \tilde y)_{i = 1}^n$ sampled i.i.d. from $\tilde{\cD}$ we have
\begin{equation}\label{eqn:training-error-assumption}
\min_{\tilde w \in \phi(\cK), b_0 \in \cB} 
\max_{\lambda \ge 0} \left[\frac{1}{n} \sum_{i = 1}^n f_{\lambda}(\langle \tilde w, \tilde x \rangle + b_0, y_i) - \frac{\lambda}{n} \max_{w_0 \in \phi^{-1}(\tilde w) \cap \cK}  \langle x, Q w_0 \rangle^2\right] \le \tau.
\end{equation}
Then with probability at least $1 - 2 \delta$, $\min_{w \in \cK, b \in \cB} \hat{L}_f(w,b) \le \tau$.
\end{restatable}

Note that the assumption \eqref{eqn:training-error-assumption} implicitly suggests a low-dimensional concentration assumption: we expect $\frac{1}{n} \sum_{i = 1}^n f_{\lambda}(\langle \tilde w, \tilde x \rangle + b_0, y_i)$ to be approximately the test loss of $(\tilde w,b_0)$ under the surrogate distribution $\tilde{\mathcal D}$. 
As we discuss more in \cref{apdx:training-error}, combining this training error bound with the correct choice of $C_{\delta}(w)$ in Theorem~\ref{thm:main-gen}, which is essentially $C_{\delta}(w) = \E \max_{w_0 \in \phi^{-1}(\tilde w) \cap \cK}  \langle x,  Q w_0 \rangle^2$), yields a matching lower bound to \eqref{eqn:main-bound} on the Moreau envelope test loss and so our generalization bound is asymptotically sharp. This establishes a non-asymptotic analogue of the existing asymptotic Moreau envelope theory (see \cref{sec:related}), and recovers the special case of well-specified linear models \citep{optimistic-rates}. 

\section{Discussion}

In this work, we significantly extend the localized uniform convergence technique developed in the study of noisy interpolation to any loss function and label generating process under mild conditions. Though the application of Moreau envelope to study GLMs is not new in the statistical literature, our general theory establishes novel non-asymptotic generalization bounds in a wide variety of overparameterized settings. We believe the generality of our framework may allow further applications in other areas of statistics, such as robust statistics and high-dimensional inference.

As mentioned in \cref{sec:related}, the applicability of our theory is still considerably limited by the Gaussian data assumption,
required by our use of the Gaussian minimax theorem.
It does appear experimentally that it may hold much more broadly (\cref{sec:experiment});
proving that this is the case could allow us to study kernel methods and bring us closer to a theoretical understanding of deep neural networks. 
Some work has been done in related settings to extend Gaussian-based results to broader distributions via universality arguments \citep[e.g.][]{hu:universality,liang2020precise,montanari:universal-erm},
but it is not yet clear how to apply those techniques to our general framework.
The GMT formulation also does not allow for multi-class classification or two-layer networks, because of their vector-valued outputs. Overcoming these two challenges seems to be crucial avenues for future work. 

\begin{ack}
F.K.\ was supported in part by NSF award CCF-1704417, NSF award IIS-1908774, and N. Anari’s Sloan Research Fellowship.
P.S.\ was supported in part by NSF award DMS-2113426.
D.J.S.\ was supported in part by the Canada CIFAR AI Chairs program.
Part of this work was initiated when F.K., P.S., and N.S.\ were visiting the Simons Institute for the Theory of Computing for their program on Computational Complexity of Statistical Inference. 
This work was done as part of the
Collaboration on the Theoretical Foundations of Deep Learning (\httpsurl{deepfoundations.ai}).
\end{ack}

\printbibliography

\clearpage
\appendix
\section{Organization of the Appendices}
In this appendix, we provide additional simulation results and complete proofs of all the results in the main text.
In \cref{sec:experiment}, we provide additional simulation results.
In \cref{apdx:preliminaries}, we introduce standard notation and tools which we use throughout the remainder of the appendices.
In \cref{apdx:proof-main-gen}, we give a proof of our main result \cref{thm:main-gen}.
In \cref{apdx:applications}, we apply VC theory to handle low-dimensional concentration and prove the generalization guarantees for linear regression and classification.
In \cref{apdx:training-error}, we prove \cref{thm:training-error}.
In \cref{apdx:benign-overfitting}, we establish a norm bound for interpolators and apply our generalization bound of \cref{sec:applications} to show consistency.

\section{Additional Numerical Simulations} \label{sec:experiment}

This section presents additional numerical simulations on synthetic data to confirm our theory and test it beyond the case of Gaussian covariates. All code is available from \url{https://github.com/zhoulijia/moreau-envelope}.\footnote{The ridge path is computed using SVD implemented by \texttt{np.linalg.svd}. The LASSO path is computed using coordinate descent implemented by \texttt{sklearn.linear\_model.lasso\_path}, and $\ell_1$ and $\ell_2$ margin classifiers are fitted using \texttt{sklearn.svm.LinearSVC} with the default squared hinge loss option.} 

\subsection{Linear Regression}

We fit linear models to minimize the square loss with $\ell_1$ and $\ell_2$ penalty. For simplicity, we ignore the intercept term in this section, but we will consider models with intercept in the context of linear classification. We can obtain many data distributions by combining the different options below:

\paragraph{Feature Distribution.} The marginal distribution of $x$ is always given by $x = \Sigma^{1/2} z$, where $z$ is a random vector with i.i.d. coordinates that have mean 0 and variance 1. 

\begin{figure}[htbp]
\centering
\hspace{12ex}
\includegraphics[width = 9.5cm]{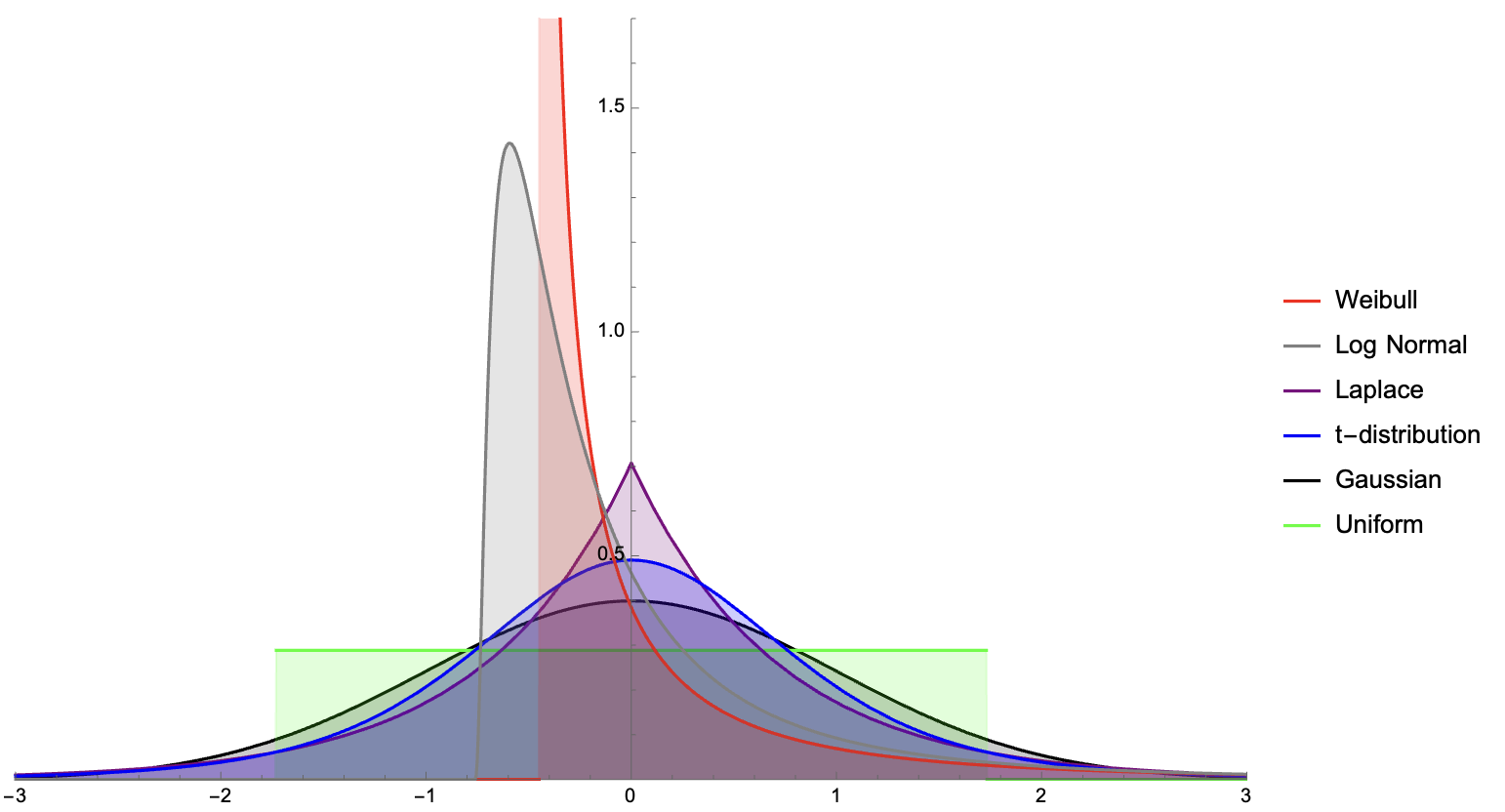}
\caption{Probability density plot for the continuous distributions of $z$ that we consider.}
\end{figure}

The coordinate distributions of $z$ that we consider in the simulations include
\begin{itemize}[wide, labelwidth=0pt, labelindent=0pt]
    \item Gaussian  
    \begin{itemize}
        \item the standard Gaussian distribution has density $p(z) = \frac{1}{\sqrt{2 \pi}} e^{-\frac{1}{2} z^2}$
    \end{itemize}
    
    \item Uniform
    \begin{itemize}
        \item the uniform distribution between 0 and 1 has mean 0 and variance $\frac{1}{12}$. After normalization, it becomes the uniform distribution between $-\sqrt{3}$ and $\sqrt{3}$. It's symmetric, bounded from above and below, and therefore sub-Gaussian
    \end{itemize}

    \item Laplace
    \begin{itemize}
        \item Laplace distribution with scale parameter $b$ has density $p(z) = \frac{1}{2b} e^{-\frac{|z|}{b}}$ and variance $2b^2$, so we should choose $b = \frac{1}{\sqrt{2}}$
        \item it is symmetric, unbounded, and has fatter tails compared to Gaussian (sub-exponential)
        
    \end{itemize}
\end{itemize}

We also consider discrete distributions
\begin{itemize}[wide, labelwidth=0pt, labelindent=0pt]
    \item Rademacher 
    \begin{itemize}
        \item the discrete distribution with equal chance of being $-1$ or $1$. It is easy to see that it has mean 0 and unit variance.
    \end{itemize}
    
    \item Poisson
    \begin{itemize}
        \item Poisson distribution with rate parameter 1 is supported on the non-negative integers (skewed and bounded from below) and has density $\pr(\tilde{z} = k) = \frac{e^{-1}}{k!}$. Its mean and variance are both equal to 1, and so we take $z = \tilde{z} - 1$ to normalize
    \end{itemize}
\end{itemize}

and heavy-tailed distributions
\begin{itemize}[wide, labelwidth=0pt, labelindent=0pt]
    \item Student's t-distribution
    \begin{itemize}
        \item t-distribution with 5 degrees of freedom has density $p(\tilde{z}) = \frac{8}{3\sqrt{5}\pi \left( 1 + \frac{\tilde{z}^2}{5}\right)^3}$
        \item It has variance $\frac{5}{3}$ and so we let $z =\sqrt{\frac{3}{5}} \tilde{z}$. It is symmetric, unbounded and has finite fourth moment. However, moments of order 5 or higher do not exist.
    \end{itemize}
    
    \item Weibull
    \begin{itemize}
        \item Weibull distribution with scale parameter $\lambda = 1$ and shape parameter $k = 0.5$ has density $p(\tilde{z}) = \frac{e^{-\sqrt{\tilde{z}}}}{2\sqrt{\tilde{z}}} \bone_{\{ \tilde{z} \geq 0\}}$. It has mean 2 and variance 20 and so we take $\tilde{z} = \frac{z - 2}{\sqrt{20}}$
    \end{itemize}
    
    \item Log-Normal
    \begin{itemize}
        \item the distribution of $e^Z$, where $Z$ follows the standard Gaussian distribution. It has mean $\sqrt{e}$ and variance $e(e-1)$, and so we can choose $z = \frac{e^Z - \sqrt{e}}{\sqrt{e(e-1)}}$
    \end{itemize}
\end{itemize}

\paragraph{Covariance Matrix and Scaling.} For simplicity, we choose $\Sigma$ to be diagonal and consider
\begin{itemize}[wide, labelwidth=0pt, labelindent=0pt]
    \item Isotropic features $\Sigma = I_d$ in the proportional scaling $(n = 300, d = 350)$
    \item Junk features in the over-parameterized scaling $(n = 300, d = 3000)$
    \[
    \Sigma_{kk} = 
    \begin{cases}
        1 &\mbox{if } k = 1, 2, 3 \\
        0.05^2 &\mbox{otherwise} \\
    \end{cases}
    \]
    \item Non-benign features in the over-parameterized scaling $(n = 300, d = 3000)$
    \[
    \Sigma_{kk} = 
    \begin{cases}
        1 &\mbox{if } k = 1, 2, 3 \\
        \frac{1}{k^2} &\mbox{otherwise} \\
    \end{cases}
    \]
\end{itemize}

The junk features setting is known to satisfy the benign overfitting conditions \citep{junk-feats, bartlett2020benign}, by which the minimal $\ell_2$-norm interpolator is consistent. In contrast, \citet{bartlett2020benign} also shows that overfitting is not benign in the second case, but the theory from \citet{optimistic-rates} shows that the optimally-tuned ridge regression can be consistent.

\paragraph{Conditional Distribution of $y$.} Let 
\begin{equation*}
    \begin{split}
        w^* &= (1.5, 0, ..., 0) \\
        \xi &\sim \cN(0, 0.5) \\
    \end{split}
\end{equation*}

and consider
\begin{itemize}[wide, labelwidth=0pt, labelindent=0pt]
    \item a well-specified linear model:
    \[
    y = \langle w^*, x \rangle + \xi
    \]
    \item a mis-specified model:
    \[
    y = \underbrace{\langle w^*, x \rangle}_{\text{linear signal}} + \underbrace{|x_1| \cdot \cos x_2 }_{\text{non-linear term}} + \underbrace{ x_3 \cdot \xi}_{\text{heteroscedasticity}}
    \]
\end{itemize}
The second model does not satisfy the classical assumptions for linear regression because the Bayes predictor
\[
\E[y | x] = \langle w^*, x \rangle + |x_1| \cdot \cos x_2
\]
is non-linear and the variance of the residual also depends on $x_4$. Even though statistical inference can be challenging for models like this, we can hope to learn a model that competes with the optimal linear predictor (which is not necessarily the same as $w^*$) in terms of prediction error. 

\subsubsection{Speculative Risk Bounds for Non-Gaussian Features}

Though our theory is restricted to Gaussian features, we conjecture that it can be extended to a more general class of distributions using Rademacher complexity and we use numerical simulations to confirm our conjecture.

\paragraph{Ridge Regression}
\begin{enumerate}
    \item Isotropic features: similar to Lemma 10 in \citet{optimistic-rates}, we can choose $C_{\delta}$ in corollary~\ref{corr:gen-square-loss} by the simple Cauchy-Schwarz bound
    \[
    \langle Qw, x \rangle \leq \| Qw\|_2 \cdot \| x\|_2 \approx \sqrt{d}  \| Qw\|_2
    \]
    resulting in the following bound
    \begin{equation} \label{eqn:ridge-bound-isotropic}
        L_f(w) \leq (1 + o(1)) \left(\sqrt{\hat{L}_f(w)} + \sqrt{\frac{d}{n}} \cdot \| Qw\|_2 \right)^2
    \end{equation}
    
    \item Junk and non-benign features: choosing $C_{\delta}$ in corollary~\ref{corr:gen-square-loss} according to Lemma~\ref{lem:gen-ball1} yields 
    \begin{equation} \label{eqn:ridge-bound}
        L_f(w) \leq (1 + o(1)) \left(\sqrt{\hat{L}_f(w)} + \|w\|_2 \sqrt{\frac{\Tr(\Sigma^{\perp})}{n}} \right)^2
    \end{equation}
\end{enumerate}

In all of the experiments, we use a constant close to 1 to replace the $1 + o(1)$ factor in our generalization bounds. Note that \eqref{eqn:ridge-bound} can be interpreted in terms of Rademacher complexity:
\begin{equation*}
    \begin{split}
        \E_{\substack{x_1, ..., x_n \sim \cD \\ s \sim \text{Unif}(\{ \pm 1\}^n)}} \left[ \sup_{\| w\|_2 \leq B} \left| \frac{1}{n} \sum_{i=1}^n s_i \langle w, Q^T x_i \rangle \right| \right]
        &= \frac{B}{n} \cdot \E_{\substack{x_1, ..., x_n \sim \cD \\ s \sim \text{Unif}(\{ \pm 1\}^n)}} \left[ \norm{ \sum_{i=1}^n s_i Q^T x_i }_2 \right] \\
        &\leq B \cdot \sqrt{\frac{\Tr(\Sigma^{\perp})}{n}}
    \end{split}
\end{equation*}

The last inequality holds generally for any distribution with $\E_{x \sim \cD} [xx^T] = \Sigma$ by Cauchy-Schwarz inequality. In our examples, $x = \Sigma^{1/2} z$ and $z$ is scaled to satisfy $\E [zz^T] = I_d$. Therefore, we will use equation \eqref{eqn:ridge-bound-isotropic} and \eqref{eqn:ridge-bound} even for non-Gaussian data. 

Equations \eqref{eqn:ridge-bound-isotropic} and \eqref{eqn:ridge-bound} are qualitatively similar with subtle technical differences. Compared with equation \eqref{eqn:ridge-bound}, the bound \eqref{eqn:ridge-bound-isotropic} uses the smaller norm $\| Qw\|_2$ and figure 2 of \citet{uc-interpolators} demonstrates that this projection is crucial for tight bounds in the isotropic setting. On the other hand, equation \eqref{eqn:ridge-bound} incorporates the covariance splitting technique \citep{bartlett2020benign} because large eigenvalues of $\Sigma$ can be killed off in $\Sigma^{\perp}$ by projection $Q$ while $\Tr(I_d) = d$ in the isotropic case. It is shown in our corollary~\ref{corr:benign-overfitting} that this bound without the projection is already tight enough to establish the consistency of minimal-$\ell_2$ norm interpolator in the junk feature setting. Hence, we expect \eqref{eqn:ridge-bound} to be tight throughout the ridge path. In contrast, the theory in \citet{optimistic-rates} predicts that \eqref{eqn:ridge-bound} is tight for the non-benign setting only up to the point where the ridge estimate has norm as large as the optimal linear predictor. We believe using the local Gaussian width theory introduced in Section~\ref{sec:localGW} (i.e. an optimal choice of $C_{\delta}(w)$) can get tight bound throughout the ridge path in this setting, but we do not have experiments in this appendix to confirm it.

In the theoretical analysis of \citet{optimistic-rates}, they further write $\| Qw\|_2$ as a function of $\|w\|_2, \| w^* \|_2$ and the excess risk $\| w-w^*\|_{\Sigma}^2$ in the isotropic case, then solve the equation in terms of $\| w-w^*\|_{\Sigma}^2$ to get a norm-based generalization bound as a function of $\| w\|_2$ when $\hat{L}_f(w) = 0$ (see their theorem 6). Since the solution for general non-zero $\hat{L}_f(w)$ can have a quite tedious expression, for the purpose of numerically checking the applicability and tightness of this approach, we will use simpler equation \eqref{eqn:ridge-bound-isotropic} in the experiments.

\paragraph{LASSO Regression} 
Similar to the section above, we use the analogy to Rademacher complexity to extend our theory to the $\ell_1$ case. Since we can no longer bound the $\ell_{\infty}$ norm of a sum using the Cauchy-Schwarz inequality, it is easier to directly work with the empirical Rademacher complexity (which also should be similar to the expected Rademacher complexity in the settings that we consider)
\[
\frac{\| w\|_1}{n} \cdot \E_{ s \sim \text{Unif}(\{ \pm 1\}^n)} \left[ \norm{ \sum_{i=1}^n s_i Q^T x_i }_{\infty} \right] 
\]
and we can estimate the expected norm by
\[
\frac{1}{B} \sum_{k = 1}^B \norm{ \sum_{i=1}^n s_{k, i} Q^T x_i }_{\infty}
\]
for a large value of $B$ and $s_1, ..., s_B$ sampled independently from $\text{Unif}(\{ \pm 1\}^n)$. In our implementation, $s_1, ..., s_B$ are fresh samples each time the risk bound is computed. To summarize, we use the following expression for the calculation of risk bound:
\begin{enumerate}
    \item Isotropic features:
    \begin{equation} \label{eqn:lasso-bound-isotropic}
        \left(\sqrt{\hat{L}_f(w)} + \| Qw\|_1 \cdot \frac{1}{nB} \sum_{k = 1}^B \norm{ \sum_{i=1}^n s_{k, i}  x_i }_{\infty} \right)^2
    \end{equation}
    
    \item Junk and non-benign features: 
    \begin{equation} \label{eqn:lasso-bound}
        \left(\sqrt{\hat{L}_f(w)} + \| w\|_1 \cdot \frac{1}{nB} \sum_{k = 1}^B \norm{ \sum_{i=1}^n s_{k, i}  Q^T x_i }_{\infty} \right)^2
    \end{equation}
\end{enumerate}
which are analogous to \eqref{eqn:ridge-bound-isotropic} and \eqref{eqn:ridge-bound}. 

We note that it is important to use the Rademacher complexity to extend to non-Gaussian features in the $\ell_1$ case, rather than a bound similar to $\frac{\| w\|_1 \E \| x\|_{\infty}}{\sqrt{n}}$. Empirically, the latter is too small to provide a valid upper bound on the test loss. This is because $\| x\|_{\infty}$ is deterministic for distributions like the Rademacher distribution, while the random signs in the definition of Rademacher complexity allows a tail behavior more similar to Gaussian and so we can regain a log factor in the norm component. 

\subsubsection{Experimental Results}

For both ridge and LASSO regression, risk curves measured in the square loss are shown in three figures corresponding to the different data covariances. Within each figure, there are 16 subplots corresponding to the different combinations of one of the eight feature distributions and label generating process (well-specified vs mis-specified) as defined at the beginning of the section. Therefore, there are 96 subplots in total. Discussion of the experimental outcome can be found in the caption of each figure.

Similar to the situation in the rest of the experiments, the training error is close to 0 with sufficiently small regularization, and the confidence bands are wider with heavy-tailed distributions. Also, the null risk and the Bayes risk are different across different feature distributions when there is model misspecification (see the calculation in the next subsection for more details).

\paragraph{Ridge Regression.}

The plots for isotropic, junk and non-benign features in the ridge regression setting can be found in figures \ref{fig:isotropic-ridge}, \ref{fig:junk-ridge} and \ref{fig:non-benign-ridge}, respectively. Generally speaking, the experiments confirm the tightness and wide applicability of our generalization guarantees. The specific feature distribution and model misspecification do not seem to affect the shape of test error curve.

\paragraph{LASSO Regression.} The plots for isotropic, junk, and non-benign features in the LASSO regression setting can be found in \cref{fig:isotropic-lasso,fig:junk-lasso,fig:non-benign-lasso}. The risk bounds in the $\ell_1$ case are not as tight as in the $\ell_2$ case because they are only expected to be tight in certain parts of the entire regularization path. As mentioned earlier, we can get sharp bounds for the entire path using local Gaussian width, but it requires a more fine-grained analysis than \eqref{eqn:lasso-bound-isotropic} and \eqref{eqn:lasso-bound}. Similar results and experiments were obtained by \citet{wang2021tight,donhauser2022fastrate}.

\begin{figure} 
  \centering
  \includegraphics[width = \linewidth]{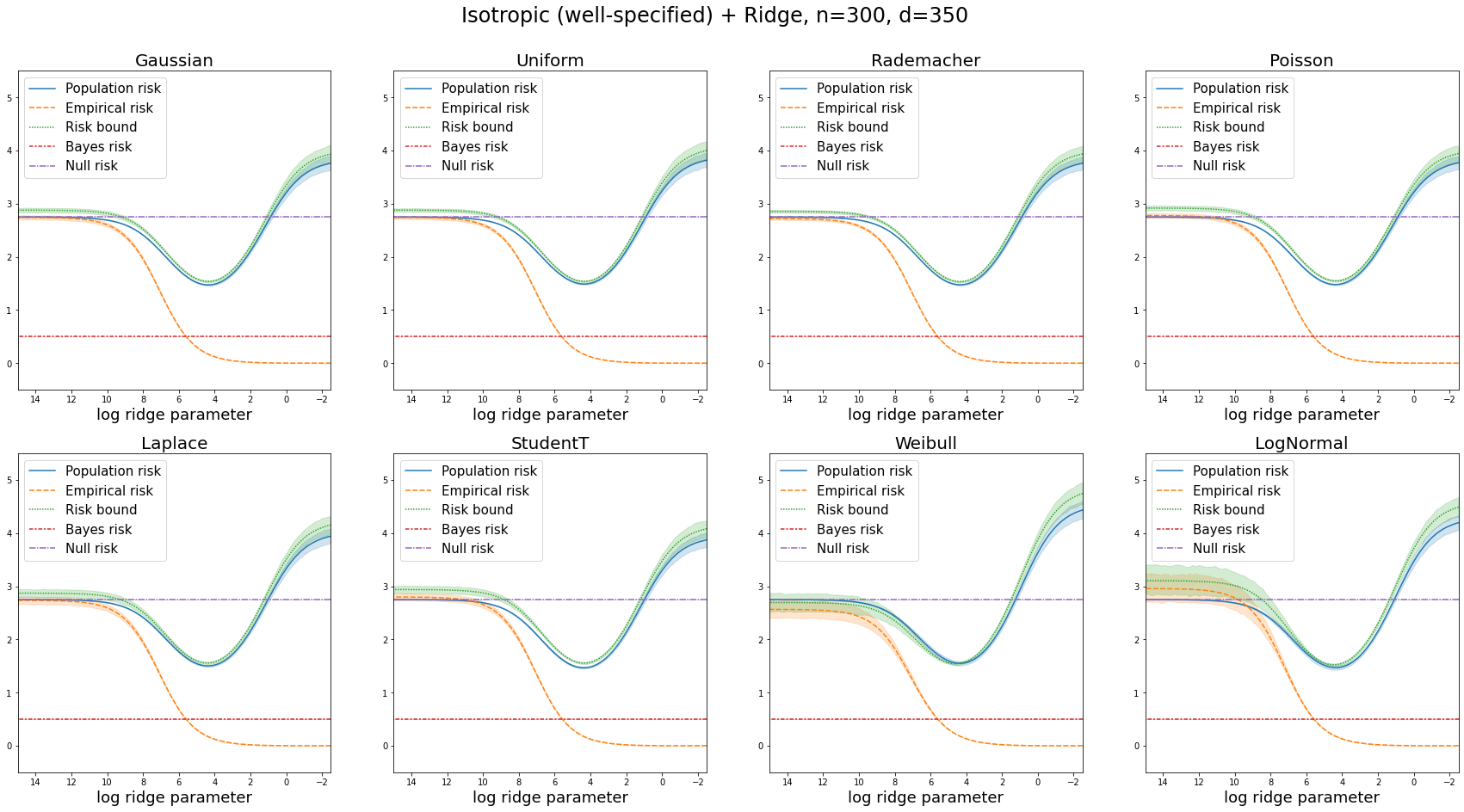}\\
  \vspace{6ex}
  \includegraphics[width = \linewidth]{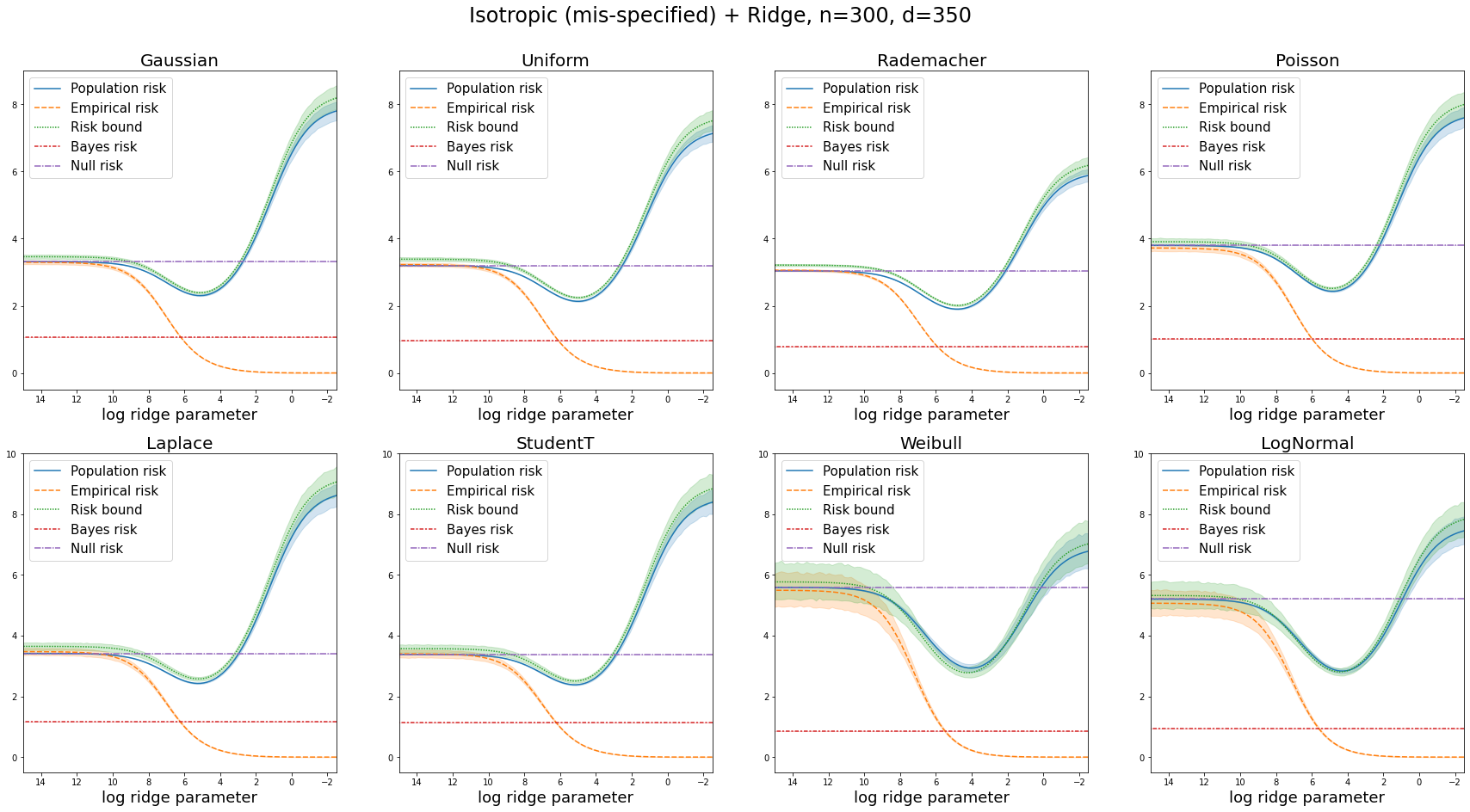}
  \vspace{2ex}
  \caption{Ridge regression with isotropic data $(n = 300, d = 350)$. As proved by theorem 7 in \citet{optimistic-rates}, the risk bound \eqref{eqn:ridge-bound-isotropic} follows the test error curve closely. This is true even in the non-Gaussian and mis-specified settings. Note that we do not have benign-overfitting because we are in the proportional scaling regime with $d$ close to $n$, and the population risk of the minimal-$\ell_2$ norm interpolator is even worse than the null-risk (more significantly so with misspecification). The optimally-tuned ridge regression has risk better than the null risk, but it is still far from the Bayes risk because the consistency result of optimally-tuned ridge regression in \citet{optimistic-rates} assumes $\Tr(\Sigma)/n \to 0$.}
  \label{fig:isotropic-ridge}
\end{figure}

\begin{figure} 
  \centering
  \includegraphics[width = \linewidth]{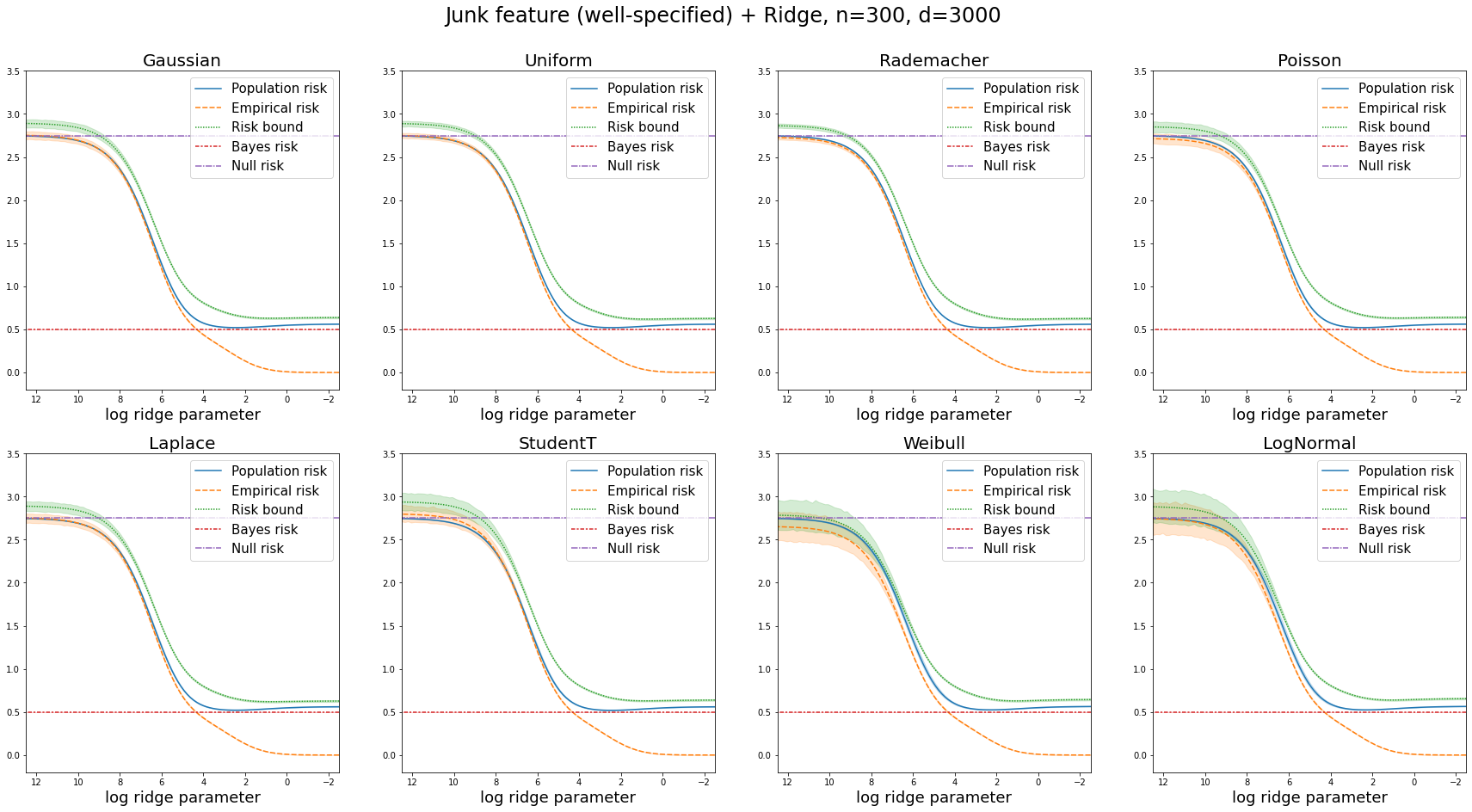}\\
  \vspace{6ex}
  \includegraphics[width = \linewidth]{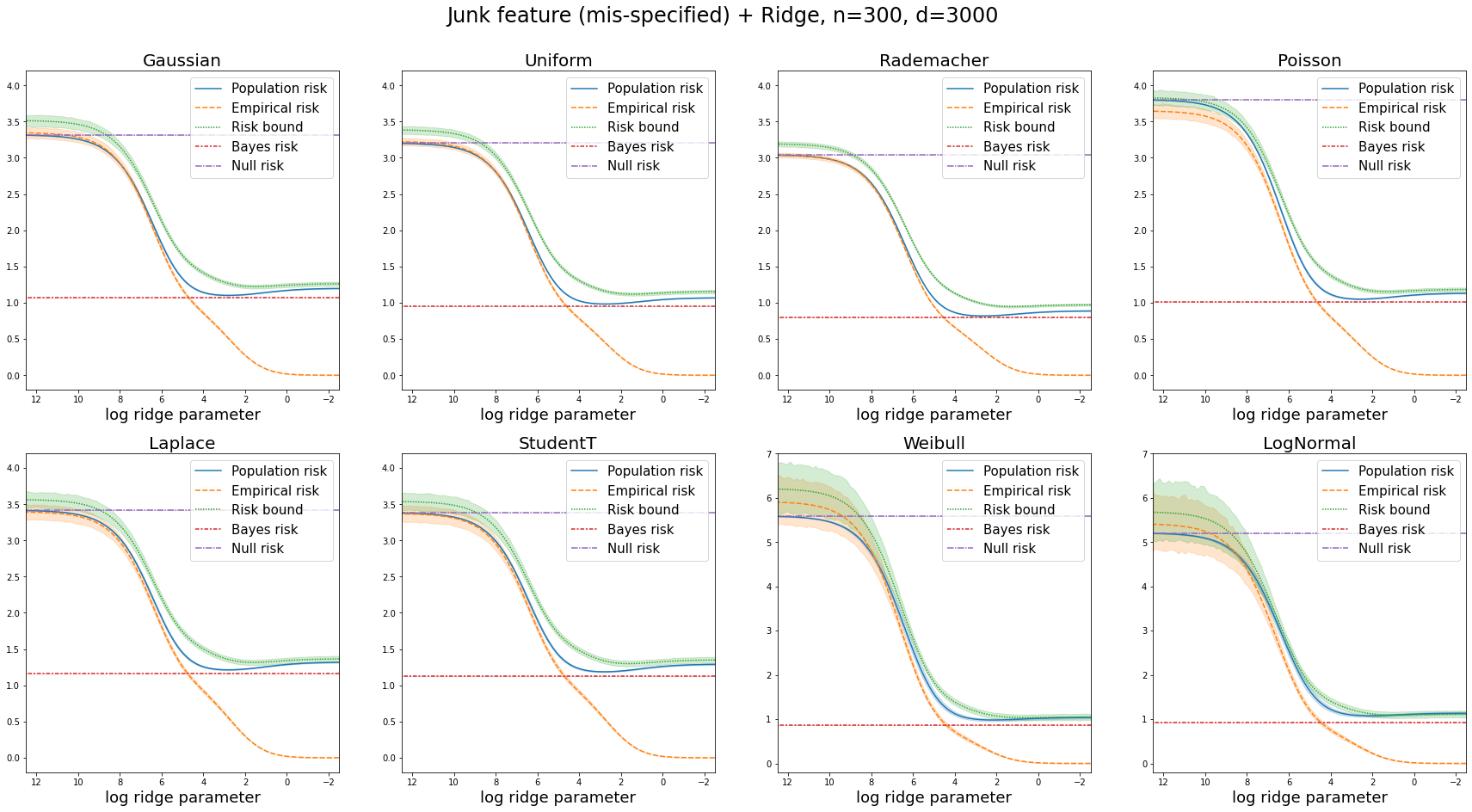}
  \vspace{2ex}
  \caption{Ridge regression with junk features $(n = 300, d = 3000)$. In the junk features setting, as predicted in section~\ref{sec:benign-overfitting}, the test error curve is essentially flat once the regularization is small enough to fit the signal, and we get nearly optimal population risk as long as we do not over-regularize the predictor. The test error curve can be expected to be more flat with increasing $d$. This phenomenon is also consistent across different feature distributions and label generating processes. Our bound \eqref{eqn:ridge-bound} closely tracks the performance of ridge regression along the entire regularization path.}
  \label{fig:junk-ridge}
\end{figure}

\begin{figure} 
  \centering
  \includegraphics[width = \linewidth]{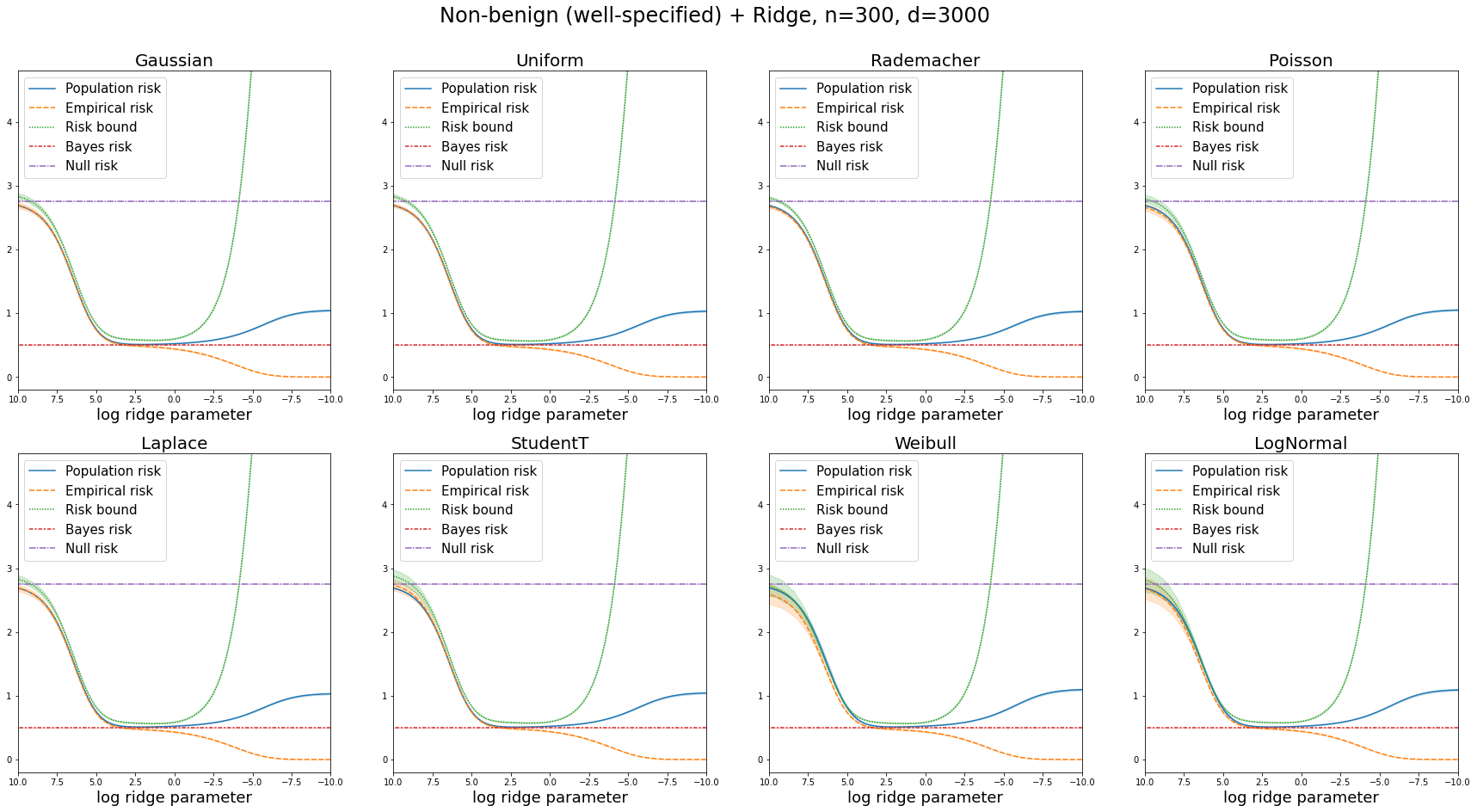}\\
  \vspace{6ex}
  \includegraphics[width = \linewidth]{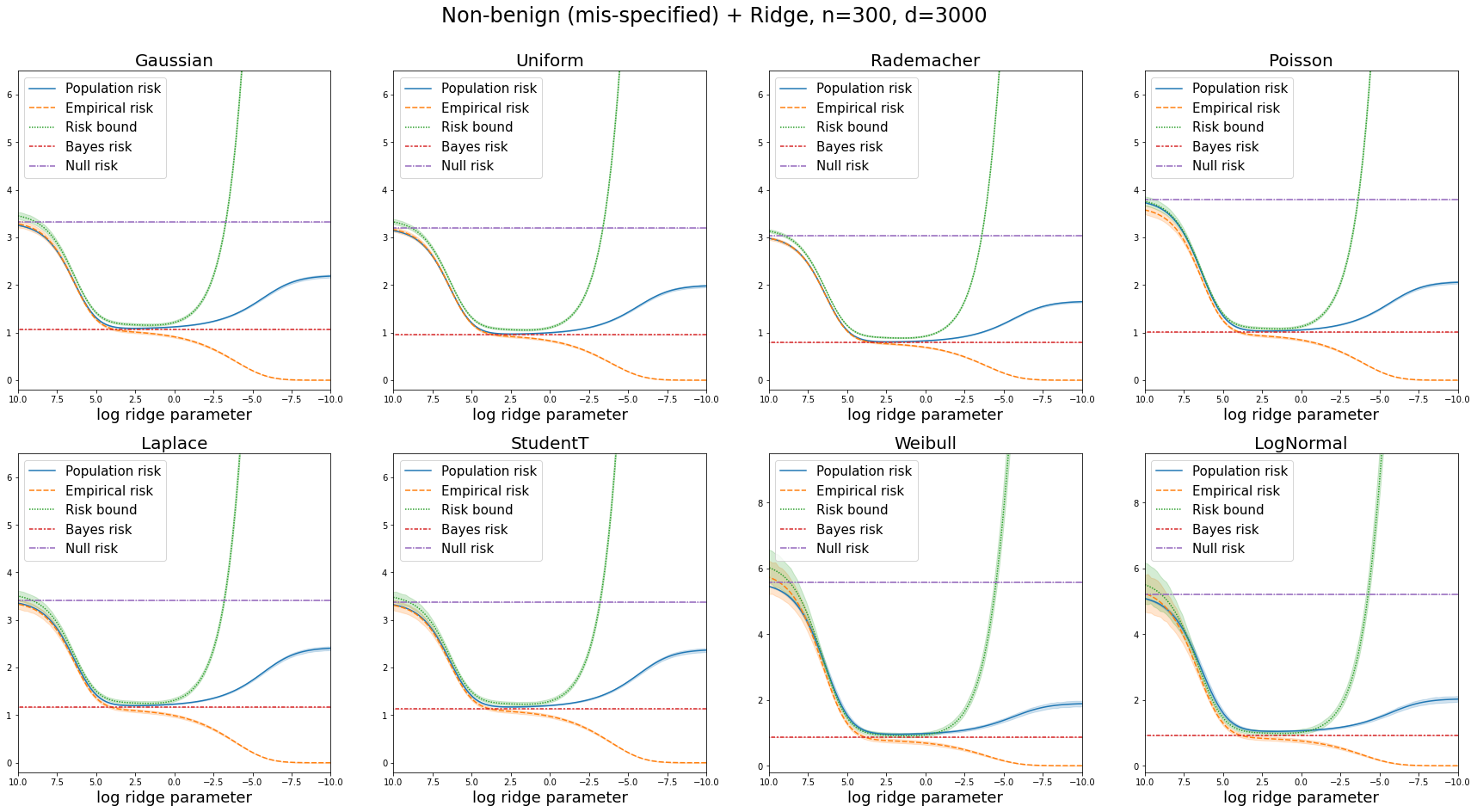}
  \vspace{2ex}
  \caption{Ridge regression with non-benign features $(n = 300, d = 3000)$. In the non-benign features setting, as proved by corollary 3 in \citet{optimistic-rates}, the optimally-tuned ridge regression achieves nearly optimal prediction risk. Our risk bound is tight up to the point up to the point where the test error starts to increase. As expected, the minimal norm interpolator fails to achieve consistency even though we are in the overparameterized regime. Note that bound \eqref{eqn:ridge-bound} is dramatically more pessimistic in the under-regularized part of the ridge path. Once again, the data distribution and model misspecification has no effect on the shape of the test error curve and risk bound.}
  \label{fig:non-benign-ridge}
\end{figure}

\begin{figure} 
  \centering
  \includegraphics[width = \linewidth]{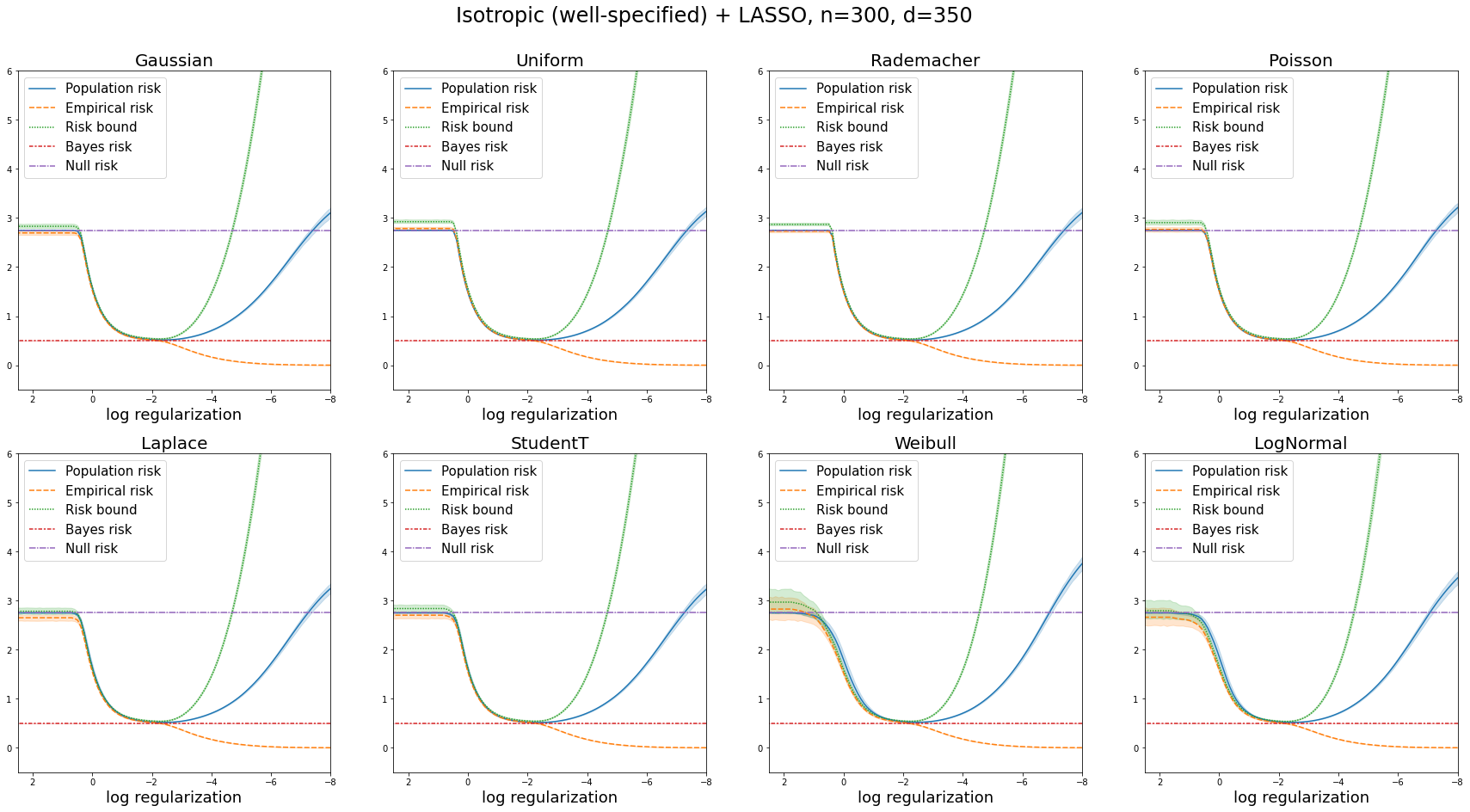}\\
  \vspace{6ex}
  \includegraphics[width = \linewidth]{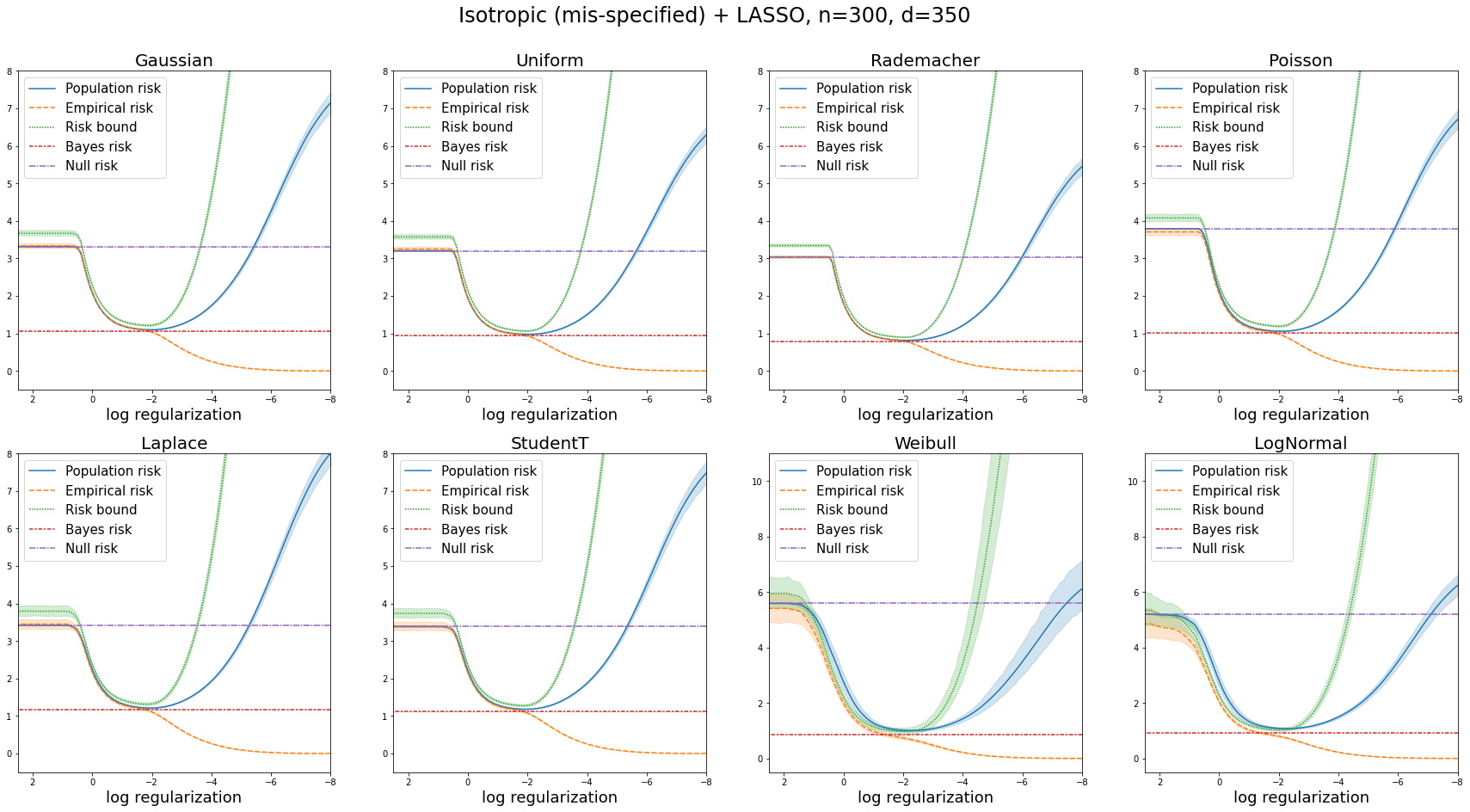}
  \vspace{2ex}
  \caption{LASSO regression with isotropic data $(n = 300, d = 350)$. Contrary to the inconsistency of optimally-tuned ridge regression in this setting, the regularized LASSO estimator can achieve nearly optimal population risk thanks to sparsity. The risk bound \eqref{eqn:lasso-bound-isotropic} appears to be valid and sufficient for the consistency of optimal LASSO in the distributions that we consider, though it is not very tight for interpolation. Recall that the minimal-$\ell_1$ norm interpolator suffers from an exponentially slow convergence rate when $d = n^{\alpha}$ \citep{wang2021tight} and observe that the population risk of the minimal-$\ell_1$ norm interpolator is again worse than the null-risk. }
  \label{fig:isotropic-lasso}
\end{figure}

\begin{figure} 
  \centering
  \includegraphics[width = \linewidth]{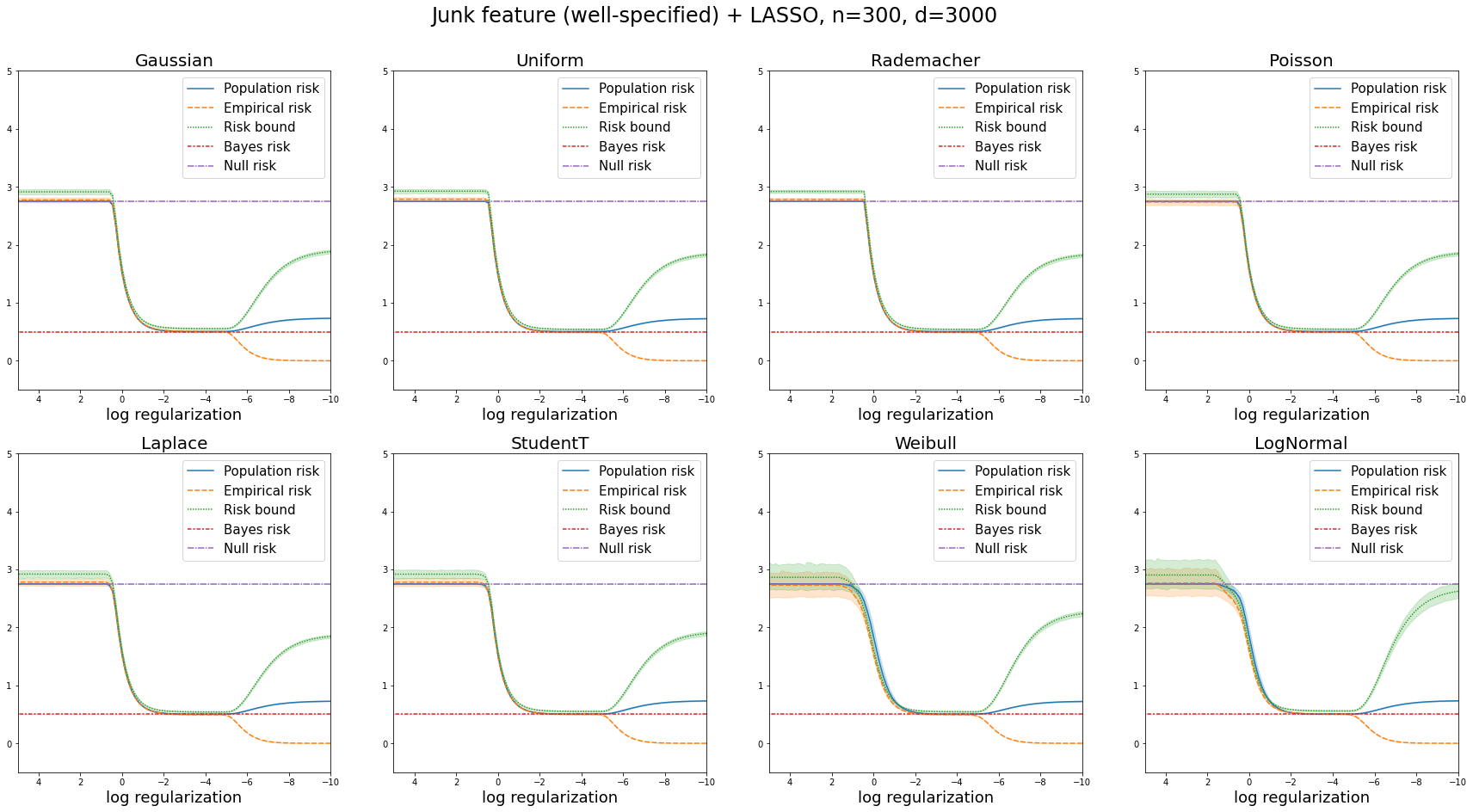}\\
  \vspace{6ex}
  \includegraphics[width = \linewidth]{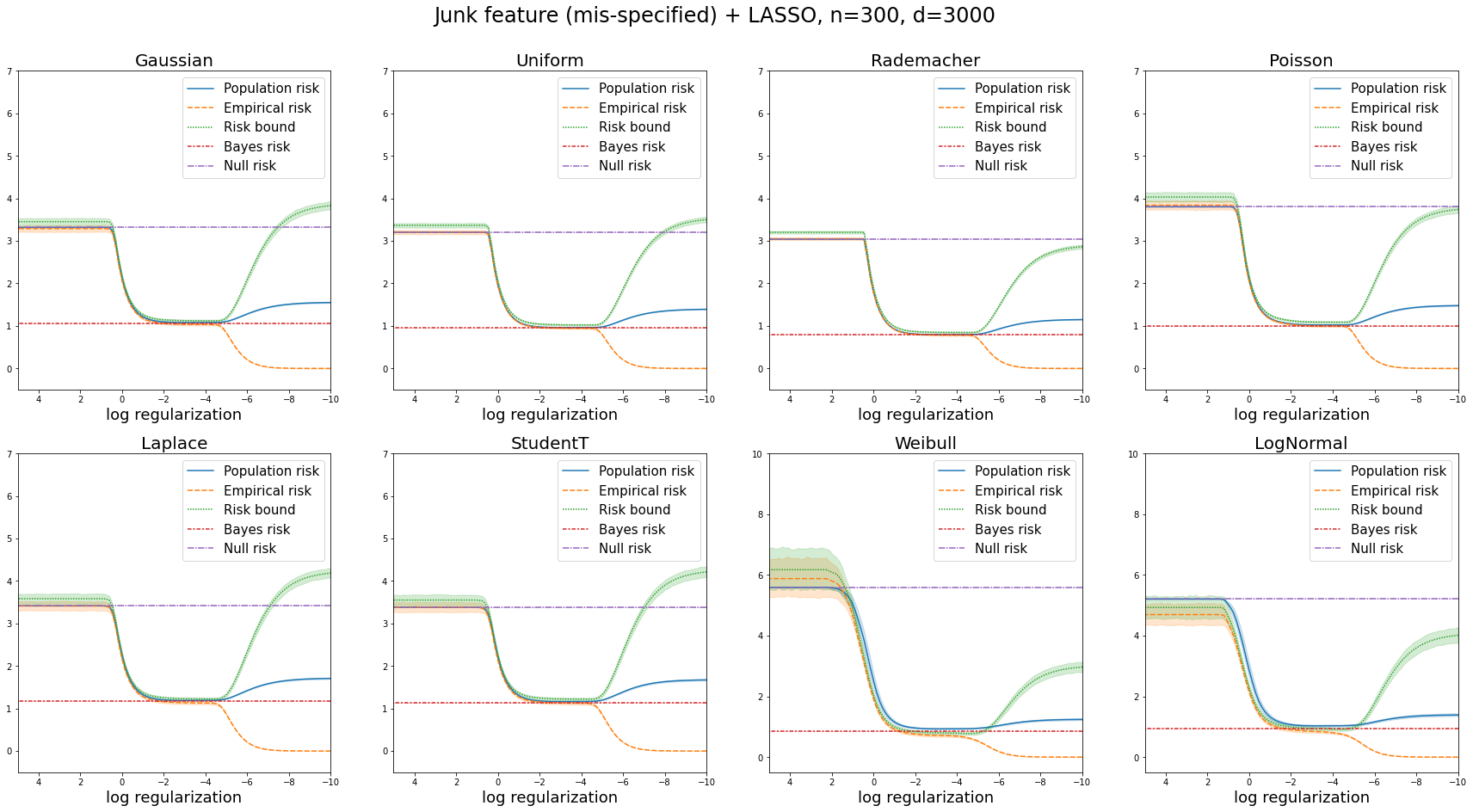}
  \vspace{2ex}
  \caption{LASSO regression with junk features $(n = 300, d = 3000)$. Similar to the isotropic setting, the regularized LASSO can achieve nearly optimal prediction risk and the risk bound \eqref{eqn:lasso-bound} is sufficient to explain this phenomenon. Once again, the data distribution and model misspecification appear to have no effect on the shape of the test error curve. It is theoretically possible to use a nearly identical risk bound to show the consistency of minimal-$\ell_1$ norm interpolator when $n$ is large and $d$ is super-exponential in $n$ \citep{uc-interpolators}, but as we can see, $n = 300$ and $d = 3000$ is not quite large enough yet. On the other hand, overfitting is more benign than what \eqref{eqn:lasso-bound} predicts, suggesting a better analysis may yield a weaker condition required for consistency. }
  \label{fig:junk-lasso}
\end{figure}

\begin{figure} 
  \centering
  \includegraphics[width = \linewidth]{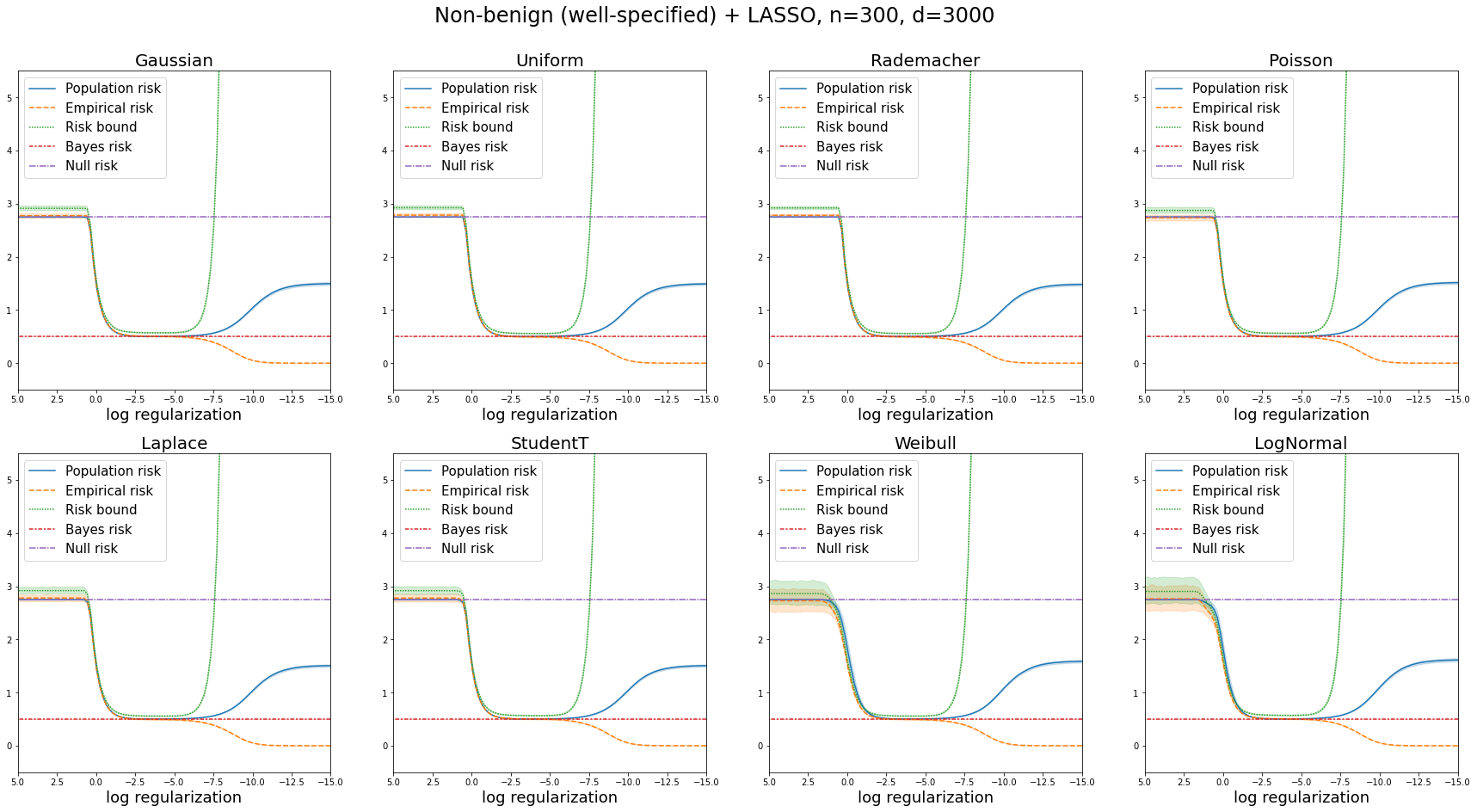}\\
  \vspace{6ex}
  \includegraphics[width = \linewidth]{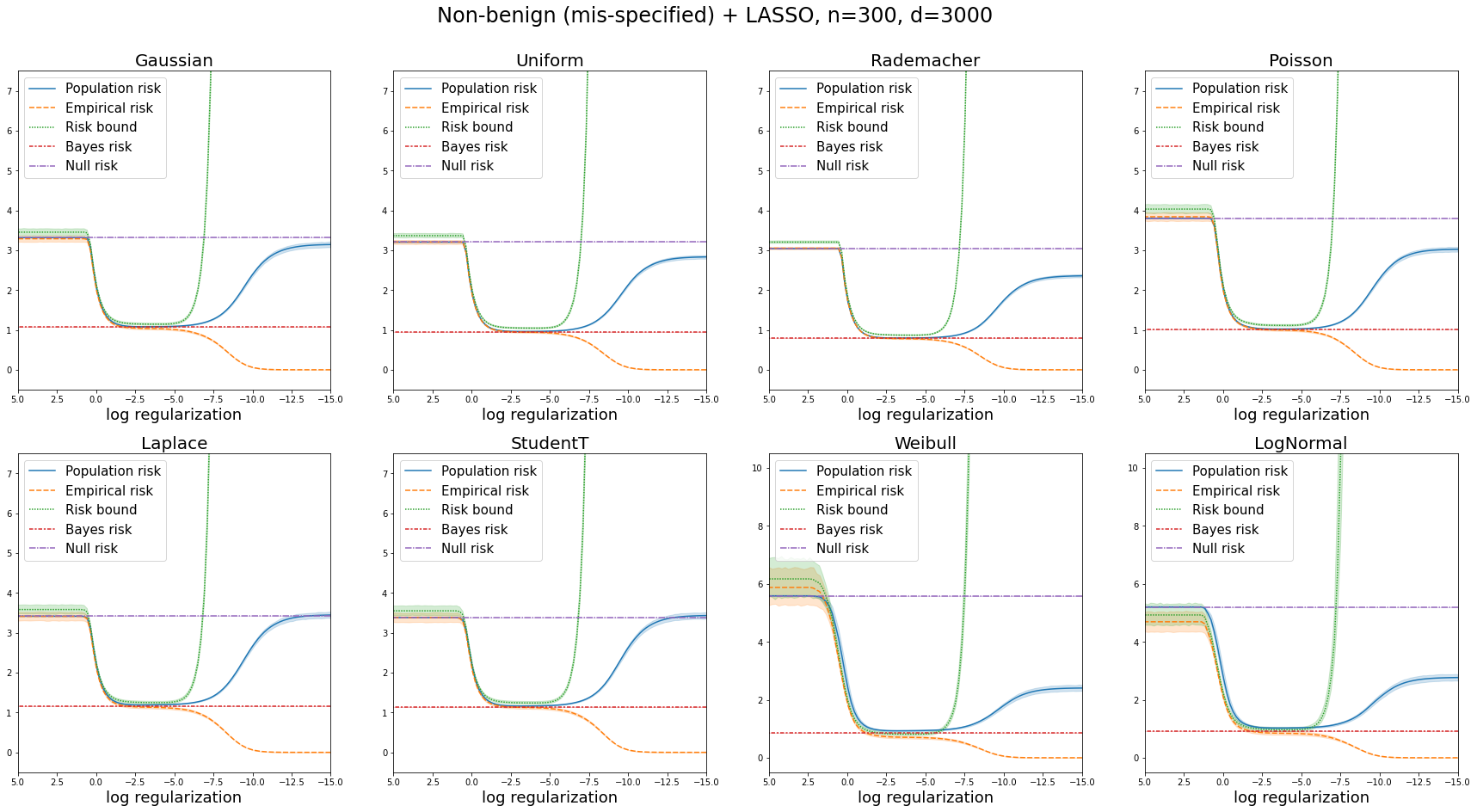}
  \vspace{2ex}
  \caption{LASSO regression with non-benign features $(n = 300, d = 3000)$. Though the population risk and the associated risk bound of regularized LASSO can be quite close to the Bayes risk, overfitting with minimal-$\ell_1$ norm interpolator does not appear to be benign (and there is no existing theoretical result suggesting that consistency is possible with a larger $n$ or $d$). In particular, its $\ell_1$ norm increases much more quickly than the junk-features case. Though the \eqref{eqn:lasso-bound} is not tight throughout the entire regularization path, it is still a valid upper bound on the test error across different feature distributions and label generating processes. }
  \label{fig:non-benign-lasso}
\end{figure}

\subsubsection{Note on Computing the Optimal Linear Predictor and Population Risk}

Since we are considering quite high-dimensional settings and we need many repeated experiments for different regularization strengths, we generally want to avoid drawing a large test set to estimate the prediction error when it is possible. In the case of square loss, we can always write the population loss (using the Mahalanobis norm notation \eqref{eqn:mahalonobis}) as
\[
L_f(w) = L_f(\tilde{w}) + \| w - \tilde{w} \|_{\Sigma}^2
\]
where $\tilde{w}$ is the optimal linear predictor satisfying the first order condition:
\[
\E[x(x^T\tilde{w} - y)]  = 0.
\]

\paragraph{Linear Model.} In the well-specified case, by the independence between $x$ and $\xi$, the above becomes
\[
\Sigma \tilde{w} = \Sigma w^* \implies \tilde{w} = w^*.
\]
Therefore, we have $L_f(\tilde{w}) = \E[(y - \langle w^*, x \rangle)^2] = \sigma^2$.

\paragraph{Mis-specified Model.} To determine the optimal linear predictor in this case, we want to set
\begin{equation*}
    \begin{split}
        \Sigma \tilde{w} &= \E[xy] \\
        &= \E[x (\langle w^*, x \rangle + |x_1| \cdot \cos x_2) ] \\
        &= \Sigma w^* + \E[x_1 \cdot |x_1| ] \E[\cos x_2] e_1 + \E[|x_1|] \E [ x_2 \cos x_2 ] e_2 
    \end{split}
\end{equation*}
and so
\[
\tilde{w} = w^* + \E[x_1 \cdot |x_1| ] \E[\cos x_2] \Sigma^{-1} e_1 + \E[|x_1|] \E [ x_2 \cos x_2 ] \Sigma^{-1} e_2.
\]

At the same time, it is routine to check that the optimal error is given by
\[
L_f(\tilde{w}) = \E[y^2] - \langle \E[xy], \Sigma^{-1} \E[xy] \rangle.
\]

It remains to compute the null risk
\begin{equation*}
    \begin{split}
        \E[y^2] 
        &= \E [(\langle w^*, x \rangle + |x_1| \cdot \cos x_2 + x_3 \xi)^2] \\
        &= \E [(\langle w^*, x \rangle + |x_1| \cdot \cos x_2)^2] + \Sigma_{33} \sigma^2 \\
        &= \langle w^*, \Sigma w^* \rangle + \E [x_1^2] \E[\cos^2 x_2] + 2 \E[\langle w^*, x \rangle (|x_1| \cdot \cos x_2) ] + \Sigma_{33} \sigma^2 \\
        &= \langle w^*, \Sigma w^* \rangle + \E [x_1^2] \E[\cos^2 x_2] + 2 \left( \E[ x_1 \cdot |x_1|] \E[ \cos x_2 ] w^*_1 + \E[|x_1|] \E[x_2 \cos x_2] w^*_2 \right) + \Sigma_{33} \sigma^2 \\
    \end{split}
\end{equation*}
and
\begin{equation*}
    \begin{split}
        \langle \E[xy], \Sigma^{-1} \E[xy] \rangle
        &= \langle \Sigma w^* + \E[| x_1| \cos (x_2) x] ,  w^* + \Sigma^{-1} \E[| x_1| \cos (x_2) x]  \rangle \\
        &= \langle \Sigma w^* ,  w^* \rangle + 2 \langle w^*,  \E[| x_1| \cos (x_2) x] \rangle +  \langle  \E[| x_1| \cos (x_2) x] , \Sigma^{-1} \E[| x_1| \cos (x_2) x] \rangle. \\
    \end{split}
\end{equation*}

Therefore, we have
\begin{equation*}
    \begin{split}
        L_f(\tilde{w}) 
        &=  \E [x_1^2] \E[\cos^2 x_2] + \Sigma_{33} \sigma^2 - \E[ x_1 \cdot |x_1|]^2 \E[\cos x_2]^2 \Sigma^{-1}_{11} -  \E[|x_1|]^2 \E[x_2 \cos(x_2)]^2 \Sigma^{-1}_{22} 
    \end{split}
\end{equation*}

It remains to compute quantities like $\E[|x|], \E[x \cdot |x|], \E[\cos x], \E[x\cos x]$ for each of the eight feature distributions. Since they are one dimensional quantities, we can afford to draw a very large number of samples to estimate them.

\subsection{Linear Classification}
Similarly, we fit linear models to minimize the squared hinge loss with $\ell_2$ and $\ell_1$ penalty. We can consider the same feature distributions and data covariance structure as in the preceding section. For faster computation (because margin classifiers can be slower to compute than regressors), we take $k = 1$, and $n = 100, d = 120$ in the proportional scaling and $n = 100, d = 2000$ in the overparameterized scaling. The label $y$ is generated by the following model:
\[
\eta = \langle w^*, x \rangle + b^*, \quad \Pr(y = 1 \, | \, x) = 1 - \Pr(y = -1 \, | \, x) = g(\eta)
\]
where $g: \R \to [0,1]$ is the logistic link function. Since we use the squared hinge loss for learning (which is not the negative log-likelihood function), the linear model that we learn is not necessarily well-calibrated and so this can also be considered as a mis-specified setting. Therefore, we will only consider one label generating process in the classification context. Finally, by our Moreau envelope theory, we can use completely the same risk bounds from \eqref{eqn:ridge-bound-isotropic} to \eqref{eqn:lasso-bound} for $\ell_2$ and $\ell_1$ margin classifiers.

\begin{figure} 
  \centering
  \includegraphics[width = \linewidth]{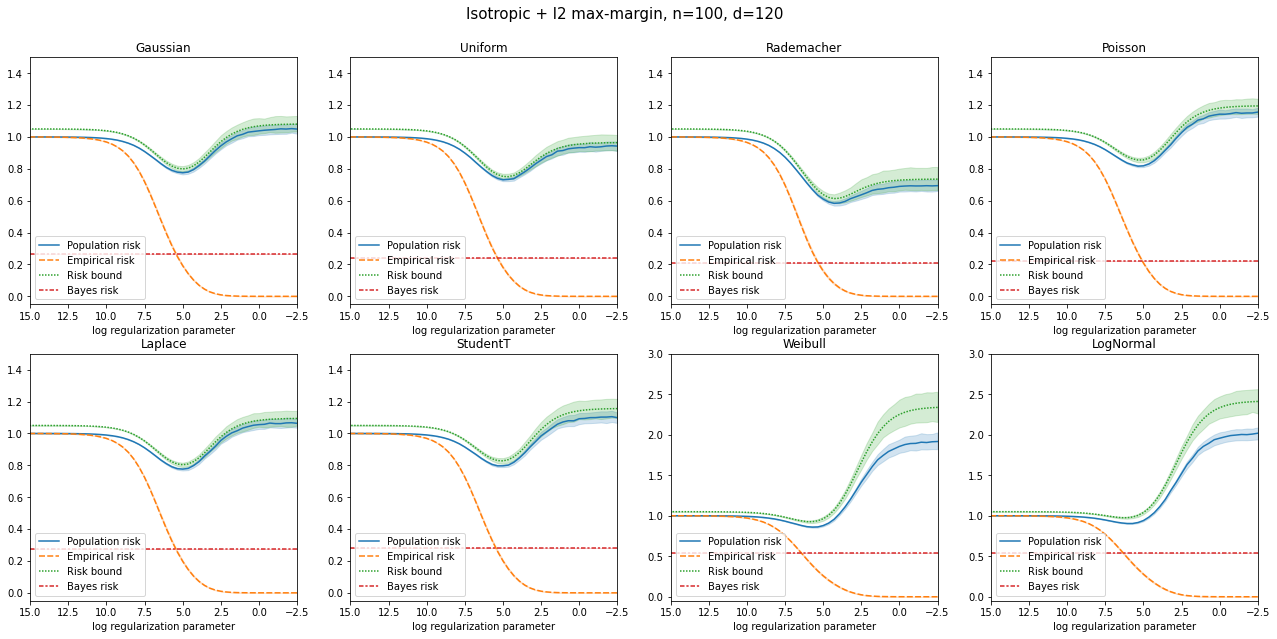}\\
  \vspace{3ex}
  \includegraphics[width = \linewidth]{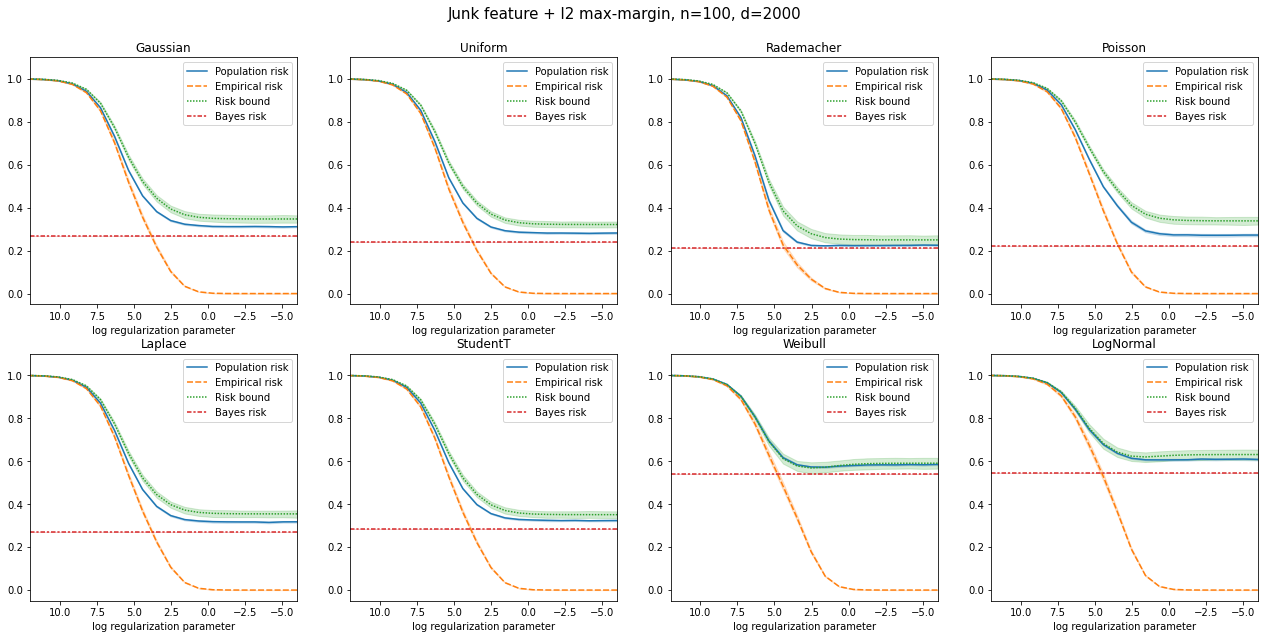} \\
  \vspace{3ex}
  \includegraphics[width = \linewidth]{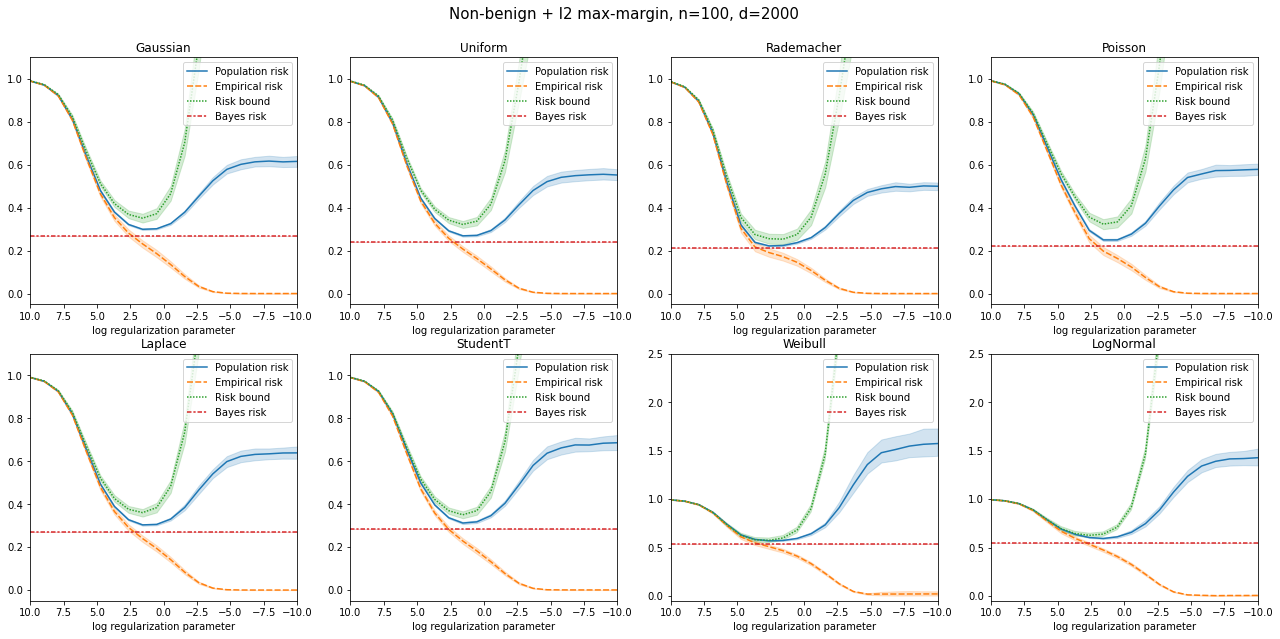} \\
  \vspace{1ex}
  \caption{$\ell_2$ margin classification: isotropic, junk and non-benign features.}
  \label{fig:l2-margin}
\end{figure}

\begin{figure} 
  \centering
  \includegraphics[width = \linewidth]{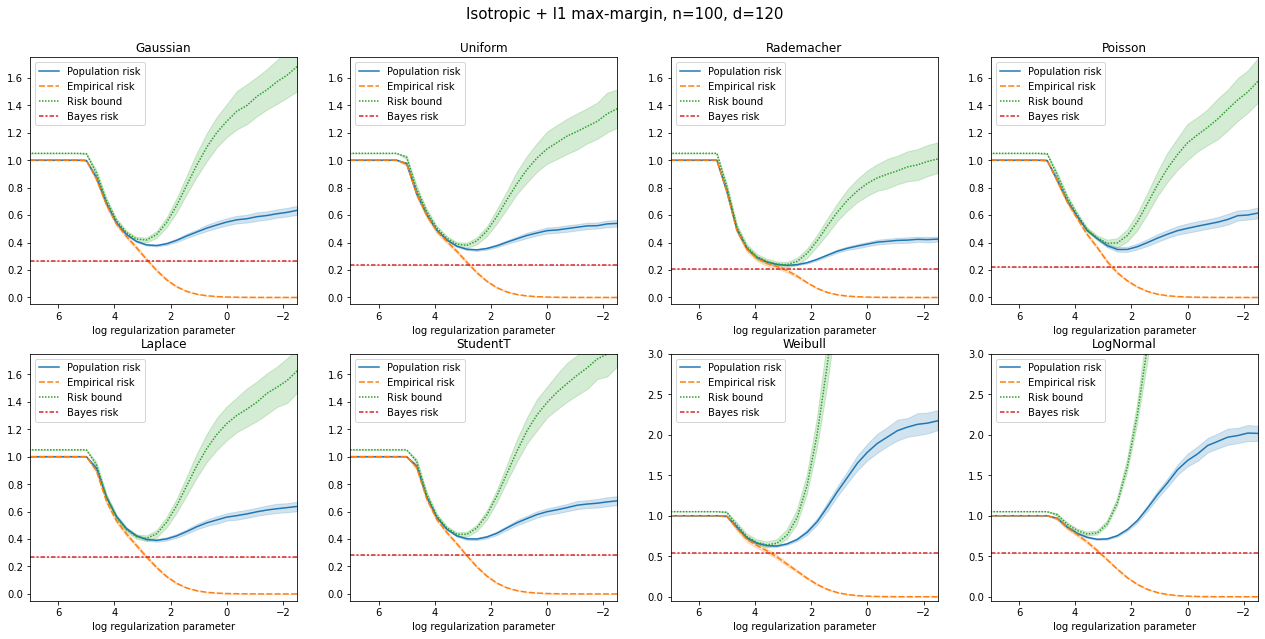}\\
  \vspace{3ex}
  \includegraphics[width = \linewidth]{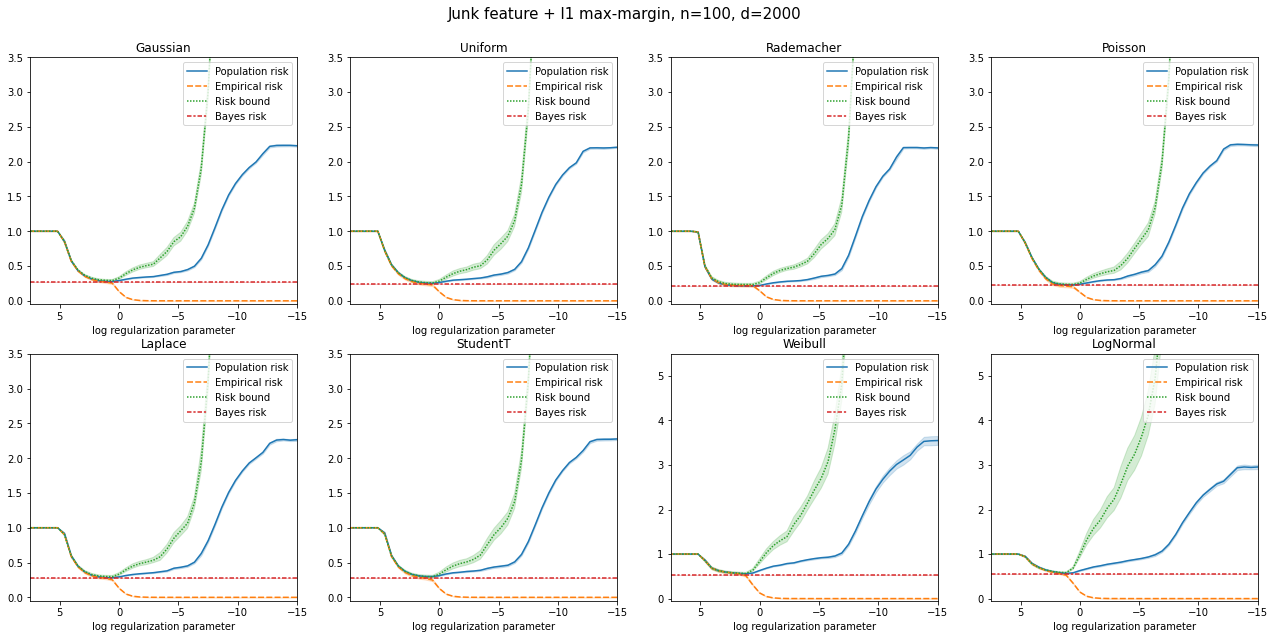} \\
  \vspace{3ex}
  \includegraphics[width = \linewidth]{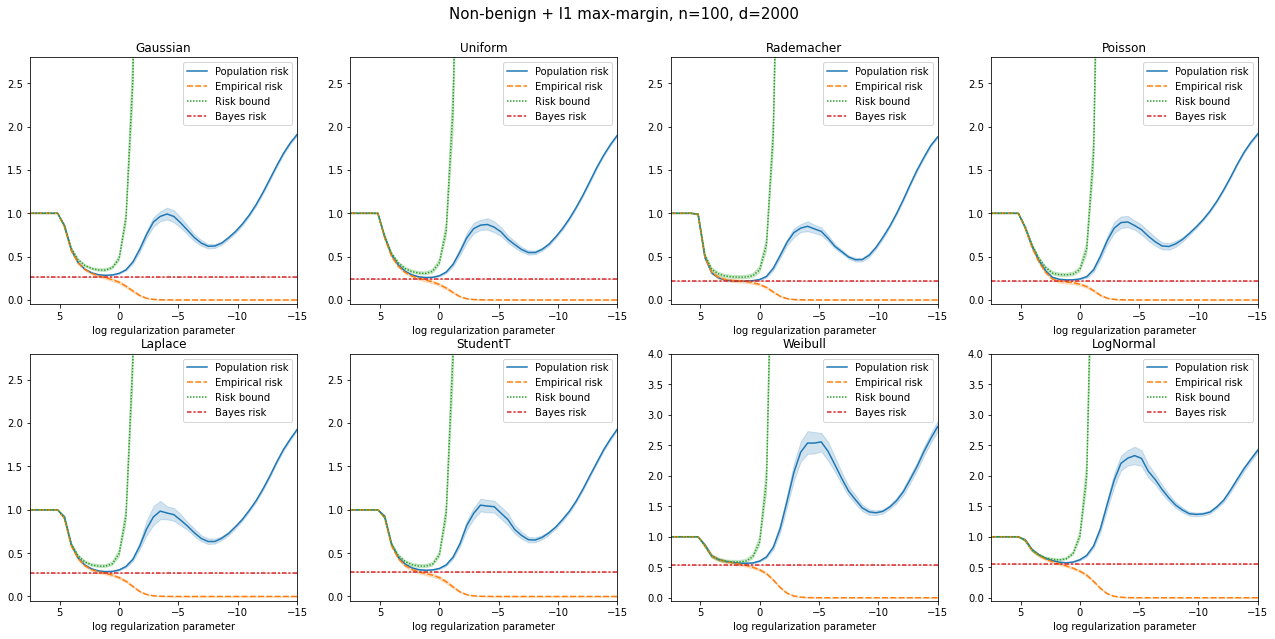}
  \caption{$\ell_1$ margin classification: isotropic, junk and non-benign features.}
  \label{fig:l1-margin}
\end{figure}

\subsubsection{Experimental Results}

The plots for $\ell_2$ and $\ell_1$ margin classifiers can be found in \cref{fig:l2-margin,fig:l1-margin}. Each figure contain three subplots, and each subplot corresponds to one of the data covariance and contains the risk curves measured in squared hinge loss for the eight feature distributions. 

\paragraph{$\ell_2$-Margin Classifiers.} As in the regression case, overfitting is not benign when the features are isotropic and the population risk of $\ell_2$ max-margin classifier can be worse than the null risk. The risk bounds tightly control the test errors across different feature distributions. The difference between risk bound and the actual test error is larger when the feature distribution is heavy-tailed, but the confidence interval is also wider due to the relatively small sample size. 

In the junk feature setting, the under-regularized part of the regularization path is essentially flat for all feature distributions. Overall, the experimental result is very similar to \cref{fig:junk-ridge}, as predicted by our theory in section~\ref{sec:benign-overfitting}. The non-benign case is also similar to \cref{fig:non-benign-ridge} except that the U-shape curve is quite narrower near the optimal amount of regularization. 

\paragraph{$\ell_1$-Margin Classifiers.} In each of the subplots, the risk bound is tight only up to a certain point before the $\ell_1$ norm starts to increase quite a lot, leading to loose bound near interpolation. However, the risk bound is tight enough to establish consistency of optimally-tuned predictor in the junk and non-benign features setting. Again, the population risk of $\ell_1$ max-margin classifier can be worse than the null risk even in the junk features setting. Observe that different distributions do not seem to change the shape of generalization curve, and there is an interesting multiple descent phenomenon in the non-benign feature case, which has already been discovered in previous literature \citep{li2021minimum, liang2020multiple, chen2021multiple}. 

\subsubsection{Note on Computing the Population Risk with Gaussian Features}

When the feature distribution is Gaussian, we can estimate
\[
L_f(w, b) 
= \E \left[ \max(0, 1 - y(\langle w, x \rangle + b))^2 \right]
\]
without drawing a new high-dimensional dataset from $\cD$. First, we can write $x = \Sigma^{1/2} z$. Note that conditioning on $\eta$ is the same as conditioning on $\langle w^*, x \rangle = \langle \Sigma^{1/2} w^*,  z \rangle \sim \cN(0, \| w^*\|_{\Sigma}^2)$ and the conditional distribution of $z$ is 
\[
\frac{ \eta - b^*}{\| w^*\|_{\Sigma}^2} \Sigma^{1/2} w^* + P z
\]
where $P = I - \frac{(\Sigma^{1/2} w^*)(\Sigma^{1/2} w^*)^T}{\| w^*\|_{\Sigma}^2 } $ and so the conditional distribution of $\langle w, x\rangle + b$ is 
\begin{equation*}
    \begin{split}
        & \left\langle w, \Sigma^{1/2} \left( \frac{ \eta - b^*}{\| w^*\|_{\Sigma}^2} \Sigma^{1/2} w^* + P z \right) \right\rangle + b \\
        = \, & b + \frac{\langle w, \Sigma w^* \rangle}{\| w^*\|_{\Sigma}^2} (\eta - b^*) + \langle P \Sigma^{1/2} w, z \rangle \sim \, \cN \left(\mu(\eta), \,\sigma^2 \right)
    \end{split}
\end{equation*}

where $\mu(\eta) = b + \frac{\langle w, \Sigma w^* \rangle}{\| w^*\|_{\Sigma}^2} (\eta - b^*)$ and
\[
\sigma^2 
= w^T (\Sigma^{1/2} P \Sigma^{1/2}) w = w^T \Sigma w - \frac{\langle w, \Sigma w^*\rangle^2}{\| w^*\|_{\Sigma}^2}.
\]
Since $x$ is independent of $y$ conditioned on $\eta$, we have that
\begin{equation*}
    \begin{split}
        L(w, b)
        &= \E \left[ \E \left[ \max(0, 1 - y(\langle w, x \rangle + b))^2 \, | \, \eta \right] \right]\\
        &= \E \left[ g(\eta) \cdot \max(0, 1 - \mu(\eta) - \sigma z)^2  + (1 - g(\eta)) \cdot  \max(0, 1 + \mu(\eta) + \sigma z)^2 \right]
    \end{split}
\end{equation*}

We can then estimate the population error by drawing samples from a two-dimensional distribution. 

\subsubsection{Note on Computing the Optimal Linear Predictor}

The linear predictor that minimizes the population squared hinge loss generally does not have a simple closed-form expression, but we can run SGD on the population objective in order to find the optimal linear predictor $\tilde{w}, \tilde{b}$. For simplicity, we choose
\[
w^* = (5, 0, ..., 0) \quad \text{ and } \quad b^* = 3.
\]
In this case, we can simplify the optimization problem to an one-dimensional problem by observing that $\tilde{w}_i = 0$ for $i \neq 1$. Indeed, we can check the first order condition holds
\begin{equation*}
    \begin{split}
        \frac{\partial}{\partial w_i} L_f(\tilde{w}, \tilde{b}) 
        &= -2 \E \left[ y\max(0, 1 - y(\langle \tilde{w}, x \rangle + \tilde{b})) x_i \right] \\
        &= -2 \E \left[ y\max(0, 1 - y( \tilde{w}_1 x_1 + \tilde{b})) \right] \E \left[x_i \right] = 0\\
    \end{split}
\end{equation*}
because $y$ is independent of $x_i$ with $i \neq 1$. Therefore, we can just generate $\{ x_{i, 1}, y_i \}$ from $\cD$ and perform one-pass SGD \citep[e.g. theorem 6.1 of][ ]{bubeck2015convex} to find $\tilde{w}_1, \tilde{b}$. In the experiments, we find choosing the initial step size to be 0.1 works well.

\section{Preliminaries}\label{apdx:preliminaries}
\paragraph{General Notation. } Following the tradition in statistics, we denote $X  = (x_1, ..., x_n)^T \in \R^{n \times d}$ as the design matrix. In the proof section, we slightly abuse the notation of $\eta_i$ to mean $Xw_i^*$ and $\xi$ to mean the $n$-dimensional random vector whose $i$-th component satisfies $y_i = g(\eta_{1,i}, ..., \eta_{k,i}, \xi_i)$. Note that we can write $X = Z\Sigma^{1/2}$ where $Z$ is a random matrix with i.i.d. standard normal entries.

We use the standard notation 
\begin{equation}\label{eqn:mahalonobis}
\|x\|_{\Sigma} := \sqrt{\langle x, \Sigma x \rangle}
\end{equation}
to denote the \emph{Mahalonobis norm} with respect to positive semidefinite matrix $\Sigma$. 

\paragraph{Additional Covariance Split Notation.} Because we will need to refer to the two parts of $\phi(w)$ often, in the remainder of the appendix we introduce the further notation $w^{\perp} = Q w$, $w^{\parallel} = (I - Q)w$ for the $\Sigma$-projection of $w$ onto the span of $w^*_1,\ldots,w^*_k$, and 
\[ r(w) := \|\Sigma^{1/2} Q w\| = \|Q w\|_{\Sigma} \]
for the Mahalanobis norm in the orthogonal space. 
We also will use the notation $X^{\parallel} = X Q$ and $X^{\perp} = X (I - Q)$ for the corresponding projections of the design matrix $X$, which are independent of each other. 

\paragraph{Concentration of Lipschitz functions.} Recall that a function $f : \mathbb{R}^n \to \mathbb{R}$ is $L$-Lipschitz with respect to the norm $\norm\cdot$ if it holds for all $x, y \in \mathbb R^n$ that $|f(x) - f(y)| \le L\|x - y\|$. We use the concentration of Lipschitz functions of a Gaussian.

\begin{theorem}[\cite{van2014probability}, Theorem 3.25] \label{thm:gaussian-concentration}
If $f$ is $L$-Lipschitz with respect to the Euclidean norm and $Z \sim \cN(0,I_n)$, then
\begin{equation}
    \Pr(|f(Z) - \E f(Z)| \ge t) \le 2e^{-t^2/2L^2}.
\end{equation}
\end{theorem}

The following straightforward concentration result is Lemma 2 of \citet{uc-interpolators}.
\begin{lemma} \label{lem:norm-concentration}
Suppose that $Z \sim \cN(0,I_n)$. Then
\begin{equation}
    \Pr(\left|\|Z\|_2 - \sqrt{n}\right| \ge t) \le 4 e^{-t^2/4}.
\end{equation}
\end{lemma}

We will use the following to help relate our problem to the surrogate distribution in our proof of \cref{thm:main-gen}.
\begin{lemma} \label{lem:conditional_dist}
Fix any integer $k < d$ and any $k$ vectors $w_1^*, ..., w_k^*$ in $\R^d$ such that $\Sigma^{1/2}w_1^*, ..., \Sigma^{1/2} w_k^*$ are orthonormal. Denoting
\begin{equation} \label{def:P-matrix}
    P = I_d - \sum_{i=1}^k (\Sigma^{1/2} w_i^*)(\Sigma^{1/2} w_i^*)^T
,\end{equation}
the distribution of $X$ conditional on $Xw_1^* = \eta_1, ..., Xw_k^* = \eta_k$ is the same as that of
\begin{equation}
    \sum_{i=1}^k \eta_i (\Sigma w_i^*)^T + ZP \Sigma^{1/2}.
\end{equation}
\end{lemma}

\begin{proof}
We can write $X = Z\Sigma^{1/2}$. The key observation is that $ZP$, $Z\Sigma^{1/2}w_1^*, ..., Z\Sigma^{1/2}w_k^*$ are independent. To see why this is the case, we can vectorize each term:
\[
\begin{pmatrix}
\text{vec}(ZP) \\
\text{vec}(Z\Sigma^{1/2}w_1^*) \\
\text{...} \\
\text{vec}(Z\Sigma^{1/2}w_k^*) \\
\end{pmatrix}
=
\begin{pmatrix}
P \otimes I_n \\
(\Sigma^{1/2}w_1^*)^T \otimes I_n \\
\text{...} \\
(\Sigma^{1/2}w_k^*)^T \otimes I_n
\end{pmatrix}
\text{vec}(Z)
\]
From the above representation, we see that the joint distribution is multivariate Gaussian and the covariance matrix is
\[
\begin{pmatrix}
P \otimes I_n \\
(\Sigma^{1/2}w_1^*)^T \otimes I_n \\
\text{...} \\
(\Sigma^{1/2}w_k^*)^T \otimes I_n
\end{pmatrix}
\begin{pmatrix}
P \otimes I_n \\
(\Sigma^{1/2}w_1^*)^T \otimes I_n \\
\text{...} \\
(\Sigma^{1/2}w_k^*)^T \otimes I_n
\end{pmatrix}^T
= \diag\left(P \otimes I_n, I_n, ..., I_n \right)
\]
Therefore, the distribution of $ZP$ remains unchanged after conditioning on $Z\Sigma^{1/2}w_1^*, ..., Z\Sigma^{1/2}w_k^*$, and we can write
\begin{equation*}
    \begin{split}
        Z 
        &= Z \left( \sum_{i=1}^k (\Sigma^{1/2} w_i^*)(\Sigma^{1/2} w_i^*)^T \right) + ZP \\
        &= \sum_{i=1}^k \eta_i (\Sigma^{1/2} w_i^*)^T  + ZP.
    \end{split}
\end{equation*}

The proof is concluded by the fact that $X = Z\Sigma^{1/2}$.
\end{proof}

A key ingredient of our technique is the Gaussian minimax theorem.

\begin{theorem}[(Convex) Gaussian Minmax Theorem; \cite{thrampoulidis2015regularized,gordon1985some}]\label{thm:gmt}
Let $Z : n \times d$ be a matrix with i.i.d. $N(0,1)$ entries and suppose $G \sim \cN(0,I_n)$ and $H \sim \cN(0,I_d)$ are independent of $Z$ and each other. Let $S_w,S_u$ be compact sets and $\psi : S_w \times S_u \to \mathbb{R}$ be an arbitrary continuous function.
Define the \emph{Primary Optimization (PO)} problem
\begin{equation}
    \Phi(Z) := \min_{w \in S_w} \max_{u \in S_u} \langle u, Z w \rangle + \psi(w,u)
\end{equation}
and the \emph{Auxiliary Optimization (AO)} problem
\begin{equation}
    \phi(G,H) := \min_{w \in S_w} \max_{u \in S_u} \|w\|_2\langle G, u \rangle + \|u\|_2 \langle H, w \rangle + \psi(w,u).
\end{equation}
Under these assumptions, $\Pr(\Phi(Z) < c) \le 2 \Pr(\phi(G,H) \le c)$ for any $c \in \mathbb{R}$.

Furthermore, if we suppose that $S_w,S_u$ are convex sets and $\psi(w,u)$ is convex in $w$ and concave in $u$, then $\Pr(\Phi(Z) > c) \le 2 \Pr(\phi(G,H) \ge c)$. 
\end{theorem}

\section{Proof of Theorem \ref{thm:main-gen}} \label{apdx:proof-main-gen}

First, let's try to formulate the generalization problem as a PO:
\begin{lemma} \label{lem:main-gen-PO}
For any deterministic function $F: \R^d \times \R \to \R^{+}$, define the primary optimization (PO) problem conditioned on $\eta_1, ..., \eta_k, \xi$ as
\begin{equation}
    \Phi:= \sup_{\substack{(w, b) \in \R^{d+1} \\ u \in \R^n} } \inf_{\lambda \in \R^n} \,   \langle \lambda, Z (P \Sigma^{1/2} w) \rangle + \psi(w, b, u, \lambda \, | \, \eta_1, ..., \eta_k, \xi)
\end{equation}
where $P$ is defined by \eqref{def:P-matrix} in \cref{lem:conditional_dist} and 
\begin{equation}
    \begin{split}
        \psi(w, b, u, \lambda \, | \, \eta_1, ..., \eta_k, \xi) 
        = \, & F(w, b) + \langle \lambda, \sum_{i=1}^k \eta_i \langle w, \Sigma w_i^* \rangle - u \rangle \\
        & -\frac{1}{n} \sum_{i=1}^n f(u_i + b, g(\eta_{1, i}, ..., \eta_{k, i}, \xi_i)).
    \end{split}
\end{equation}
Then it holds that for any $t \in \R$
\begin{equation}
    \Pr \left( \, \sup_{(w, b) \in \R^{d+1}} \, F(w, b) - \hat{L}_f(w, b) > t \, \, \Big| \, \, \eta_1, ..., \eta_k, \xi \right) = \Pr \left( \Phi > t \right)
\end{equation}
and the probability over $\Phi$ is taken only over the randomness of $Z$.
\end{lemma}

\begin{proof}
By introducing a variable $u = Xw$, we have
\begin{equation*}
	\begin{split}
		&\sup_{(w, b) \in \R^{d+1}} \, F(w, b) - \hat{L}_f(w, b) \\
		=&  \sup_{(w, b) \in \R^{d+1}} \, F(w, b) - \frac{1}{n} \sum_{i=1}^n f(\langle w, x_i \rangle + b, y_i) \\	
		=&  \sup_{\substack{(w, b) \in \R^{d+1}, u \in \R^n \\ u = Xw}} \,   F(w, b) - \frac{1}{n} \sum_{i=1}^n f(u_i + b, g(\eta_{1, i}, ..., \eta_{k, i}, \xi_i)) \\	
		=&  \sup_{(w, b) \in \R^{d+1}, u \in \R^n } \inf_{\lambda \in \R^n} \,   \langle \lambda, Xw - u \rangle + F(w, b) - \frac{1}{n} \sum_{i=1}^n f(u_i + b, g(\eta_{1, i}, ..., \eta_{k, i}, \xi_i)) \\	
	\end{split}
\end{equation*}

and so by independence of $\xi$ and $X$ and \cref{lem:conditional_dist}, it holds that for any $t \in \R$
\begin{equation*}
    \begin{split}
        &\Pr \left( \, \sup_{(w, b) \in \R^{d+1}} \, F(w, b) - \hat{L}_f(w, b) > t \, \, \Big| \, \, \eta_1, ..., \eta_k, \xi \right) \\
        = & \Pr \left( \, \sup_{ \substack{(w, b) \in \R^{d+1} \\ u \in \R^n} } \inf_{\lambda \in \R^n} \,   \langle \lambda, \left( \sum_{i=1}^k \eta_i (\Sigma w_i^*)^T + ZP \Sigma^{1/2} \right) w - u \rangle + F(w, b) -\frac{1}{n} \sum_{i=1}^n f(u_i + b, g(\eta_{1, i}, ..., \eta_{k, i}, \xi_i)) > t \right)\\
        = & \Pr \left( \, \sup_{\substack{(w, b) \in \R^{d+1} \\ u \in \R^n} } \inf_{\lambda \in \R^n} \,   \langle \lambda, ZP \Sigma^{1/2} w \rangle + \psi(w, b, u, \lambda \, | \, \eta_1, ..., \eta_k, \xi)  > t \right)\\
        = & \Pr \left( \Phi > t \right).
    \end{split}
\end{equation*}
Note that this probability is a random variable measurable with respect to the random vectors $\eta_1, ..., \eta_k$ and $\xi$.
\end{proof}

Next, let's use a truncation argument similar to the one in \citet{uc-interpolators} and then apply GMT.  Proving the following two lemmas is an exercise in real analysis, which
we include for completeness.

\begin{lemma} \label{lem:truncation}
Let $f: \R^d \to \R$ be an arbitrary function and $\cS^d_r = \{x \in \R^d: \| x\|_2 \leq r \}$, then for any set $\cK$, it holds that
\begin{equation}
    \lim_{r \to \infty} \sup_{w \in \cK \cap \cS^d_r} f(w) = \sup_{w \in \cK} f(w).
\end{equation}
If $f$ is a random function, then for any $t \in \R$
\begin{equation}
    \Pr \left(\sup_{w \in \cK} f(w) > t \right) = \lim_{r \to \infty} \Pr \left(\sup_{w \in \cK \cap \cS^d_r} f(w) > t \right).
\end{equation}
\end{lemma}

\begin{proof}
We consider two cases:
\begin{enumerate}
    \item Suppose that $\sup_{w \in \cK} f(w) = \infty$. Then for any $M > 0$, there exists $x_M \in \cK$ such that $f(x_M) > M$. Hence for any $r > \norm{x_M}_2$, it holds that
    \begin{equation*}
        \sup_{w \in \cK \cap \cS^d_r} f(w) > M \implies \liminf_{r \to \infty} \sup_{w \in \cK \cap \cS^d_r} f(w) \geq M
    \end{equation*}
    As the choice of $M$ is arbitrary, we have $\lim_{r \to \infty} \sup_{w \in \cK \cap \cS^d_r} f(w) = \infty$ as desired. 
    
    \item Suppose that $\sup_{w \in \cK} f(w) = M < \infty$. Then for any $\epsilon > 0$, there exists $x_{\epsilon} \in \cK$ such that $f(x_{\epsilon}) > M - \epsilon$. Hence for any $r > \norm{x_{\epsilon}}_2$, it holds that
    \begin{equation*}
        \sup_{w \in \cK \cap \cS^d_r} f(w) > M - \epsilon \implies \liminf_{r \to \infty} \sup_{w \in \cK \cap \cS^d_r} f(w) \geq M - \epsilon
    \end{equation*}
    As the choice of $\epsilon$ is arbitrary, we have $\liminf_{r \to \infty} \sup_{w \in \cK \cap \cS^d_r} f(w) \geq M $. On the other hand, it must be the case (by definition of supremum) that
    \begin{equation*}
        \sup_{w \in \cK \cap \cS^d_r} f(w) \leq M \implies \limsup_{r \to \infty} \sup_{w \in \cK \cap \cS^d_r} f(w)  \leq M
    \end{equation*}
    Consequently, the limit of $\sup_{w \in \cK \cap \cS^d_r} f(w) $ exists and equals $M$.
\end{enumerate}
Finally, by the fact that the supremum is increasing in $r$ and the continuity of probability measure, we have 
\begin{equation*}
    \begin{split}
        \Pr \left(\sup_{w \in \cK} f(w) > t \right)
        &= \Pr \left( \lim_{r \to \infty} \sup_{w \in \cK \cap \cS^d_r} f(w) > t \right) \\
        &= \Pr \left( \bigcup_{r \in \N} \bigcap_{R \geq r} \sup_{w \in \cK \bigcap \cS^d_R} f(w) > t \right) \\
        &= \lim_{r \to \infty} \Pr \left( \bigcap_{R \geq r} \sup_{w \in \cK \bigcap \cS^d_R} f(w) > t \right) \\
        &= \lim_{r \to \infty} \Pr \left( \sup_{w \in \cK \cap \cS^d_r} f(w) > t \right).
    \qedhere\end{split}
\end{equation*}
\end{proof}

\begin{lemma} \label{lem:truncation2}
Let $\cK$ be a compact set and $f,g$ be continuous real-valued functions on $\R^d$. Then it holds that
\begin{equation}
    \lim_{r \to \infty} \sup_{w \in \cK} \inf_{0 \leq \lambda \leq r} \lambda f(w) + g(w) = \sup_{w \in \cK: f(w) \geq 0} g(w)
.\end{equation}

If $f$ and $g$ are random functions, then for any $t \in \R$
\begin{equation}
    \Pr \left( \sup_{w \in \cK: f(w) \geq 0} g(w) \geq t \right) 
    = \lim_{r \to \infty} \Pr \left( \sup_{w \in \cK} \inf_{0 \leq \lambda \leq r} \lambda f(w) + g(w) \geq t\right).
\end{equation}
\end{lemma}

\begin{proof}
We consider two cases:
\begin{enumerate}
    \item The limiting problem is infeasible: $\forall w \in \cK, f(w) < 0$. Then by compactness and the continuity of $f$, there exists $\mu < 0$ such that for all $w \in \cK$
    \[
    f(w) < \mu \implies \sup_{w \in \cK} \inf_{0 \leq \lambda \leq r} \lambda f(w) + g(w) \leq r \mu + \sup_{w \in \cK} g(w). 
    \]
    By compactness and the continuity of $g$ again, we have $\sup_{w \in \cK} g(w) < \infty$ and so
    \[
    \lim_{r \to \infty} \sup_{w \in \cK} \inf_{0 \leq \lambda \leq r} \lambda f(w) + g(w) = -\infty
    \]
    as desired.
    
    \item The limiting problem is feasible: $\exists w_0 \in \cK, f(w_0) \geq 0$. In this case, let 
    \begin{equation*}
        \begin{split}
            w_r 
            &= \argmax_{w \in \cK} \, \inf_{0 \leq \lambda \leq r} \lambda f(w) + g(w) \\
            &= \argmax_{w \in \cK} \, r \cdot  f(w) \bone_{\{ f(w) \leq 0 \}} + g(w) \\
        \end{split}
    \end{equation*}
    be an arbitrary maximizer for each $r$. Note that a maximizer necessarily exists in $\cK$ by compactness of $\cK$ and the continuity of $f$ and $g$. By compactness of $\cK$ again, the sequence $\{w_r\}$ at positive integer values of $r$ has a subsequential limit: $\exists r_n \to \infty$ and $w_{\infty} \in \cK$ such that $w_{r_n} \to w_{\infty}$.
    
    For the sake of contradiction, assume that $f(w_{\infty}) < 0$, then by continuity, there exists $\mu < 0$ such that for all sufficiently large $n$
    \[
    f(w_{r_n}) < \mu \implies \sup_{w \in \cK} \inf_{0 \leq \lambda \leq r_n} \lambda f(w) + g(w) = r_n \cdot f(w_{r_n}) + g(w_{r_n})  \leq r_n\mu + \sup_{w \in \cK} g(w)
    \]
    which is unbounded from below as $n \to \infty$. On the other hand, we have
    \[
    \sup_{w \in \cK} \inf_{0 \leq \lambda \leq r_n} \lambda f(w) + g(w) \geq g(w_0)
    \]
    and so we have reached a contradiction; thus $f(w_\infty) \ge 0$. Observe that
    \[
    \sup_{w \in \cK} \inf_{0 \leq \lambda \leq r_n} \lambda f(w) + g(w) = r_n \cdot f(w_{r_n}) \bone_{\{ f(w_{r_n}) \leq 0 \}} + g(w_{r_n}) \leq g(w_{r_n})
    \]
    and so by continuity of $g$
    \[
    \limsup_{n \to \infty} \sup_{w \in \cK} \inf_{0 \leq \lambda \leq r_n} \lambda f(w) + g(w) \leq g(w_{\infty}) \leq \sup_{w \in \cK: f(w) \geq 0} g(w).
    \]
    The $\liminf$ direction follows immediately from the definition, and so the limit exists and equals $\sup_{w \in \cK: f(w) \geq 0} g(w)$. We can conclude that
    \begin{equation*}
    \lim_{r \to \infty} \sup_{w \in \cK} \inf_{0 \leq \lambda \leq r} \lambda f(w) + g(w) = \sup_{w \in \cK: f(w) \geq 0} g(w)
    \end{equation*}
    because it is a monotonic sequence. 
    \end{enumerate}
    
Finally, by the fact that the supremum is decreasing in $r$ and the continuity of probability measure, we have  
\begin{equation*}
    \begin{split}
        \Pr \left( \sup_{w \in \cK: f(w) \geq 0} g(w) \geq t \right)
        &= \Pr \left( \lim_{r \to \infty} \sup_{w \in \cK} \inf_{0 \leq \lambda \leq r} \lambda f(w) + g(w) \geq t\right) \\
        &= \Pr \left( \cap_{r} \sup_{w \in \cK} \inf_{0 \leq \lambda \leq r} \lambda f(w) + g(w) \geq t\right) \\
        &= \lim_{r \to \infty} \Pr \left( \sup_{w \in \cK} \inf_{0 \leq \lambda \leq r} \lambda f(w) + g(w) \geq t\right). \qedhere
    \end{split}
\end{equation*}
\end{proof}

We are now ready to apply the GMT:
\begin{lemma} \label{lem:main-gen-gmt-app}
Let $F$ be a continuous function. Consider the auxiliary problem
\begin{equation*}
    \Psi :=  \sup_{\substack{ (w,b) \in \R^{d+1}, u \in \R^n \\ \langle H, P \Sigma^{1/2}w \rangle \geq \norm{G \|P\Sigma^{1/2}w \|_2 + \sum_{i=1}^k \langle w, \Sigma w_i^* \rangle \eta_i - u}_2}}  \,  F(w, b) - \frac{1}{n} \sum_{i=1}^n f(u_i + b, g(\eta_{1,i}, ..., \eta_{k,i}, \xi_i)).
\end{equation*}
It holds that for any $t \in \R$
and $\Phi$ defined as in \cref{lem:main-gen-PO} that
\begin{equation}
    \Pr(\Phi > t) \leq 2 \Pr(\Psi \geq t).
\end{equation}
\end{lemma}

\begin{proof}
Define the truncated problems
\begin{equation}
    \Phi_r:= \sup_{ (w,b,u) \in \cS_r^{d+n+1} } \inf_{\lambda \in \R^n} \,   \langle \lambda, Z (P \Sigma^{1/2} w) \rangle + \psi(w, b, u, \lambda \, | \, \eta_1, ..., \eta_k, \xi)
\end{equation}
and
\begin{equation}
    \Phi_{r,s}:= \sup_{ (w,b,u) \in \cS_r^{d+n+1} } \inf_{\| \lambda \|_2 \leq s} \,   \langle \lambda, Z (P \Sigma^{1/2} w) \rangle + \psi(w, b, u, \lambda \, | \, \eta_1, ..., \eta_k, \xi).
\end{equation}
By definition, we have $\Phi_{r} \leq \Phi_{r,s}$ and so
\[
\Pr (\Phi_r > t) \leq \Pr (\Phi_{r,s} > t).
\]
The corresponding auxiliary problems are
\begin{equation*}
    \begin{split}
        \Psi_{r,s}
        := \sup_{ (w,b,u) \in \cS_r^{d+n+1} } \inf_{\| \lambda \|_2 \leq s} \, & \| \lambda\|_2 \langle H,P \Sigma^{1/2} w \rangle + \| P \Sigma^{1/2} w \|_2 \langle G, \lambda \rangle + \psi(w, b, u, \lambda \, | \, \eta_1, ..., \eta_k, \xi) \\
        = \sup_{ (w,b,u) \in \cS_r^{d+n+1} } \inf_{\| \lambda \|_2 \leq s} \, & \| \lambda\|_2 \langle H,P \Sigma^{1/2} w \rangle + \langle G \| P \Sigma^{1/2} w \|_2 + \sum_{i=1}^k \eta_i \langle w, \Sigma w_i^* \rangle - u , \lambda \rangle \\
        &+ F(w, b) -\frac{1}{n} \sum_{i=1}^n f(u_i + b, g(\eta_{1, i}, ..., \eta_{k, i}, \xi_i))\\
        = \sup_{ (w,b,u) \in \cS_r^{d+n+1} } \inf_{0 \leq \lambda \leq s} \, & \lambda \left( \langle H,P \Sigma^{1/2} w \rangle - \norm{ G \| P \Sigma^{1/2} w \|_2 + \sum_{i=1}^k \eta_i \langle w, \Sigma w_i^* \rangle - u }_2 \right) \\
        &+ F(w, b) -\frac{1}{n} \sum_{i=1}^n f(u_i + b, g(\eta_{1, i}, ..., \eta_{k, i}, \xi_i))\\
    \end{split}
\end{equation*}
and 
\begin{equation*}
    \Psi_r :=  \sup_{\substack{ (w,b,u) \in \cS_r^{d+n+1} \\ \langle H, P \Sigma^{1/2}w \rangle \geq \norm{G \|P\Sigma^{1/2}w \|_2 + \sum_{i=1}^k \langle w, \Sigma w_i^* \rangle \eta_i - u}_2}}  \,  F(w, b) - \frac{1}{n} \sum_{i=1}^n f(u_i + b, g(\eta_{1,i}, ..., \eta_{k,i}, \xi_i)).
\end{equation*}
By definition, it holds that $\Psi_r \leq \Psi$ and so
\[
\pr(\Psi_r \geq t) \leq \Pr (\Psi \geq t).
\]
Thus
\begin{align*}
        \Pr(\Phi > t) 
        &= \lim_{r \to \infty} \Pr(\Phi_r > t) \\
        &\leq \lim_{r \to \infty} \lim_{s \to \infty} \Pr(\Phi_{r,s} > t)
        && \text{by \cref{lem:truncation}} \\
        &\leq 2 \lim_{r \to \infty} \lim_{s \to \infty} \Pr(\Psi_{r,s} \geq t)
        && \text{by \cref{thm:gmt}} \\
        &= 2 \lim_{r \to \infty} \Pr(\Psi_{r} \geq t)
        && \text{by \cref{lem:truncation2}} \\
        &\leq 2 \Pr(\Psi \geq t). \qedhere
\end{align*}
\end{proof}

\begin{lemma} \label{lem:main-gen-AO}
Let $\Psi$ be as in \cref{lem:main-gen-gmt-app}.
Under the assumptions \eqref{eqn:low-dimension-concentration} and \eqref{eqn:complexity-defn} in \cref{thm:main-gen}, it holds with probability at least $1-\delta/2$ that
\[
\Psi \leq \sup_{(w,b) \in \R^{d+1}}  \,  F(w, b) - L_{f_{\lambda}} (w, b) + \epsilon_{\lambda, \delta}(\phi(w), b) + \frac{\lambda C_{\delta} (w)^2}{n}
\]
and if assumption \eqref{eqn:low-dimension-concentration} holds uniformly over all $\lambda \in \R^+$, then
\[
\Psi \leq \sup_{(w,b) \in \R^{d+1}}  \,  F(w, b) - \sup_{\lambda \in \R^+}\left[  L_{f_{\lambda}} (w, b) - \epsilon_{\lambda, \delta}(\phi(w), b) - \frac{\lambda C_{\delta} (w)^2}{n} \right]
\]
where the randomness is taken over $H, G, \eta_1, ..., \eta_k$ and $\xi$.
\end{lemma}

\begin{proof}
First, let's simplify the auxiliary problem.
Changing variables to subtract
$G_i \norm{P \Sigma^{1/2} w}_2 + \sum_{l=1}^k \langle w, \Sigma w_l^* \rangle \eta_{l,i}$ from each of the former $u_i$,
we have that
\begin{equation*}
    \begin{split}
        \Psi 
        &=  \sup_{\substack{ (w,b, u) \in \R^{d+n+1} \\ \norm{u}_2 \leq \langle H, P \Sigma^{1/2}w \rangle}}  \,  F(w, b) - \frac{1}{n} \sum_{i=1}^n f\left(G_i \|P\Sigma^{1/2}w \|_2 + \sum_{l=1}^k \langle w, \Sigma w_l^* \rangle \eta_{l, i} + b+ u_i, g(\eta_{1,i}, ..., \eta_{k,i}, \xi_i)\right) \\
        &=  \sup_{(w,b) \in \R^{d+1}}  \,  F(w, b) - \inf_{\substack{ u \in \R^n \text{s.t.} \\ \norm{u}_2 \leq \langle \Sigma^{1/2} PH, w \rangle}} \frac{1}{n} \sum_{i=1}^n f\left(G_i \|P\Sigma^{1/2}w \|_2 + \sum_{l=1}^k \langle w, \Sigma w_l^* \rangle \eta_{l, i} + b + u_i, g(\eta_{1,i}, ..., \eta_{k,i}, \xi_i) \right)
    \end{split}
\end{equation*}
We can analyze the second term. If $\langle \Sigma^{1/2} P H, w \rangle < 0$ then the constraint on $u$ is not satisfiable and so the infimum is $\infty$. Otherwise, by duality
\begin{equation*}
    \begin{split}
        &\inf_{\substack{ u \in \R^n \text{s.t.} \\ \norm{u}_2 \leq \langle \Sigma^{1/2} PH, w \rangle}} \sum_{i=1}^n f\left(G_i \|P\Sigma^{1/2}w \|_2 + \sum_{l=1}^k \langle w, \Sigma w_l^* \rangle \eta_{l, i} + b + u_i, g(\eta_{1,i}, ..., \eta_{k,i}, \xi_i)\right) \\
        =& \inf_{u \in \R^n} \sup_{\lambda \geq 0} \, \lambda (\| u\|_2^2 - \langle \Sigma^{1/2} PH, w \rangle^2) + \sum_{i=1}^n f\left(G_i \|P\Sigma^{1/2}w \|_2 + \sum_{l=1}^k \langle w, \Sigma w_l^* \rangle \eta_{l, i} + b + u_i, g(\eta_{1,i}, ..., \eta_{k,i}, \xi_i)\right)\\
        =&\sup_{\lambda \geq 0} - \lambda \langle \Sigma^{1/2} PH, w \rangle^2 + \inf_{u \in \R^n}  \, \sum_{i=1}^n f\left(G_i \|P\Sigma^{1/2}w \|_2 + \sum_{l=1}^k \langle w, \Sigma w_l^* \rangle \eta_{l, i} + b + u_i, g(\eta_{1,i}, ..., \eta_{k,i}, \xi_i)\right) + \lambda u_i^2\\
        =&\sup_{\lambda \geq 0} - \lambda \langle \Sigma^{1/2} PH, w \rangle^2 + \sum_{i=1}^n \inf_{u \in \R} f\left(G_i \|P\Sigma^{1/2}w \|_2 + \sum_{l=1}^k \langle w, \Sigma w_l^* \rangle \eta_{l, i} + b + u, g(\eta_{1,i}, ..., \eta_{k,i}, \xi_i)\right) + \lambda u^2\\
        =&\sup_{\lambda \geq 0} \, \sum_{i=1}^n  f_{\lambda}\left(G_i \|P\Sigma^{1/2}w \|_2 + \sum_{l=1}^k \langle w, \Sigma w_l^* \rangle \eta_{l, i} + b, g(\eta_{1,i}, ..., \eta_{k,i}, \xi_i)\right) -  \lambda \langle \Sigma^{1/2} PH, w \rangle^2
    ,\end{split}
\end{equation*}
recalling \cref{def:moreau}.
For simplicity of notation, write
\[
\tilde{x}_i = (\eta_{1,i}, ..., \eta_{k,i}, G_i) \sim \cN(0, I_{k+1})
;\]
then the joint distribution of $(\tilde{x}_i, y_i)$ is exactly the same as the surrogate distribution $\tilde{D}$ given by \eqref{eqn:surrogate-model}. Moreover, we can check that
\begin{equation*}
    \begin{split}
        P\Sigma^{1/2}w
        &= \left( I_d - \sum_{i=1}^k (\Sigma^{1/2} w_i^*)(\Sigma^{1/2} w_i^*)^T \right) \Sigma^{1/2}w\\
        &= \Sigma^{1/2} \left( I_d - \sum_{i=1}^k w_i^*( w_i^*)^T \Sigma \right) w\\
        &= \Sigma^{1/2} Q w
    \end{split}
\end{equation*}
and 
\begin{equation*}
    \begin{split}
        \Sigma^{1/2} PH 
        &= \Sigma^{1/2} \left( I_d - \sum_{i=1}^k (\Sigma^{1/2} w_i^*)(\Sigma^{1/2} w_i^*)^T \right) H \\
        &= \left( I_d - \sum_{i=1}^k (\Sigma w_i^*)( w_i^*)^T \right) \Sigma^{1/2} H = Q^T \Sigma^{1/2} H
    \end{split}
\end{equation*}
where $Q$ is given by equation \eqref{eqn:Q-defn}. Then using the definition of $\phi$ from \eqref{eqn:Q-defn}, we can write
\[
G_i \|P\Sigma^{1/2}w \|_2 + \sum_{l=1}^k \langle w, \Sigma w_l^* \rangle \eta_{l, i} = \langle \phi(w), \tilde{x}_i \rangle,
\]
giving that
\[
        \frac{1}{n} \sum_{i=1}^n  f_{\lambda}\left(G_i \|P\Sigma^{1/2}w \|_2 + \sum_{l=1}^k \langle w, \Sigma w_l^* \rangle \eta_{l, i} + b, g(\eta_{1,i}, \dots, \eta_{k,i}, \xi_i)\right)
        = \frac{1}{n} \sum_{i=1}^n  f_{\lambda}(\langle \phi(w), \tilde{x}_i \rangle + b, y_i)
.\]
By our assumption \eqref{eqn:low-dimension-concentration} and the observation in \cref{lem:conditional_dist} that the joint distribution of $(\langle \phi(w), \tilde{x} \rangle, y)$ is the same as that of $(\langle w, x\rangle, y)$, we have
\begin{equation*}
    \begin{split}
        \frac{1}{n} \sum_{i=1}^n  f_{\lambda}(\langle \phi(w), \tilde{x}_i \rangle + b, y_i)
        &\geq \E_{(\tilde{x}, \tilde{y}) \sim \tilde{D}} \, [f_{\lambda} (\langle \phi(w), \tilde{x} \rangle + b, \tilde{y}) ] - \epsilon_{\lambda, \delta}(\phi(w), b) \\
        &= L_{f_{\lambda}} (w, b) - \epsilon_{\lambda, \delta}(\phi(w), b) \\
    \end{split}
\end{equation*}
with probability at least $1-\delta/4$.

In addition, noting that $\Sigma^{1/2} H \sim \cN(0, \Sigma)$, our assumption \eqref{eqn:complexity-defn} implies that with probability at least $1 - \delta / 4$,
\[
    \langle \Sigma^{1/2} PH, w \rangle
    = \langle Q^T x, w\rangle
    = \langle Qw, x \rangle \leq C_{\delta} (w).
\]
The proof concludes by a union bound and plugging the above estimates into the expression for $\Psi$.
\end{proof}

Finally, we can prove our main theorem, restated here for convenience:

\MainGen*

\begin{proof}
By \cref{lem:main-gen-PO} and \cref{lem:main-gen-gmt-app}, we have
\[
\Pr \left( \, \sup_{(w, b) \in \R^{d+1}} \, F(w, b) - \hat{L}_f(w, b) > t \, \, \Big| \, \, \eta_1, ..., \eta_k, \xi \right) 
\leq 2 \Pr(\Psi \geq t).
\]
By the tower law and choosing 
\[
F(w, b) = L_{f_{\lambda}} (w, b) - \epsilon_{\lambda, \delta}(\phi(w), b) - \frac{\lambda C_{\delta} (w)^2}{n}
\]
in \cref{lem:main-gen-AO}, we get that
\[
\Pr \left( \, \sup_{(w, b) \in \R^{d+1}} \, L_{f_{\lambda}} (w, b) - \epsilon_{\lambda, \delta}(\phi(w), b) - \frac{\lambda C_{\delta} (w)^2}{n} - \hat{L}_f(w, b) > 0  \right) 
\leq \delta.
\]
as desired. If assumption \eqref{eqn:low-dimension-concentration} holds uniformly over $\lambda \in \R^+$, then we can choose
\[
F(w, b) = \sup_{\lambda \in \R^+} L_{f_{\lambda}} (w, b) - \epsilon_{\lambda, \delta}(\phi(w), b) - \frac{\lambda C_{\delta} (w)^2}{n}.
\]
It is straightforward to check that $F$ is continuous and the same proof goes through.
\end{proof}

\begin{remark}\label{rmk:one-sided}
Since the dimension of $\tilde{x}$ is small, we can typically expect \eqref{eqn:low-dimension-concentration} to hold for reasonable settings with a sufficiently large sample size. 
Note that this is our only assumption on $f, g$ and $\xi$, and this is required to avoid pathological learning problems. A useful aspect of the assumption \eqref{eqn:low-dimension-concentration} is that it only requires \emph{one-sided concentration} of the training loss. As emphasized by many works in statistical learning theory (e.g. \cite{lecue2013learning,mendelson2014learning,koltchinskii2015bounding,mendelson2017extending}), lower bounds on the training loss are both more convenient to establish and hold in more generic settings than upper bounds do. In this paper, we will largely apply results from VC theory to handle the low-dimensional problem; the results we appeal to are indeed one-sided and can handle relatively heavy-tailed noise \citep{vapnik2006estimation}.
\end{remark}

\section{Proof for VC theory and Section~\ref{sec:applications}} \label{apdx:applications}

\subsection{Low-Dimensional Concentration}

Recall the following definition of VC-dimension from \citet{shalev2014understanding}.

\begin{defn}
Let $\cH$ be a class of functions from $\cX$ to $\{0, 1\}$ and let $C = \{c_1, ..., c_m \} \subset \cX$. The restriction of $\cH$ to $C$ is
\[
\cH_C = \{(h(c_1), ..., h(c_m)) : h \in \cH \}.
\]
A hypothesis class $\cH$ \emph{shatters} a finite set $C \subset \cX$ if $|\cH_C| = 2^{|C|}$. The VC-dimension of $\cH$ is the maximal size of a set that can be shattered by $\cH$. If $\cH$ can shatter sets of arbitrary large size, we say $\cH$ has infinite VC-dimension.
\end{defn}

Also, we have the following well-known result for the class of nonhomogenous halfspaces in $\R^d$ \citep[Theorem 9.3 of][ ]{shalev2014understanding}:

\begin{theorem} \label{thm:linear-vc}
The class $\{ x \mapsto \normalfont{sign}(\langle w, x\rangle + b): w \in \R^d, b \in \R \}$ has VC-dimension $d+1$.
\end{theorem}

We will make use of the following result from \citet{vapnik2006estimation}:

\VCtheory* 

Combining with theorem~\ref{thm:main-gen}, we obtain the following corollary.

\VCgen*

\begin{proof}
By theorem~\ref{thm:vc-hypercontractive}, we can take
\[
\epsilon_{\lambda, \delta}(\tilde{w}, \tilde{b}) = 
\begin{cases}
    8 \tau \sqrt{\frac{h(\log(2n/h) + 1) + \log(48/\delta)}{n}} \E_{(\tilde{x}, \tilde{y}) \sim \tilde{\cD}} [f_{\lambda}(\langle \tilde{w}, \tilde{x} \rangle + \tilde{b}, \tilde{y})] 
    &\mbox{ if } (\tilde{w}, \tilde{b}) \in \phi(\cK) \times \cB\\ 
    \infty &\mbox{ otherwise }
\end{cases}
\]
and the desired conclusion follows by the observation that
\[
\E_{(\tilde{x}, \tilde{y}) \sim \tilde{\cD}} [f_{\lambda}(\langle \phi(w), \tilde{x} \rangle + b, \tilde{y})] = L_{f_{\lambda}} (w, b).
\]
The last conclusion (uniformity over $\lambda$) follows by going through the proof of Theorem~\ref{thm:vc-hypercontractive}, since it is based on reduction to uniform control of indicators.  
\end{proof}

\subsection{Linear Regression}

First, we provide a VC-dimension bound for the square loss class.
\begin{lemma} \label{lem:vc-square-loss}
Suppose $f$ is the square loss, then the VC-dimension of the class 
\[
\{(x,y) \mapsto \bone \{ f_{\lambda}(\langle \tilde w, \phi(x) \rangle + \tilde b, y) > t \} : (\tilde w, \tilde b) \in \mathbb{R}^{k + 2}, t \in \mathbb{R}, \lambda \in \mathbb{R}_{\ge 0}\}
\] 
is $O(k)$.
\end{lemma}

\begin{proof}
Since the square loss is non-negative, we only need to consider $t \geq 0$. Recall that $f_{\lambda} = \frac{\lambda}{1 + \lambda} f$ for the square loss and so
\begin{equation*}
    f_{\lambda}(\langle \tilde w, \phi(x) \rangle + \tilde b, y) > t 
    \iff (\langle \tilde w, \phi(x) \rangle + \tilde b - y)^2 > \frac{(1 + \lambda) t}{\lambda} \\
\end{equation*}
which happens if
\[
\left\langle  
\begin{pmatrix}
\tilde{w} \\
-1
\end{pmatrix}, 
\begin{pmatrix}
\phi(x) \\
y
\end{pmatrix}
\right\rangle 
+ \left( \tilde{b} - \sqrt{\frac{(1 + \lambda) t}{\lambda}} \right) > 0
\quad \text{ or } \quad
\left\langle  
\begin{pmatrix}
-\tilde{w} \\
1
\end{pmatrix}, 
\begin{pmatrix}
\phi(x) \\
y
\end{pmatrix}
\right\rangle - \left( \tilde{b} + \sqrt{\frac{(1 + \lambda) t}{\lambda}} \right) > 0.
\]
In particular, if this concept class can shatter $m$ points, so can the class of the union of two nonhomogenous halfspaces in $\R^{k+2}$. The desired conclusion follows by the well-known fact that the VC-dimension of the union of two halfspaces is $O(k)$. For example, by combining Theorem \ref{thm:linear-vc} with Lemma 3.23 of \citet{blumer1989learnability}, the VC-dimension cannot be larger than $4\log6 \cdot (k+3)$.
\end{proof}

Specializing our generalization theory to the square loss, we have:

\GenSquareLoss*

\begin{proof}
Note that if condition \eqref{eqn:hypercontractive-assumption} holds under $f$, then it also holds under all $\{f_{\lambda}: \lambda \geq 0 \}$ because $f_{\lambda} = \frac{\lambda}{1 + \lambda} f$. Moreover, we check the assumption on VC-dimension of Corollary~\ref{corr:main-gen-vc} in Lemma~\ref{lem:vc-square-loss}. From this, we get uniformly over $\lambda,w,b$ that
\[ \frac{\lambda}{1 + \lambda}\left(1 - 8\tau\sqrt{\frac{k(\log(2n/k) + 1) + \log(48/\delta)}{n}}\right) L_f(w,b) \le \hat{L}_f(w,b) +  \frac{\lambda C_{\delta}(w)^2}{n}.   \]
Multiplying through by $(1 + \lambda)/\lambda$, we can rewrites the above as 
\[
\left(1 - 8\tau\sqrt{\frac{k(\log(2n/k) + 1) + \log(48/\delta)}{n}}\right) L_f(w,b) 
\leq \left(1 + \frac{1}{\lambda} \right) \hat{L}_f(w,b) +  (1 + \lambda)\frac{ C_{\delta}(w)^2}{n} 
\]
and optimizing over $\lambda$ gives
\begin{equation*}
    \begin{split}
        & \left(1 - 8\tau\sqrt{\frac{k(\log(2n/k) + 1) + \log(48/\delta)}{n}}\right) L_f(w,b) \\
        \leq &\, \hat{L}_f(w,b) + \frac{ C_{\delta}(w)^2}{n} + \inf_{\lambda \geq 0} \, \frac{1}{\lambda} \hat{L}_f(w,b) + \lambda\frac{ C_{\delta}(w)^2}{n} \\
        = &\, \hat{L}_f(w,b) + \frac{ C_{\delta}(w)^2}{n}  + 2 \sqrt{\hat{L}_f(w,b) \frac{ C_{\delta}(w)^2}{n} } = \left(\sqrt{\hat{L}_f(w,b)} + C_{\delta}(w)/\sqrt{n}\right)^2.
    \qedhere \end{split}
\end{equation*}
\end{proof}

Finally, as an illustrative example, we consider the misspecified model mentioned in the main text where the true regression function is a polynomial. In this case, we show explicitly how to get an expression for $\tau$ in \eqref{eqn:hypercontractive-assumption} using Gaussian hypercontractivity. The following theorem is the Gaussian space analogue of Theorem 9.21 in \citet{o2014analysis} and can be proved using the same argument by Theorem 11.23 and replacing the Fourier basis on $\{-1, 1\}^n$ with the Hermite polynomials on $\R^n$.

\begin{theorem}[\cite{o2014analysis}] \label{thm:gaussian-hypercontractivity}
Let $f: \R^d \to \R$ be a polynomial of degree at most $k$. Then for any $q \geq 2$, it holds that
\begin{equation}
    \E_{z \sim \cN(0, I_d)} [|f(z)|^q]^{1/q} \leq (q - 1)^{k/2} \E_{z \sim \cN(0, I_d)} [|f(z)|^2]^{1/2}.
\end{equation}
\end{theorem}

\begin{theorem} \label{thm:poly-hypercontractivity}
Suppose that in \eqref{eqn:model}, we have
\[
y = m(\eta_1, ..., \eta_k) + s(\eta_1, ..., \eta_k) \cdot \xi
\]
where $m, s$ are both polynomials of degree at most $l$ and $\xi$ has finite eighth moment, then 
\begin{equation}
    \frac{\E [(\langle w, x\rangle + b - y)^8]^{1/8}}{\E [(\langle w, x\rangle + b - y)^2]^{1/2}} \leq \sqrt{2} \cdot \sqrt{7}^l \left(\frac{E [\xi^8]^{1/8}}{E [\xi^2]^{1/2}} \right).
\end{equation}
\end{theorem}

\begin{proof}
By triangular inequality in the $\ell_p$ space and independence between $x$ and $\xi$ \begin{equation*}
    \begin{split}
        \E [(\langle w, x\rangle + b - y)^8]^{1/8} 
        &\leq \E [(\langle w, x\rangle + b - m(\eta_1, ..., \eta_k))^8]^{1/8} + \E [(s(\eta_1, ..., \eta_k) \cdot \xi)^8]^{1/8} \\
        &=\E [(\langle w, x\rangle + b - m(\eta_1, ..., \eta_k))^8]^{1/8} + \E [s(\eta_1, ..., \eta_k)^8]^{1/8} \cdot \E [\xi^8]^{1/8} \\
    \end{split}
\end{equation*}
Since $\langle w, x \rangle, \eta_1, ..., \eta_k$ are jointly Gaussian, we can apply Theorem~\ref{thm:gaussian-hypercontractivity} and upper bound the above by
\begin{equation*}
    \begin{split}
        &\sqrt{7}^l \left( \E [(\langle w, x\rangle + b - m(\eta_1, ..., \eta_k))^2]^{1/2} + \E [s(\eta_1, ..., \eta_k)^2]^{1/2} \cdot \E [\xi^8]^{1/8} \right) \\
        \leq \, & \sqrt{7}^l \left(\frac{E [\xi^8]^{1/8}}{E [\xi^2]^{1/2}} \right) \left( \E [(\langle w, x\rangle + b - m(\eta_1, ..., \eta_k))^2]^{1/2} + \E [s(\eta_1, ..., \eta_k)^2]^{1/2} \cdot \E [\xi^2]^{1/2} \right) \\
        \leq \, & \sqrt{7}^l \left(\frac{E [\xi^8]^{1/8}}{E [\xi^2]^{1/2}} \right) \sqrt{2} \cdot \sqrt{\E [(\langle w, x\rangle + b - m(\eta_1, ..., \eta_k))^2] +  \E[s(\eta_1, ..., \eta_k)^2] \cdot \E [\xi^2]}
    \end{split}
\end{equation*}
where we use $E [\xi^8]^{1/8} \geq \E [\xi^2]^{1/2}$ in the second inequality and $\sqrt{a} + \sqrt{b} \leq \sqrt{2(a+b)}$ in the last inequality. The desired conclusion follows by observing
\[
\E [(\langle w, x\rangle + b - y)^2] = \E [(\langle w, x\rangle + b - m(\eta_1, ..., \eta_k))^2] +  \E[s(\eta_1, ..., \eta_k)^2] \cdot \E [\xi^2]
\]
because $x$ and $\xi$ are independent.
\end{proof}

\begin{remark}
The assumption that $\xi$ has finite eighth moment can be significantly relaxed because there is a version of Theorem~\ref{thm:vc-hypercontractive} in \citet{vapnik2006estimation} that replaces the exponent of 4 by $1 + \epsilon$. However, allowing heavier tails of $\xi$ comes at the cost of a larger constant in front of $\tau$ or a slower convergence rate with respect to $n$ in the low-dimensional concentration term.
\end{remark}

\subsection{Linear Classification} \label{apdx:linear-clf}

\begin{lemma}
Suppose $f$ is the squared hinge loss, then the VC-dimension of the class 
\[
\{(x,y) \mapsto \bone \{ f_{\lambda}(\langle \tilde w, \phi(x) \rangle + \tilde b, y) > t \} : (\tilde w, \tilde b) \in \mathbb{R}^{k + 2}, t \in \mathbb{R}, \lambda \in \mathbb{R}_{\ge 0}\}
\] 
is no larger than $k+3$.
\end{lemma}

\begin{proof}
Since the squared hinge loss is non-negative, we only need to consider $t \geq 0$. Recall that $f_{\lambda} = \frac{\lambda}{1 + \lambda} f$ and so
\begin{equation*}
    \begin{split}
        f_{\lambda}(\langle \tilde w, \phi(x) \rangle + \tilde b, y) > t 
        & \iff (1 - y(\langle \tilde w, \phi(x) \rangle + \tilde b))_+^2 > \frac{(1 + \lambda) t}{\lambda} \\
        & \iff 1 - y(\langle \tilde w, \phi(x) \rangle + \tilde b) > \sqrt{\frac{(1 + \lambda) t}{\lambda}} \\
        &\iff 
        \left\langle
        \begin{pmatrix}
        \tilde{w} \\
        \tilde{b}
        \end{pmatrix}, 
        \begin{pmatrix}
        -y\phi(x) \\
        -y
        \end{pmatrix}
        \right\rangle
        +
        \left( 1 - \sqrt{\frac{(1 + \lambda) t}{\lambda}} \right) > 0. \\
    \end{split}
\end{equation*}
In particular, if this class can shatter $m$ points, so can the class of nonhomogenous halfspaces in $\R^{k+2}$. But theorem~\ref{thm:linear-vc} shows that it cannot shatter more than $k+4$ points, and so the VC-dimension cannot be larger than $k+3$.
\end{proof}

By the same proof as Corollary~\ref{corr:gen-square-loss}, we have
\begin{corollary}\label{corr:gen-square-hinge-loss}
Suppose $f$ is the squared hinge loss and the surrogate distribution $\tilde \cD$ satisfies assumption \eqref{eqn:hypercontractive-assumption} uniformly over $(w, b) \in \R^{k+1}$, then with probability at least $1 - \delta$, uniformly over all $w,b$ we have
\[ \left(1 - 8\tau\sqrt{\frac{k(\log(2n/k) + 1) + \log(48/\delta)}{n}}\right) L_f(w,b) \le \left(\sqrt{\hat{L}_f(w,b)} + C_{\delta}(w)/\sqrt{n}\right)^2. \]
\end{corollary}
For illustration, we show how to check hypercontractivity \eqref{eqn:hypercontractive-assumption} under some example generative assumptions on $y$. In the first and simpler example, suppose that there is an arbitrary constant $\eta > 0$ such that
\[ \min \{\Pr(y = 1 \mid x), \Pr(y = -1 \mid x)\} \ge \eta \]
almost surely.
This assumption is satisfied, for example, if the data is generated by an arbitrary function of $\eta_1,\ldots,\eta_k$ combined with Random Classification Noise (see e.g.\ \citet{blum2003noise}), i.e. the label is flipped with some probability. Then if $\hat y = \langle w, x \rangle + b$ is the prediction, we have
\[ \E \max(0,1 - y \hat y)^2 \ge \eta \E (1 + |\hat y|)^2 \ge \eta (1 + \E[\hat y^2]), \]
and on the other hand we always have
\[ \E \max(0,1 - y \hat y)^8 \le \E (1 + |\hat y|)^8 \le 2^8 (1 + \E[\hat y^8]) \le 2^{16} (1 + \E[\hat y^2]^4) \le 2^{16} (1 + \E[\hat y^2])^4 \]
where the second-to-last inequality follows from the fact that $\hat y$ is marginally Gaussian and using standard formula for the moments of a Gaussian. It follows that
\[ \frac{\E[\max(0,1 - y \hat y)^8]^{1/8}}{\E[\max(0,1 - y \hat y)^2]^{1/2}} \le \frac{4}{\sqrt{\eta}} \]
which verifies \eqref{eqn:hypercontractive-assumption} in this setting. 

We now consider a more general situation and show that if there is a \emph{non-negligible} portion of $x$'s such that that $y$ is noisy, hypercontractivity is still guaranteed to hold. 
Let $A_{\eta}$ be the event that $\min \{\Pr(y = 1 \mid x), \Pr(y = -1 \mid x)\} \ge \eta$. Then
\begin{align*} \E \max(0,1 - y \hat y)^2 \ge \E[ \bone(A_{\eta}) \max(0,1 - y \hat y)^2] 
&\ge \eta \E[ \bone(A_{\eta}) (1 + |\hat y|)^2] \\
&\ge \eta Q(\Pr(A_{\eta})) \E[(1 + |\hat y|)^2]
\end{align*}
where $Q$ is defined below.
In the last step, we considered the worst case event $A_{\eta}$ for given $\Pr(A_{\eta})$, which corresponds to chopping the tails off of $\hat y$; considering this example, we see the inequality holds where
where $Q : (0,1] \to (0,1]$ is an explicit function
\begin{equation} Q(p) := \min \left\{ \frac{\int_{-z_p}^{z_p} |x| e^{-x^2/2} dx}{2}, \frac{\int_{-z_p}^{z_p} x^2 e^{-x^2/2} dx}{\sqrt{2\pi}} \right\} \end{equation}
and $z_p$ is defined such that $\Pr_{g \sim N(0,1)}[|g| > z_p] = p$. Repeating the argument above yields the following result:
\begin{theorem}\label{thm:classification-hypercontractive-example}
Suppose that under \eqref{eqn:model}, there exists $\eta > 0$ such that $p_{\eta} := \Pr(\min \{\Pr(y = 1 \mid x), \Pr(y = -1 \mid x)\} \ge \eta) > 0$. Then for any $w,b$ we have that for $\hat y = \langle w, x \rangle + b$,
\[ \frac{\E[\max(0,1 - y \hat y)^8]^{1/8}}{\E[\max(0,1 - y \hat y)^2]^{1/2}} \le \frac{4}{\sqrt{\eta Q(p_{\eta})}} \]
\end{theorem}
For another example, if $y$ follows a logistic regression model $\E[y \mid x] = \tanh(\beta w^*_1 \cdot x)$ with normalization $\langle w^*_1, \Sigma w^*_1 \rangle = 1$, then  by Theorem~\ref{thm:classification-hypercontractive-example} with e.g. $\eta = 1/2$, we verify \eqref{eqn:hypercontractive-assumption} with $\tau$ a constant depending only on $\beta$. The result also holds for more general models like $\E[y \mid x] = \tanh(f(\eta_1,\ldots,\eta_k))$ as long as $f$ is not always very large. 
\subsubsection{Squared Hinge Loss and Zero-One Loss}\label{apdx:classification}

In the previous section, we discussed how our generalization bound controls the population squared hinge loss, one of the standard losses used in classification. In the context of benign overfitting, this is the canonical loss to look at because it is implicitly optimized by the max-margin predictor, also known as Hard SVM (see \cref{thm:l2-overfitting}, as well as  \cite{shamir2022implicit}). 

On the other hand, it is also very natural to look at the zero-one loss of a classifier. In general, the squared hinge loss and zero-one loss are different loss functions, and their population global optima will differ. Nevertheless, in many cases the minimizer of the squared hinge loss will also have good zero-one loss. We discuss a few situations where this occurs below.

\paragraph{General Bound on Zero-One Loss from Margin Loss.} First of all, the following bound comparing the zero-one loss and margin loss always holds --- the analogous bound for the (non-squared) hinge loss is very standard and the same argument applies to squared hinge loss:
\begin{theorem}[Classical, see e.g. \cite{shalev2014understanding}]\label{thm:zero-one-trivial}
For any $w,b$, we have that
\[ \Pr(\mathop{sgn}(\langle w, x \rangle + b) \ne y) \le L_f(w,b) \]
where $f$ is the squared hinge loss.
\end{theorem}
\begin{proof}
Observe that if $\mathop{sgn}(\hat y) \ne y$, then 
\[ \hat{f}(\hat y, y) = \max(0, 1 - y \hat y)^2 \ge 1. \]
Taking the expectation over $\hat y = \langle w, x \rangle + b$ and $y$ gives the result. 
\end{proof}
In particular, when we are in the \emph{realizable} setting, where there exists a halfspace with positive margin with zero-one loss equal to zero, then as long as we can find a near-minimizer of the squared hinge test loss, Theorem~\ref{thm:zero-one-trivial} will guarantee near-optimal zero-one loss.

\paragraph{Improved Comparison in a Noisy Setting.}
It is clear from the proof that \cref{thm:zero-one-trivial}, while very general, is not always tight. For example, \citet{zhang2004statistical,bartlett2006convexity} give improved bounds which are very useful in the case that the minimizer of the squared hinge loss over \emph{all measurable functions} is contained in the class. This includes the realizable case considered above; on the other hand, it will not generally be the case that the class of linear functions includes the minimizer over all measurable functions when there is label noise. We now describe a noisy situation where minimizing the squared hinge test loss will also minimize the zero-one test loss.

For simplicity, we consider the special case of our general setup where the response $y$ is binary (classification) and also $k = 1$, so it follows a \emph{single-index} model, or equivalently
\begin{equation}\label{eqn:sim}
y = g(\eta_1, \xi) 
\end{equation}
where $\eta_1 = \langle w^*_1, x \rangle$ and $\xi$ is independent of the covariate $x$. Note that in the following discussion, we use the additional covariance splitting notation introduced in \cref{apdx:preliminaries}. 

The following lemma shows that any near-minimizer of the loss $L_f$ will have $r(w) = \| w^{\perp} \|_{\Sigma} \approx 0$, i.e. such $w$ will be essentially along the direction of the ground truth $w^*_1$. 
\begin{lemma}\label{lem:no-orthogonal-stuff}
Suppose $(x,y) \sim \mathcal{D}$ follows a single-index model \eqref{eqn:sim}, and suppose the loss functional $f(\hat{y}, y)$ is of the form \begin{equation} \label{eqn:f-classification}
f(\hat y, y) = \ell(y \hat y)
\end{equation}
for some convex function $\ell$. Then for any $w,b$ we have
\[
L_f(w,b) - L_f(w^{\parallel},b) 
= \E[\ell(y(\langle w^{\parallel}, x \rangle + b) + g \| w^{\perp}\|_{\Sigma})] - \E[\ell(y(\langle w^{\parallel}, x \rangle + b))]
\]
where $g$ is a standard Gaussian random variable independent of everything else, and so by Jensen's inequality, we have
\[
L_f(w,b) \geq L_f(w^{\parallel},b) 
.\]
Furthermore, suppose $\ell$ is not the constant function, then the equality holds iff $\| w^{\perp}\|_{\Sigma} = 0$.
\end{lemma}

\begin{proof}
Let $w^{\parallel} = (I - Q) w$ and $w^{\perp} = Q w$. 
Expanding the definition, we have
\begin{align*}
L_f(w,b) - L_f(w^{\parallel},b) 
&= \E[\ell(y(\langle w^{\parallel}, x \rangle + \langle w^{\perp}, x \rangle + b))] - \E[\ell(y(\langle w^{\parallel}, x \rangle + b))].
\end{align*}
By the definition of $Q$, $\langle w^*_1, \Sigma w^{\perp} \rangle = 0$ and so $\langle w^{\perp}, x \rangle$ is independent of $\langle w^*_1, x \rangle$ and $\langle w^{\parallel}, x \rangle$. Hence, it also independent of $y$ due to \eqref{eqn:sim}. Let $g \sim \cN(0,1)$ be a standard Gaussian random variable independent of $x$, then it follows that
\[ L_f(w,b) - L_f(w^{\parallel},b) 
 = \E[\ell(y(\langle w^{\parallel}, x \rangle + b + g \| w^{\perp}\|_{\Sigma}))] - \E[\ell(y(\langle w^{\parallel}, x \rangle + b))].
 \]
 Moreover, since $y$ is $\{\pm 1\}$ valued and independent of $g$, $gy$ is equal in law to $g$ conditioned on $y$ and
 \[ L_f(w,b) - L_f(w^{\parallel},b) 
 = \E[\ell(y(\langle w^{\parallel}, x \rangle + b) + g \| w^{\perp}\|_{\Sigma} )] - \E[\ell(y(\langle w^{\parallel}, x \rangle + b))]. \]
 The nonnegativity of this expression now follows from Jensen's inequality, since $\ell$ is assumed to be convex, and if $\ell$ is assumed to be non-constant then the equality holds iff $\| w^{\perp}\|_{\Sigma} = 0$. 
\end{proof}

Since $\ell$ is only assumed to be convex, this includes the logistic loss, squared hinge loss, hinge loss, and squared loss (in the classification setting). The previous lemma directly implies that $\| w^{\perp}\|_{\Sigma} \to 0$ for any $w$ which approaches the optimal squared hinge loss. This means that $w$ will align with the true direction $w^*_1$; we now show that in the zero bias case, this leads to the near-optima of the squared hinge loss having optimal zero-one loss. Note that this may not be the case in more general settings, as even if $w$ is aligned with $w^*_1$, the relative size of $w$ and the bias $b$ also needs to match the ground truth in order to truly minimize the zero-one loss.

\begin{theorem}\label{thm:zero-one-consistency}
Suppose that $f(\hat y,y)$ is the squared hinge loss, so $\ell(z) = \max(0, 1 - z)^2$ in the notation of \eqref{eqn:f-classification}. Suppose with probability 1, it holds that
\begin{equation}\label{eqn:well-specified-classification}
\eta_1 \cdot \E_{\xi}[g(\eta_1,\xi)] > 0.
\end{equation}
Then every global optima of the squared hinge loss with zero bias term, $L_f(w,0)$, is of the form $w = \alpha w^*_1$ with $\alpha > 0$. Furthermore, for any $w$ we have the inequality
\[ L_f(w,0) \geq L_f(w^{\parallel},0) \geq \inf_{w} L_f(w,0)
\]
and so we have that for any sequence $w_n$ that $L_f(w_n,0) \to \inf_w L_f(w,0)$, it holds that 
\[
\Pr[sgn(\langle w_n, x \rangle) \ne y] \to \Pr[sgn(\langle w^*_1, x \rangle) \ne y].
\]
\end{theorem}

\begin{proof}
By Lemma~\ref{lem:no-orthogonal-stuff}, it suffices to consider $w$ along the direction $w^*_1$ and show that the optimal $w$ cannot point in the direction opposite to $w^*_1$. To this end, observe that
\[ \frac{\partial \ell}{\partial z} = 2(z - 1) \bone\{z \le 1\} \]
and by the chain rule, using that $L_f(\alpha w^*_1, 0) = \E[\ell(y\alpha \langle w^*_1, x \rangle)]$, we have
\[ \frac{\partial}{\partial \alpha} L_f(\alpha w^*_1, 0) = 2\E[(y\alpha \langle w^*_1, x \rangle - 1) \bone\{y \alpha \langle w^*_1, x \rangle \le 1\} y\langle w^*_1, x \rangle]. \]
Evaluating this at $\alpha = 0$ gives
\[ 
\frac{\partial}{\partial \alpha} L_f(\alpha w^*_1, 0)\Big|_{\alpha = 0} = -2\E[y\langle w^*_1, x \rangle].
\]
Applying the law of total expectation, we have shown
\[ \frac{\partial}{\partial \alpha} L_f(\alpha w^*_1, 0)\Big|_{\alpha = 0} = -2\E[\E[y \mid x] \langle w^*_1, x \rangle] < 0 \]
under the assumption of the Lemma. It is easy to see that $L_f(\alpha w^*_1, 0)$ is convex in $\alpha$, which concludes the proof of the first part. We can also have final conclusion because $L_f(w_n, 0) - L_f(w_n^{\parallel}, 0) \to 0$ implies $\|w_n^{\perp}\|_{\Sigma} \to 0$ by Lemma~\ref{lem:no-orthogonal-stuff}, and $\lim\inf_{n \to \infty} \langle w_n, w^*_1 \rangle > 0$ by the first part of the theorem.
\end{proof}

The condition \eqref{eqn:well-specified-classification} is mild and easy to check for standard generative models like logistic regression, where we have that $\E[y \mid x] = \tanh(\beta \langle w^*_1, x \rangle)$ and so $\E_{\xi}[\eta_1 y \mid x] > 0$ by Chebyshev's correlation inequality (using that $\tanh$ is an increasing function). Finally, we note that the last conclusion of Theorem~\ref{thm:zero-one-consistency} means that near-minimizers of the test loss $L_f(w,0)$ are near-minimizers of the zero one loss, under the further well-specified assumption that $sgn(\langle w^*_1, x \rangle)$ achieves the Bayes-optimal classification rate (i.e. minimum of zero-one loss over all functions). 

\subsection{Sharpness of Improved Lipschitz Contraction} \label{apdx:l1-sharpness}
In this section, we show that the Lipschitz contraction bound \eqref{eqn:lipschitz-bd} for $1$-Lipschitz loss functions $f$,
\[ (1 - o(1)) L_f(w) \le \hat{L}_f(w) + \sqrt{\frac{C_{\delta}(w)^2}{n}} \]
has sharp constants in the case of the $L_1$ loss $f(\hat y, y) := |y - \hat y|$. This shows that the only way to tighten the bound further is to consider one with a different functional form (e.g. the Moreau envelope bound with the Huber test loss). In particular, the Moreau envelope version of the bound is significantly more useful when looking at interpolators.

\paragraph{Data Distribution.} We will show tightness in the setting of the junk features model. Let's consider 
\[ 
x \sim \cN(0,\Sigma), \quad y \sim \cN(0,\sigma^2)
\]
where the response $y$ is independent of the covariate $x$ and the covariance $\Sigma$ is given by
\[ 
\Sigma = \begin{bmatrix}
1 & 0 \\
0 & \frac{\lambda_n}{d_J}  I_{d_J}
\end{bmatrix}.
\]
In addition, following \citet{junk-feats}, we consider the asymptotics where first, for fixed $n$, we take $d_J \to \infty$, and then we take $n \to \infty$ with $\lambda_n = \sqrt{n}$. 

\paragraph{Predictor.} The $w$ which demonstrates tightness is of the form
\[ w = (r, w_{\sim 1}) \]
where $r> 0$ is a parameter and $w_{\sim 1}$ is constructed based on the training data $(x_i,y_i)_{i = 1}^n$ to minimize $\|w_{\sim 1}\|_2$ given the constraint
\[ \langle w_{\sim 1}, x_{i, \sim 1} \rangle = \sigma \cdot \operatorname{sgn}(y_i - r x_{i,1}). \]

\paragraph{Tightness.} Since $w_{\sim 1}$ plays no role in a new prediction\footnote{This is because $w_{\sim 1}$ lies in the span of $x_{i, \sim 1}$, but a new sample from $x_{\sim 1}$ will be almost surely orthogonal to all $x_{i, \sim 1}$ in the training set as $d_J \to \infty$.}, we have
\[
\lim_{d_J \to \infty} L_f(w) = \E |y - r x_1|
\]
and as $n \to \infty$
\begin{equation*}
    \begin{split}
        \hat{L}_f(w) 
        &= \frac{1}{n} \sum_{i = 1}^n |y_i - r x_{i, 1} - \langle w_{\sim 1}, x_{i, \sim 1} \rangle | = \frac{1}{n} \sum_{i = 1}^n |y_i - r x_{i, 1} - \sigma \cdot \operatorname{sgn}(y_i - r x_{i,1}) | \\
        &= \frac{1}{n} \sum_{i = 1}^n ||y_i - r x_{i,1}| - \sigma| \approx \E |y-rx_1| - \sigma
    \end{split}
\end{equation*}
because $\Pr(|y - r x_1| < \sigma) \to 0$ as $r \to \infty$ and $\frac{1}{n} \sum_{i = 1}^n |y_i - r x_{i,1}| \to \E |y-rx_1|$ by the law of large numbers. Therefore, the actual generalization gap for $w$ will be
\begin{equation}
    \lim_{r \to \infty} \lim_{n \to \infty} \lim_{d_J \to \infty} \, L_f(w) - \hat{L}_f(w) = \sigma.
\end{equation}
On the other hand, following the analysis from \citet[Appendix B]{junk-feats}, we have\footnote{Again, this is because the vectors $x_{1, \sim 1}, \ldots, x_{n, \sim 1}$ will asymptotically be orthogonal to each other and have norm $\sqrt{\lambda_n}$ and we use each of them to fit a label of size $\sigma$.}
\[
\lim_{d_J \to \infty} \|w\|_2^2 = r^2 + \frac{\sigma^2 n}{\lambda_n},
\]
and by taking $C_{\delta}(w)$ as in Lemma~\ref{lem:gen-ball1} and using $\Tr \Sigma = 1 + \lambda_n$, the bound \eqref{eqn:lipschitz-bd} gives
\begin{equation}
    L_f(w) - \hat{L}_f(w) \le \|w\|_2 \sqrt{\frac{1 + \lambda_n}{n}}.
\end{equation}
Since $\| w\|_2 \approx \sigma \sqrt{\frac{n}{\lambda_n}}$ and $\sqrt{\frac{1+\lambda_n}{n}} \approx \sqrt{\frac{\lambda_n}{n}}$, the value of the bound converges to $\sigma$ as $n \to \infty$. 
\section{Proof of Theorem~\ref{thm:training-error}} \label{apdx:training-error}
\LocalGW*
\begin{proof}
We can write the training error as a minmax problem by introducing a variable $\hat y = X w$ and using Lagrange multipliers to write the minimum of the training loss (Primary Optimization) as
\[ 
 \Phi := \min_{w \in \cK, b_0 \in \cB, \hat y} \max_{\lambda} \frac{1}{n} \sum_{i = 1}^n f(\hat y_i, y_i)  + \langle \lambda, \hat y - X^{\parallel} w^{\parallel} - X^{\perp} w^{\perp}- b_0 \rangle. \]
Note that here we are using the additional covariance splitting notation introduced in \cref{apdx:preliminaries}, and we interpret the subtraction of $b_0$ as entrywise (equivalently, as subtracting the vector $b_0 \vec{1}$).

Similarly, define the Auxiliary Optimization problem (which will be related to the Primary Optimization below) as a random variable depending on independent random vectors $g \sim \cN(0,I_n)$ and $h \sim \cN(0,I_d)$ as
\[
 \Psi := \min_{w \in \cK, b_0 \in \cB, \hat y} \max_{\lambda} \frac{1}{n} \sum_{i = 1}^n f(\hat y_i,y_i)  + \langle \lambda, \hat y - X^{\parallel} w^{\parallel}- b_0 \rangle - \langle \lambda, g \rangle \|w^{\perp}\|_{\Sigma^{\perp}} - \langle h, Q \Sigma^{1/2} w^{\perp} \rangle\|\lambda\| \]
and truncated versions of both problems
\[ \Phi_s := \min_{w \in \cK, b_0 \in \cB, \hat y} \max_{\|\lambda\| \le s} \frac{1}{n} \sum_{i = 1}^n f(\hat y_i, y_i)  + \langle \lambda, \hat y - X^{\parallel} w^{\parallel} - X^{\perp} w^{\perp}- b_0 \rangle\]
and
\[
 \Psi_s := \min_{w \in \cK, b_0 \in \cB, \hat y} \max_{\|\lambda\| \le_s} \frac{1}{n} \sum_{i = 1}^n f(\hat y_i, y_i)  + \langle \lambda, \hat y - X^{\parallel} w^{\parallel}- b_0 \rangle - \langle \lambda, g \rangle \|w^{\perp}\|_{\Sigma^{\perp}} - \langle h, Q \Sigma^{1/2} w^{\perp} \rangle\|\lambda\| \]
 By definition, we have $\Psi_s \le \Psi$ and by applying Lemma~\ref{lem:truncation2} and Theorem~\ref{thm:gmt} %
we have that $\Pr(\Phi > t) \le \lim_{s \to \infty} \Pr(\Phi_s > t) \le 2\lim_{s \to \infty} \Pr(\Psi_s > t) \le 2 \Pr(\Psi > t)$.

It remains to prove a high probability upper bound on the Auxiliary Optimization $\Psi$. Observe that we can rewrite%
\begin{align*}
\Psi = \min_{w \in \cK, b_0 \in \cB, \hat y} \max_{\lambda} \frac{1}{n} \sum_{i = 1}^n f(y_i, \hat y_i)  + \langle \lambda, \hat y - X^{\parallel} w^{\parallel} - g\|w^{\perp}\|_{\Sigma^{\perp}}- b_0 \rangle - \langle h, (\Sigma^{\perp})^{1/2} w^{\perp} \rangle\|\lambda\| 
\end{align*}
and then solving the optimization over $\lambda$ gives
\begin{align*}
\Psi &= \min_{w \in \cK, b_0 \in \cB, \hat y : \|\hat y - X^{\parallel} w^{\parallel} - g\|w^{\perp}\|_{\Sigma^{\perp}}- b_0\| \le \langle (\Sigma^{\perp})^{1/2} h,  w^{\perp} \rangle}  \frac{1}{n} \sum_{i = 1}^n f(y_i, \hat y_i) \\
&= \min_{w \in \cK, b_0 \in \cB, \hat y : \|\hat y - X^{\parallel} w^{\parallel} - g\|w^{\perp}\|_{\Sigma^{\perp}}- b_0\| \le |\langle (\Sigma^{\perp})^{1/2} h,  w^{\perp} \rangle|}  \frac{1}{n} \sum_{i = 1}^n f(y_i, \hat y_i)
\end{align*}
where the last equality is by observing that if $\langle \Sigma^{\perp} h,  w^{\perp} \rangle$, we can flip the sign of $w^{\perp}$ to get a feasible point of the constraint with the absolute value and with the same objective value. Next, applying Lemma~\ref{lem:truncation2} we can rewrite this as
\begin{align*}
\Psi &= \lim_{r \to \infty} \min_{w \in \cK,b_0 \in \cB, \hat y} \max_{\lambda \in [0,r]} \frac{1}{n} \sum_{i = 1}^n f(y_i, \hat y_i) + \lambda\left(\frac{1}{n} \|\hat y - X^{\parallel} w^{\parallel} - g \|w^{\perp}\|_{\Sigma^{\perp}}- b_0\|^2 - \frac{1}{n} \langle (\Sigma^{\perp})^{1/2} h,  w^{\perp} \rangle^2\right) \\
&= \lim_{r \to \infty} \min_{w \in \cK, b_0 \in \cB} \max_{\lambda \in [0,r]} \frac{1}{n} \sum_{i = 1}^n f_{\lambda}(y_i, (X^{\parallel} w^{\parallel})_i + g_i \|w^{\perp}\|_{\Sigma^{\perp}} + b_0) - \lambda \frac{1}{n} \langle (\Sigma^{\perp})^{1/2} h,  w^{\perp} \rangle^2 \\
&\le \min_{w \in \cK, b_0 \in \cB} \max_{\lambda \ge 0} \frac{1}{n} \sum_{i = 1}^n f_{\lambda}(y_i, (X^{\parallel} w^{\parallel})_i + g_i \|w^{\perp}\|_{\Sigma^{\perp}} + b_0) - \lambda \frac{1}{n} \langle(\Sigma^{\perp})^{1/2} h,  w^{\perp} \rangle^2
\end{align*}
where in the second equality we used the definition of the Moreau envelope and the minimax theorem \citep{sion1958general} to move the minimum over $\hat y$ inside the max.%

Next, observing that the first term only depends on $\phi(w)$ we can write this equivalently as
\[ \min_{\substack{\phi(w): w \in \cK \\ b_0 \in \cB}} \max_{\lambda \ge 0} \left[\frac{1}{n} \sum_{i = 1}^n f_{\lambda}(y_i, (X^{\parallel} w^{\parallel})_i + g_i \|w^{\perp}\|_{\Sigma^{\perp}} + b_0) - \lambda \frac{1}{n} \max_{u  \in \cK: \phi(u) = \phi(w)} \langle (\Sigma^{\perp})^{1/2} h,  u^{\perp} \rangle^2\right]\]
which proves the conclusion, using that $ (X^{\parallel} w^{\parallel})_i + g_i \|w^{\perp}\|_{\Sigma^{\perp}} + b_0$ is equivalent in law to $\langle \tilde w, \tilde x \rangle + b_0$ where $\tilde w = \phi(w)$.
%
\end{proof}

\subsection{Geometric Interpretation}
In this section, we elaborate on the discussion from Section~\ref{sec:localGW} to explain how the result Theorem~\ref{thm:training-error} is a dual result which witnesses tightness of Theorem~\ref{thm:main-gen}, and to give a geometric interpretation of both results by connecting them to summary functional $\psi(w,b)$ defined in \eqref{eqn:psi}. A couple of new results are also established in this subsection, but they are not used in the rest of the paper. 

Recall that the main result of this paper, Theorem~\ref{thm:main-gen}, establishes an upper bound on the test error of an arbitrary predictor $w$ in terms of the training error $\hat{L}_f(w,b)$ and complexity functional $C_{\delta}(w)$. How can we choose the complexity functional $C_{\delta}(w)$ to optimize the bound? In this section, we show that when analyzing the Constrained Empirical Risk Minimizer over $(w,b) \in \cK \times \cB$ with $\cK,\cB$ bounded convex sets %
\[ (\hat w, \hat b) = \arg\min_{w \in \cK, b \in \cB} \hat{L}_f(w,b) \]
choosing $C_{\delta}(w)$ based on the \emph{local Gaussian width} of the projected set $Q \cK$ will result in an essentially tight generalization bound. (Recall from Definition~\ref{eqn:Q-defn} that $Q$ is the projection orthogonal to the space $w^*_1,\ldots,w^*_k$ which the true regression function in the GLM depends upon.)

The characterization of the performance of constrained ERM we present connects to and builds upon ideas and themes explored previously in a long line of work in the M-estimation literature. For instance, the previous work of \citet{thrampoulidis2018precise} (see also references within and our \cref{sec:related}) gives a similar asymptotic characterization for the performance of constrained/regularized ERM. Compared with that work, here we focus on non-asymptotic results, which apply outside of the proportional scaling limit, and we establish a connection between this characterization and generalization bounds (which apply to all predictors, not just the ERM). Another difference to that result is that ours applies to generative models of the data beyond just linear regression, in particular GLMs, a setting which has been considered in other works in the CGMT literature \citep[e.g.][]{montanari2019generalization,liang2020precise,thrampoulidis2020theoretical}.
In the special case of regression with the squared loss, we recover the nonasymptotic local Gaussian width theory of \citet{optimistic-rates}. %

\paragraph{Informal Summary.} Before stating the formal results, we start with an informal discussion summarizing the key results and their geometric interpretation.  First, we observe that the conclusion of our main result (Theorem~\ref{thm:main-gen}) can be naturally rearranged as a lower bound on the training loss:
\begin{equation}\label{eqn:informal-lb}
\max_{\lambda \ge 0} \left[ L_{f_{\lambda}}(w,b) - \frac{\lambda C(w)^2}{n} \right] \le \hat{L}_f(w,b), 
\end{equation}
where for this informal overview we write $C(w) = C_{\delta}(w)$ to omit the dependence on the failure probability, and also ignore the small error term $\epsilon_{\lambda,\delta}$. A key observation at this point is that the test error $L_{f_{\lambda}}(w,b)$ depends on $w$ only through its projection $\phi(w)$ from Definition~\ref{eqn:Q-defn}: in other words, via
its projection onto the span of $w^*_1,\ldots,w^*_k$ and its Mahalanobis norm in the orthogonal space $\|\Sigma^{1/2} Q w\|$. It is natural to choose $C(w)$ depending only on $\phi(w)$.

Hence a natural choice of $C(w)$
is the (local) \emph{Gaussian width}
\begin{equation}\label{eqn:cw-informal}
C(w) := \E_{x \sim \cN(0,\Sigma)} \sup_{v \in \cK_{\phi(w)}} \langle Q v, x \rangle 
\end{equation}
where the localized set $\cK_{\phi(w)}$ is defined as
\[ \cK_{\phi(w)} := \{ v \in \cK : v^{\parallel} = w^{\parallel}, r(v) \le r(w) \} \]
and the notation indicates that this set only depends on $w$ through $\phi(w)$, equivalently $w^{\parallel}$ and $r(w)$.
With this choice of $C(w)$, we define the summary functional
\begin{equation}\label{eqn:psi}
\psi(w,b) = \psi(\phi(w),b) := \max_{\lambda \ge 0} \left[ L_{f_{\lambda}}(w,b) - \frac{\lambda C(w)^2}{n} \right] 
\end{equation}
to be the left hand side of \eqref{eqn:informal-lb} (where the notation $\psi(\phi(w),b)$ is used to indicate that $\psi$ depends on $w$ only through $\phi(w)$). 
We will obtain two major conclusions:
\begin{enumerate}
    \item Formalizing the previous discussion, the conclusion of Theorem~\ref{lem:local-gw} is that with some small finite sample corrections, this choice of $C(w)$ satisfies the assumption of Theorem~\ref{thm:main-gen} and so $\psi(w,b)$ indeed lower bounds the training error $\hat{L}_f(w,b)$ as in \eqref{eqn:informal-lb}. %
    \item The conclusion of Theorem~\ref{thm:training-error} is that the lower bound in \eqref{eqn:informal-lb} with this $C(w)$ is tight for the constrained ERM. In other words, with high probability
    \[ \min_{w \in \cK, b \in \cB} \hat{L}_f(w,b) \approx \min_{w \in \cK, b \in \cB} \psi(w,b) \]
    where the right-hand-side is deterministic (and the right hand side optimization depends on $w$ only through the low-dimensional vector $\phi(w)$).
    This is established by upper bounding the training error via an application of the Convex Gaussian Minmax Theorem. 
\end{enumerate}
Combining the two conclusions, we see that when we apply our generalization bound (Theorem~\ref{thm:main-gen}) with a sufficiently tight choice of $C(w)$ based on the local gaussian width and the optimal envelope parameter $\lambda$, it will predict the actual generalization error of the constrained ERM. So our generalization bound is tight in a pretty general situation; in particular, when the constrained ERM is consistent under proportional scaling (the setting most commonly considered in the asymptotic CGMT literature).

To clarify the geometric interpretation of this result, we also show in Lemma~\ref{lem:psi-convex} that with this choice of $C(w)$, the left hand side of \eqref{eqn:informal-lb} will be convex in $w$ and $b$; hence, for a fixed upper bound on the training error there is a corresponding sublevel set of the convex function which consists of the points whose training error satisfy the constraint, and as the upper bound shrinks this set will narrow around the minimum of the convex function.
 
 \paragraph{Formal Results.} 
 %
%
%
First, we formalize the idea that $\psi(w,b)$ is a lower bound on the training error $\hat{L}_f(w,b)$. As in the general Theorem~\ref{thm:main-gen}, we take the one-sided concentration of the low-dimensional surrogate problem as an assumption to state a general result, since the precise details of that concentration estimate will depend on the exact setting. To give a finite sample result, we define a straightforward approximation $C_{\delta,\rho}(w)$ of the local gaussian width functional \eqref{eqn:cw-informal} which is defined based on a $\rho$-net approximation of $\phi(\cK)$, and includes the dependence on the failure probability $\delta$; since $\phi(\cK)$ is a low-dimensional set living in $\mathbb{R}^{k + 2}$, the contribution of this correction (just like the contribution from the error term in the low-dimensional concentration assumption \eqref{eqn:low-dimension-concentration}) will become negligible if we consider an asymptotic setting $n \to \infty$ with $k$ fixed. 
\begin{lemma}\label{lem:local-gw}
Let $\cK \subset \mathbb{R}^d$ and $\cB \subset \mathbb{R}$. Suppose that we have assumption \eqref{eqn:low-dimension-concentration} from Theorem~\ref{thm:main-gen} with error parameter $\epsilon_{\lambda,\delta}(\tilde w, \tilde b)$ uniformly over envelope parameter $\lambda \ge 0$. Let $\rho > 0$ be arbitrary, and let $\mathcal{S}$ be a proper $\rho$-covering in Euclidean norm of the set $\{\phi(w) : w \in \cK\}$ so that for every $w \in \cK$ there exists $w'$ with $\phi(w') \in \mathcal{S}$ such that
\[ \|\phi(w) - \phi(w')\|_2 < \rho. \]
and define (where as above, $w'$ denotes the element in the covering corresponding to $w$)
\[ C_{\delta,\rho}(w) := \E_{x \sim \cN(0,\Sigma)} \left[ \sup_{v \in \cK_{\phi(w'),\rho}} \langle Qv, x \rangle \right] + (r(w') + \rho)\sqrt{2\log(16|\mathcal S|/\delta)} \]
where
\[ \cK_{\phi(w),\rho} := \{ v \in \cK : \|v^{\parallel} - w^{\parallel}\|_{\Sigma} < \rho, r(v) \le r(w) + \rho \}. \]
Then:
\begin{enumerate}
    \item With probability at least $1 - \delta/4$, we have for all $w \in \cK$ that
    \[ \langle Q w, x \rangle \le C_{\delta,\rho}(w), \]
    i.e. the assumption \eqref{eqn:complexity-defn} of Theorem~\ref{thm:main-gen} is satisfied.
    \item As an immediate consequence of Theorem~\ref{thm:main-gen}, we have with probability at least $1 - \delta$ that
    \[ \sup_{\lambda \ge 0} \left[L_{f_{\lambda}}(w,b) - \epsilon_{\lambda,\delta}(\phi(w),b) - \lambda \frac{C_{\delta,\rho}(w)^2}{n}\right] \le \hat{L}_f(w,b) \]
    uniformly over $w \in \cK$, $b \in \cB$.
\end{enumerate}
\end{lemma}
\begin{proof}
We only need to check the first conclusion, since the second one follows immediately by Theorem~\ref{thm:main-gen}. First, observe from expanding the definitions that 
\[  \|w^{\parallel} - (w')^{\parallel}\|_{\Sigma}^2 + (r(w) - r(w'))^2 = \|\phi(w) - \phi(w')\|_2^2 < \rho \]
so that $w \in \cK_{\phi(w'),\rho}$. Next, observe by applying Gaussian concentration (Theorem~\ref{thm:gaussian-concentration}) and the union bound over $\mathcal S$ that with probability at least $1 - \delta/4$, for $x \sim \cN(0,\Sigma)$ and every $w'$ with $\phi(w') \in \mathcal S$ we have that
\[ \sup_{v \in \cK_{\phi(w'),\rho}} \langle Qv, x \rangle \le \E_{x \sim \cN(0,\Sigma)} \left[ \sup_{v \in \cK_{\phi(w'),\rho}} \langle Qv, x \rangle \right] + (r(w') + \rho)\sqrt{2\log(16|\mathcal S|/\delta)} \]
where we use that the supremum is $(r(w') + \rho)$-Lipschitz because every $v \in \mathcal{K}_{\phi(w),\rho}$ satisfies $\|\Sigma^{1/2} Q v\| = r(v) \le r(w') + \rho$, and the supremum of Lipschitz functions is Lipschitz with the same constant.  Since we showed that $w \in \cK_{\phi(w'),\rho}$, we then have that
\[ \langle Q w, x \rangle \le C_{\delta,\rho}(w) \]
as desired. 
\end{proof}
We now discuss how Theorem~\ref{thm:training-error} formalizes the idea that the training error of ERM is the minimum of $\psi(w,b)$. %
To understand the statement, take $w_0,b_0$ to be minimizers of $\psi(w,b)$. We observe that there exists such minimizers so that
\[ C(w) = \E_{x \sim \cN(0,\Sigma)} \sup_{v \in \cK_{\phi(w)}} \langle Q v, x \rangle  = \E_{x \sim \cN(0,\Sigma)} \sup_{\phi(v) = \phi(w)} \langle Q v, x \rangle, \]
i.e. so that the optimizing $v$ satisfies $r(v) = r(w)$, otherwise we can replace $w$ by $v$ without reducing $\psi$. %
Given this observation, we have that the quantity \eqref{eqn:training-error-assumption} will concentrate about $\psi(w_0,b_0)$
and the best choice of $w_0,b_0$ to make is the minimizer of this quantity, so that we set $\tau$ to be
\[ \tau \approx \min_{w_0 \in \cK, b_0 \in \cB} \psi(w_0,b_0) \]
and this upper bounds the training error of constrained ERM
as discussed in the informal overview. Again, see Theorem~\ref{thm:training-error} for the formal version of this.

Finally, we formalize the claim that the summary functional $\psi(w,b)$ defined in \eqref{eqn:psi} is convex. This is not used in the proofs of the main results above, but (as explained earlier) makes the geometric interpretation of the result clearer, and generalizes the convexity of  analogous summary functionals observed in previous work for the well-specified regression setting, including \cite{thrampoulidis2018precise,optimistic-rates}. We note that this convexity will be approximate for the finite-sample version 
$\sup_{\lambda \ge 0} \left[L_{f_{\lambda}}(w,b) - \epsilon_{\lambda,\delta}(\phi(w),b) - \lambda \frac{C_{\delta,\rho}(w)^2}{n}\right]$
in the conclusion of Theorem~\ref{lem:local-gw}, because of the finite-sample error terms like $\epsilon_{\lambda,\delta}$. In some settings, the finite-sample version of the functional can also be made to be convex: see \cite{optimistic-rates} for the case of  regression with squared loss.
\begin{lemma}\label{lem:psi-convex}
Given that the loss $f(y,\hat y)$ is convex in $\hat y$ and $\cK,\cB$ are convex sets, the functional $C(w) = C(\phi(w))$ defined in \eqref{eqn:cw-informal} is concave as a function of $\phi(w)$ and 
$\psi(w,b) = \psi(\phi(w),b)$ defined in \eqref{eqn:psi} is convex as a function of $(\phi(w),b)$.
\end{lemma}
\begin{proof}
First we show $C(w)$ is concave as a function of $\phi(w)$. Recall from \eqref{eqn:cw-informal} that %
\[ C(w) = \E_{x \sim \cN(0,\Sigma)} \sup_{v \in \cK_{\phi(w)}} \langle Q v, x \rangle  \]
where
\[ \cK_{\phi(w)} = \{ v \in \cK : v^{\parallel} = w^{\parallel}, r(v) \le r(w) \}. \]
It suffices to prove that for any $x$, the function
\[ F(w) = F(\phi(w)) := \sup_{v \in \cK_{\phi(w)}} \langle Q v, x \rangle \]
is concave in $\phi(w)$.
If $\phi(w) = \alpha \phi(w_1) + (1 - \alpha) \phi(w_2)$, $v_1$ is a maximizer of $F(w_1)$ and $v_2$ is a maximizer of $F(w_2)$ then 
\[ r(\alpha v_1 + (1 - \alpha) v_2) \le \alpha r(v_1) + (1 - \alpha) r(v_2) \le \alpha r(w_1) + (1 - \alpha) r(w_2) = r(w) \]
so $\alpha v_1 + (1 - \alpha) v_2 \in \cK_{\phi(w)}$ and so
\[ F(w) \ge \langle Q(\alpha v_1 + (1 - \alpha) v_2), v \rangle = \alpha F(w_1) + (1 - \alpha) F(w_2) \]
which proves the concavity. 

Next we prove convexity of $\psi$. 
By expanding the definition of the Moreau envelope, we see that
\begin{align*} 
\psi(w,b) 
&= \max_{\lambda \ge 0} \left[ L_{f_{\lambda}}(w,b) - \frac{\lambda C(w)^2}{n} \right] \\
&= \max_{\lambda \ge 0} \left[\E \min_u f(y,u) + \lambda(u - \langle w, x \rangle - b)^2 - \frac{\lambda C(w)^2}{n} \right] \\
&= \max_{\lambda \ge 0} \left[\min_g \E f(y, \langle w, x \rangle + b + g(x,y)) + \lambda g(x,y)^2 - \frac{\lambda C(w)^2}{n} \right] \\
&=\min_{g : \sqrt{\E g(x,y)^2} \le C(w)/\sqrt{n}}  \E f(y, \langle w, x \rangle + b + g(x,y)) 
\end{align*}
and we claim the final expression is convex in $w$ and $b$. This follows from Lemma~\ref{lem:convex-min}
because the objective $\E f(y, \langle w, x \rangle + b + g(x,y))$ is jointly convex in $\phi(w),g,b$, and the minimization is over the constraint $\sqrt{\E g(x,y)^2} - C(w)/\sqrt{n} \le 0$ which is a jointly convex constraint.%
\end{proof}
The following lemma is a version of a standard fact in convex analysis, see e.g. Section 3.2.5 of \cite{boyd2004convex}.
\begin{lemma}\label{lem:convex-min}
Suppose that real-valued functions $f(x,y)$ and $g(x,y)$ are both jointly convex in $(x,y) \in \mathcal{X} \times \mathcal{Y}$ where $\mathcal{X},\mathcal{Y}$ are convex sets. Then
\[ h(x) := \inf_{y \in \mathcal{Y} : f(x,y) \le 0} g(x,y) \]
is a convex function on $\mathcal{X}$.
\end{lemma}
\begin{proof}
Suppose that $x = \alpha x_1 + (1 - \alpha) x_2$ and $y_1,y_2$ are arbitrary points such that both $f(x_1,y_1), f(x_2,y_2) \le 0$. By joint convexity, we have that
\[ f(\alpha x_1 + (1 - \alpha) x_2, \alpha y_1 + (1 - \alpha) y_2) \le \alpha f(x_1,y_1) + (1 - \alpha)f(x_2,y_2) \le 0 \]
and so
\[ h(x) \le g(\alpha x_1 + (1 - \alpha) x_2,\alpha y_1 + (1 - \alpha) y_2) \le \alpha g(x_1,y_1) + (1 - \alpha) g(x_2,y_2). \]
Taking the infimum over all such $y_1,y_2$ such $f(x_1,y_1), f(x_2,y_2) \le 0$ proves that
\[ h(x) \le \alpha h(x_1) + (1 - \alpha) h(x_2) \]
which shows the convexity. 
\end{proof}

\emph{A simple example.} To sketch how the summary functional $\psi$ works and connect to the previous literature, we consider a simple example (Ordinary Least Squares). To start with, we consider a well-specified model with $y = \langle w^*, x \rangle + \xi$ where $\xi$ is noise independent of $x$ with variance $\sigma^2$ and bounded eighth moment. Then the summary functional for $f$ the squared loss and taking $C(w) \approx \|Q w\|_{\Sigma}\sqrt{d}$ is (using Lemma~\ref{lem:calculus-lambda})
\[ \psi(w,b) = (\sqrt{L(w,b)} - \|Q w\|_{\Sigma} \sqrt{d/n})^2 = (\sqrt{\sigma^2 + \|w - w^*\|_{\Sigma}^2 + b^2} - \|Q w\|_{\Sigma} \sqrt{d/n})^2.  \]
Note $\|w - w^*\|_{\Sigma}^2 = \|w^{\parallel} - w^*\|_{\Sigma}^2 + \|Q w\|_{\Sigma}^2$ by the Pythagorean Theorem. 
To minimize $\psi$, it is  optimal to take $w^{\parallel} = w^*$ and $b = 0$ which leaves choosing $r(w) = \|Q w\|_{\Sigma}$ to minimize
\[  (\sqrt{\sigma^2 + r(w)^2} - r(w) \sqrt{d/n})^2  \]
and this in turn is minimized at
$r(w) = \sigma^2(d/n)/(1 - d/n)$, which will be the excess test loss of the constrained ERM. Note that to make the calculation easy, we considered a well-specified model and the summary functional reduced to the same one as in \cite{optimistic-rates} once we solved the optimization over $\lambda$, and the calculation can be made rigorous and nonasymptotic following the arguments there; see also \cite{thrampoulidis2018precise} and references for related asymptotic results. In this example, it can be checked that the calculation generalizes in a straightforward way to misspecified models under our general assumptions, if we let $w^*$ to be the minimizer of the population squared loss (i.e. the oracle predictor.)  and defining the excess test loss to be the gap compared to $w^*$.
\section{\texorpdfstring{$\ell_2$}{l2} Benign Overfitting}
\label{apdx:benign-overfitting}
In this section, we give the proofs of the result for benign overfitting under the $\ell_2$ condition. We continue to make use of the additional covariance split notation introduced in \cref{apdx:preliminaries}.
\subsection{Properties of Sqrt-Lipschitz Functions}
In this section, we establish some elementary properties of the squares of Lipschitz functions. This is a natural class to consider since in particular, the squared loss and squared hinge loss both fall into this class of functions.
We say a function $f : \mathbb{R} \to \mathbb{R}_{\ge 0}$ is $L$-sqrt-Lipschitz if $\sqrt{f}$ is $L$-Lipschitz. Since
\[\frac{1}{2} f(x)^{-1/2} f'(x) =  \frac{d}{dx} \sqrt{f(x)} \]
we can equivalently say that a function $f$ is $L$-sqrt-Lipschitz if
\[ |f'(x)| \le 2L \sqrt{f(x)} \]
for all $x$.
Based on this characterization, one can observe that any $H$-smooth and nonnegative function is $\sqrt{H}$-sqrt-Lipschitz; this is proved in Lemma 2.1 of \cite{srebro2010optimistic} although not using this terminology. We proceed to establish some useful properties of sqrt-Lipschitz functions. First, we show that $L$-sqrt-Lipschitz functions form a convex set.
\begin{lemma}
If $f$ is $L$-sqrt-Lipschitz convex and $g$ is $L$-sqrt-Lipschitz convex then so is $(1 - \alpha)f + \alpha g$ for any $\alpha \in [0,1]$.
\end{lemma}
\begin{proof}
Observe that
\begin{align*} |(1 - \alpha)f'(x) + \alpha g'(x)| \le (1 - \alpha)|f'(x)| + \alpha |g'(x)| 
&\le 2L[(1 - \alpha) \sqrt{f(x)} + \alpha \sqrt{g(x)}] \\
&\le 2L\sqrt{(1 - \alpha)f(x) + \alpha g(x)} 
\end{align*}
where the second step is the assumption that $f$ and $g$ are $L$-sqrt-Lipschitz and the last step uses the concavity of the square-root function.
\end{proof}
Next, the following lemma formalizes the idea that sqrt-Lipschitz functions satisfy a local and scale-sensitive version of the Lipschitz property.
\begin{lemma}\label{lem:sqrt-lipschitz-convex}
Suppose that $f(x)$ is convex and $L$-sqrt-Lipschitz. Then for any $\epsilon > 0$,
\[ f(x + h) \ge (1 - \epsilon)f(x) - L^2h^2/\epsilon. \]
\end{lemma}
\begin{proof}
Observe that
\[ f(x + h) \ge f(x) + f'(x)h \ge f(x) - 2L\sqrt{f(x)}|h| \ge f(x) - \epsilon f(x) - L^2h^2/\epsilon \]
where the first inequality is by convexity, the second inequality is by the $L$-sqrt-Lipschitz property, and the third inequality is the AM-GM inequality. 
\end{proof}
This leads to a corresponding local Lipschitz property of the training loss.
\begin{lemma}\label{lem:training-smooth}
Let $\epsilon \in (0,1)$ be arbitrary, let $w_0 \in \mathbb{R}^d$ and $b_0 \in \mathbb{R}$. Suppose that nonnegative loss function $f(\hat y, y)$ is convex and $L$-sqrt-Lipschitz in $\hat y$.
The following inequality holds determinsitically for any $x_1,\ldots,x_n \in \mathbb{R}^d$, $y_1,\ldots,y_n \in \mathbb{R}$, $w \in \mathbb{R}^d$, and $b \in \mathbb{R}$:
\[ (1 - \epsilon)\hat{L}_f(w,b) \le \hat{L}_f(w_0,b_0) + \frac{2L^2}{\epsilon n} \sum_{i = 1}^n \langle w - w_0, x_i \rangle^2 + 2(b - b_0)^2/\epsilon  \]
\end{lemma}
\begin{proof}
By applying Lemma~\ref{lem:sqrt-lipschitz-convex}, we have that
\[ f(\langle w_0, x_i \rangle + b_0, y_i) 
\ge (1 - \epsilon)f(\langle w, x_i \rangle + b, y_i) - L^2(\langle w - w_0, x_i \rangle + (b - b_0))^2/\epsilon \]
and then applying the inequality $(a + b)^2 \le 2a^2 + 2b^2$ gives
\[ f(\langle w_0, x_i \rangle + b_0, y_i) 
\ge (1 - \epsilon)f(\langle w, x_i \rangle + b, y_i) - 2L^2\langle w - w_0, x_i \rangle^2 - 2(b - b_0)^2/\epsilon.\]
Summing this inequality over $i$ from $1$ to $n$ and rearranging gives the conclusion. 
\end{proof}
\subsection{Norm Bounds}

\NormBound*
\begin{proof}
By Theorem~\ref{thm:training-error} it suffices to show that with probability at least $1 - \delta/2$,
\[ \min_{w_0 \in \cK, b_0 \in \cB} 
\max_{\lambda \ge 0} \left[\frac{\lambda}{1 + \lambda} \frac{1}{n} \sum_{i = 1}^n f(y_i, (X^{\parallel} w_0^{\parallel})_i + b_0 + g_i \|w_0^{\perp}\|_{\Sigma^{\perp}}) - \frac{\lambda}{n} \langle Q x,  w_0^{\perp} \rangle^2\right] = 0. \]
Using Lemma~\ref{lem:calculus-lambda}, it suffices to show with probability at least $1 - \delta/2$ that there exists $w_0,b_0$ such that
\[ \frac{1}{n}  \sum_{i = 1}^n f(y_i, (X^{\parallel} w_0^{\parallel})_i + b_0 + g_i \|w_0^{\perp}\|_{\Sigma^{\perp}}) \le \frac{1}{n} \langle Q x,  w_0^{\perp} \rangle^2. \]
Decompose $w_0 = w_0^{\parallel} + w_0^{\perp}$ where $w_0^{\perp} = Q w_0$; then
using Lemma~\ref{lem:training-smooth}, we have that for any $\epsilon > 0$
\[ (1 - \epsilon) \hat{L}_f(w,b_0) \le \hat{L}_f(w^{\parallel}, b_0) + \frac{2}{\epsilon n} \sum_{i = 1}^n g_i^2 \|w_0^{\perp}\|_{\Sigma^{\perp}}^2 \]
so it suffices to show that with probability $1 - \delta/2$, there exists $w_0, b_0$ and $\epsilon > 0$ with
\[ \frac{1}{1 - \epsilon} \hat{L}_f(w^{\parallel}, b_0) + \frac{2}{\epsilon(1 - \epsilon) n} \sum_{i = 1}^n g_i^2 \|w_0^{\perp}\|_{\Sigma^{\perp}}^2 \le \frac{1}{n} \langle Q x,  w_0^{\perp} \rangle^2. \]
We consider $w_0^{\perp} = \alpha \frac{Q x}{\|Q x\|}$ for some constant $\alpha > 0$ to be determined later. Observe that $Q x$ is equal in law to $(\Sigma^{\perp})^{1/2} H$ for $H \sim \cN(0,I_d)$ with $H$ independent of $X^{\parallel}$ and $y_1,\ldots,y_n$. Plugging this in, what we want to show is
\begin{equation}\label{eqn:goal-norm-bound} \frac{1}{1 - \epsilon} \hat{L}_f(w^{\parallel}, b_0) + \frac{2}{\epsilon(1 - \epsilon) n} \alpha^2 \sum_{i = 1}^n g_i^2 \frac{\|(\Sigma^{\perp}) H\|_2^2}{\|(\Sigma^{\perp})^{1/2} H\|_2^2} \le \frac{\alpha^2}{n} \|(\Sigma^{\perp})^{1/2} H\|_2^2. 
\end{equation}
By the union bound, the following occur together with probability at least $1 - \delta/2$ for some absolute constant $C > 0$:
\begin{enumerate}
    \item Using the first part of Lemma~\ref{lem:sigmah-concentration}, we have
    \[ \|(\Sigma^{\perp})^{1/2} H\|_2^2 \ge \left(1 - C\frac{\log(4/\delta)}{\sqrt{R(\Sigma^{\perp})}}\right) \Tr(\Sigma) \]
    \item Using the last part of Lemma~\ref{lem:sigmah-concentration}, we have
    \[ \frac{\| \Sigma^{\perp} H\|_2^2}{\| (\Sigma^{\perp})^{1/2} H\|_2^2} \le C \log (4/\delta) \frac{\Tr((\Sigma^{\perp})^2)}{(\Tr \Sigma)^2} \]
    \item Using subexponential Bernstein's inequality (Theorem 2.8.1 of \cite{vershynin2018high}), requiring $n = \Omega(\log(1/\delta))$,
    \[ \frac{1}{n} \sum_i g_i^2 \le 2. \]
    \item Using \eqref{eqn:low-dimensional-upper},
    \[ \hat{L}_f(w^{\sharp},b^{\sharp}) \le L_f(w^{\sharp},b^{\sharp}) + \rho_1.  \]
\end{enumerate}
Taking $w_0^{\parallel} = w^{\sharp}$ and $b_0 = b^{\sharp}$, we therefore have
\begin{align*}
 &\frac{1}{1 - \epsilon} \hat{L}_f(w^{\sharp}, b^{\sharp}) + \frac{2}{\epsilon(1 - \epsilon) n} \alpha^2 \sum_{i = 1}^n g_i^2 \frac{\| \Sigma^{\perp} H\|_2^2}{\| (\Sigma^{\perp})^{1/2} H\|_2^2} \\
 &\le 
  \frac{1}{1 - \epsilon} (L_f(w^{\sharp},b^{\sharp}) + \rho_1) + \frac{4C}{\epsilon(1 - \epsilon)} \alpha^2  \log(4/\delta) \frac{\Tr((\Sigma^{\perp})^2)}{\Tr(\Sigma^{\perp})} \\
  &\le   \frac{1}{1 - \epsilon} (L_f(w^{\sharp},b^{\sharp}) + \rho_1) + \frac{4C n}{\epsilon(1 - \epsilon) R(\Sigma^{\perp})}  \log(4/\delta) \frac{\alpha^2 \Tr(\Sigma^{\perp})}{n}
\end{align*}
where in the last step we used the definition of $R(\Sigma^{\perp})$
and on the other hand we have
\[  \frac{\alpha^2\|(\Sigma^{\perp})^{1/2} H\|_2^2}{n} \ge \left(1 - C\frac{\log(4/\delta)}{\sqrt{R(\Sigma^{\perp})}}\right) \frac{\alpha^2 \Tr(\Sigma^{\perp})}{n} \]
which means we have the desired \eqref{eqn:goal-norm-bound} provided
\[ \left(1 - C\frac{\log(4/\delta)}{\sqrt{R(\Sigma^{\perp})}} - \frac{4C n \log(4/\delta)}{\epsilon(1 - \epsilon) R(\Sigma^{\perp})}  \right)\alpha^2 \ge \frac{n}{\Tr(\Sigma^{\perp})} \frac{1}{1 - \epsilon} (L_f(w^{\sharp},b^{\sharp}) + \rho_1) \]
and this satisfies the constraint $\|w^{\sharp}\|^2 + \alpha^2 \le B^2$ provided that
\[ \frac{1}{(1 - \epsilon)\left(1 - C\frac{\log(4/\delta)}{\sqrt{R(\Sigma^{\perp})}} - \frac{4C n \log(4/\delta)}{\epsilon(1 - \epsilon) R(\Sigma^{\perp})}  \right)} \le 1 + \rho_2. \]
Taking $\epsilon = \rho_2/10$, this can be guaranteed if
\[ R(\Sigma^{\perp}) = \Omega\left(\frac{n\log^2(4/\delta)}{\rho_2} \right). \]
\end{proof}
Below are some supporting lemmas used in the proof.
\begin{lemma}[Lemma 10 of \cite{uc-interpolators}]\label{lem:sigmah-concentration}
For any covariance matrix $\Sigma$ and $H \sim \cN(0,I_d)$, it holds that with probability at least $1 - \delta$,
\begin{equation}
1- \frac{\| \Sigma^{1/2} H\|_2^2 }{\Tr(\Sigma)}  \lesssim \frac{\log (4/\delta)}{\sqrt{R(\Sigma)}}
\end{equation}
and
\begin{equation}
\| \Sigma H\|_2^2 \lesssim \log (4/\delta) \Tr(\Sigma^2).
\end{equation}

Therefore, provided that $ R(\Sigma) \gtrsim \log (4/\delta)^2$, it holds that
\begin{equation} \label{eqn:v*-norm}
    \left( \frac{\| \Sigma H\|_2}{\| \Sigma^{1/2} H\|_2} \right)^2 \lesssim \log (4/\delta) \frac{\Tr(\Sigma^2)}{\Tr(\Sigma)}
.\end{equation}

\end{lemma}
\begin{lemma}\label{lem:calculus-lambda}
Suppose that $a,b > 0$. Then if $a/b > 1$, we have
\[ \max_{\lambda \ge 0} \left[\frac{\lambda}{1 + \lambda}a - \lambda b\right] = (\sqrt{a} - \sqrt{b})^2, \]
and if $a/b \le 1$ then
\[  \max_{\lambda \ge 0} \left[\frac{\lambda}{1 + \lambda}a - \lambda b\right] = 0. \]
\end{lemma}
\begin{proof}
Observe that the objective can be rewritten as
\[ g(\lambda) := a - \frac{1}{1 + \lambda}a - \lambda b \]
and the derivative of this expression with respect to $\lambda$ is
\[ g'(\lambda) = \frac{1}{(1 + \lambda)^2} a - b. \]
Therefore the unique critical point of $g$ on the domain $(-1,\infty)$ is at $1 + \lambda = \sqrt{a/b}$. This is the global maximum of $g$ on this domain because $g$ goes to $-\infty$ as $\lambda \to -1$ and as $\lambda \to \infty$. At this point, we have that
\[ g(\lambda) = a - \sqrt{ab} - (\sqrt{a/b} - 1)b = a + b - 2\sqrt{ab} = (\sqrt{a} - \sqrt{b})^2. \]
If $a/b > 1$ this is the global maximum on $[0,\infty)$. Otherwise, the maximum is at the boundary at $\lambda = 0$.
\end{proof}

\subsection{Consistency}

\GenBall*
\begin{proof}
First, we have by Jensen's inequality that
\[ \E \left[ \sup_{\|w\| \le 1} \langle Q x, w \rangle \right] = \E \|Q x\|_2 \le B \sqrt{\E \|Q x\|_2^2} = \sqrt{\Tr(\Sigma^{\perp})}. \]
Applying Theorem~\ref{thm:gaussian-concentration} gives that with probability at least $1 - \delta/4$,
\[ \sup_{\|w\| \le 1} \langle Q x, w \rangle  \le \sqrt{\Tr(\Sigma^{\perp})} + 2\left(\sup_{\|u\| \le 1} \|(\Sigma^{\perp})^{1/2} u\|_2\right)\sqrt{\log(8/\delta)}. \]
\end{proof}
\begin{lemma}\label{lem:gen-ball2}
In the setting of Lemma~\ref{lem:gen-ball1}, suppose that the loss $f$ is the squared loss or squared hinge loss,
and correspondingly $\epsilon_{\lambda,\delta}(w) = \frac{\lambda}{1 + \lambda} \epsilon_{\delta}(w)$. Then with probability at least $1 - \delta$,
\[ L_{f}(w,b) - \epsilon_{\delta}(\phi(w),b) \le \left(\sqrt{\hat{L}_f(w,b)} + \frac{\|w\|_2}{\sqrt{n}}\left[\sqrt{\Tr(\Sigma^{\perp})} + 2\|(\Sigma^{\perp})^{1/2}\|_{op}\sqrt{\log(2/\delta)}\right]\right)^2. \]
\end{lemma}
\begin{proof}
This follows by combining Lemma~\ref{lem:gen-ball1}, 
Corollary~\ref{corr:gen-square-loss}, and Corollary~\ref{corr:gen-square-hinge-loss}.
\end{proof}

\BenignOverfit*
\begin{proof}
It suffices to prove the inequality for fixed $w^{\sharp}, b^{\sharp}$: the conclusion follows automatically from the right-continuity of the CDF of $L_f(\hat w, \hat b)$. 

From Lemma~\ref{lem:norm-bound-interpolator} we have with probability at least $1 - \delta/2$
\[ \|\hat w\|^2 \le \|w^{\sharp}\|_2^2 + (1 + \rho_2)\frac{n}{\Tr(\Sigma_2^{\perp})}(L_f(w^{\sharp},b^{\sharp}) + \rho_1) \]
and from Lemma~\ref{lem:gen-ball2} %
we have for any $w,b$ that with probability at least $1 -\delta/2$
\[ L_{f}(w,b) - \epsilon_{\delta}(\phi(w),b) \le \left(\sqrt{\hat{L}_f(w,b)} + \frac{\|w\|_2}{\sqrt{n}}\left[\sqrt{\Tr(\Sigma^{\perp})} + 2\|(\Sigma^{\perp})^{1/2}\|_{op}\sqrt{\log(2/\delta)}\right]\right)^2 \]
and so for $\hat w, \hat b$ we have
\begin{align*} 
&L_{f}(\hat w,\hat b) - \epsilon_{\delta}(\phi(\hat w),\hat b) \\ 
&\le \frac{\|w\|_2^2}{n}\left[\sqrt{\Tr(\Sigma^{\perp})} + 2\|(\Sigma^{\perp})^{1/2}\|_{op}\sqrt{\log(2/\delta)}\right]^2 \\
&\le \left(\|w^{\sharp}\|_2^2/n + (1 + \rho_2)\frac{1}{\Tr(\Sigma_2^{\perp})}(L_f(w^{\sharp},b^{\sharp}) + \rho_1)\right) \left[\sqrt{\Tr(\Sigma^{\perp})} + 2\|(\Sigma^{\perp})^{1/2}\|_{op}\sqrt{\log(2/\delta)}\right]^2 \\
&= \left(\frac{\|w^{\sharp}\|_2^2\Tr(\Sigma^{\perp})}{n} + (1 + \rho_2)(L_f(w^{\sharp},b^{\sharp}) + \rho_1)\right) \left[1 + 2\|(\Sigma^{\perp})^{1/2}\|_{op}\sqrt{\frac{\log(2/\delta)}{\Tr(\Sigma^{\perp})}}\right]^2
\end{align*}
which proves the result (recalling the definition of $r(\Sigma^{\perp})$).
\end{proof}

\Consistency*
\begin{proof}
The first assumption in \eqref{eqn:l2-benign-assumptions} directly implies that we can choose a sequence $\rho_{2,n} \to 0$ where $\rho_{2,n}$ is the parameter in \eqref{eqn:rho2}. Recalling the general fact that $r(\Sigma^{\perp})^2 \ge R(\Sigma^{\perp})$ \citep{bartlett2020benign}, we see that the same assumption implies $1/r(\Sigma^{\perp}) \to 0$ which implies $\rho_{3,n} \to 0$ where $\rho_{3,n}$ is as defined in Theorem~\ref{thm:l2-overfitting}.

Combining this with (the proof of) Corollary~\ref{corr:main-gen-vc} and using the assumption $k_n/n \to 0$ allows us to handle the $\epsilon_{\delta}(\phi(\hat w), \hat b)$ term, guaranteeing it is negligible compared to the population loss $L_{f,n}(\hat w_n, \hat b_n)$.

To see why we can take $\rho_1 \to 0$, we use Chebyshev's inequality after observing
\[ \Var(\hat{L}_{f,n}(w^{\sharp}_n, b^{\sharp}_n)) = \frac{1}{n} \Var(f(\langle w^{\sharp}, x \rangle + b, y)) \lesssim \frac{1}{n} (\E f(\langle w^{\sharp}, x \rangle + b, y))^2 \]
where we used independence and the hypercontractivity assumption.
\end{proof}

\end{document}